\documentclass[]{./meta/fairmeta}

\definecolor{improve}{RGB}{0,140,0}   %
\definecolor{worse}{RGB}{180,0,0}     %
\usepackage{adjustbox}
\usepackage{booktabs}
\usepackage{multirow}
\usepackage{hyperref}
\usepackage{url}
\usepackage{makecell}
\usepackage[dvipsnames,table]{xcolor}
\usepackage{xspace}
\usepackage{xpatch}
\usepackage{makecell}
\usepackage{nicefrac}
\usepackage{tikz-cd}
\usepackage[framemethod=TikZ]{mdframed}
\usepackage{soul}
\usepackage{multirow}
\usepackage{scalerel}
\usepackage{algorithm}
\usepackage{algpseudocode}
\usepackage{caption}
\usepackage{amsmath,amsfonts,bm,amsthm,amssymb,amsfonts}
\usepackage{stmaryrd}
\usepackage{wrapfig}
\usepackage{calc}
\usepackage{array}
\usepackage{booktabs}
\usepackage{enumitem}
\usepackage{adjustbox}

\usepackage[titletoc,page]{appendix}
\usepackage{titletoc}
\usepackage{tocloft}

\newcommand{\textS}{{\fontfamily{cmr}\selectfont\char159}} %

\newcommand{\scriptscriptscriptstyle}[1]{%
  \mbox{\scalebox{0.7}{$\scriptscriptstyle #1$}}%
}
\addto\extrasenglish{%
  \def\sectionautorefname{\textS\!}%
  \def\subsectionautorefname{\textS\!}%
}

\usepackage{hyperref}

\newmdenv[backgroundcolor=metabg, roundcorner=5pt, skipabove=7pt, linewidth=0pt, innertopmargin=4pt]{metaframe}

\newif\ifmetatemplate
\metatemplatetrue

\usepackage{jmath}

\usepackage{xspace}
\usepackage{calc}
\usepackage{accents}
\usepackage{xparse,xpatch}
\usepackage{thmtools}
\usepackage{hhline}
\usepackage{pifont}

\newcommand{\cmark}{\ding{51}}%
\newcommand{\xmark}{\ding{55}}%

\newcommand{\tdhc}{\textsc{td-hc}\xspace}
\newcommand{\tdfl}{\textsc{td-flow}\xspace}

\usepackage[colorinlistoftodos,disable]{todonotes}

\newcommand{\onestepterm}[1]{\overset{\smash{\scaleto{\rightarrow\mathstrut}{4pt}}}{#1}}
\newcommand{\bootterm}[1]{\overset{\smash{\scaleto{\curvearrowright\mathstrut}{5pt}}}{#1}}

\newsavebox{\onesteparrow}
\newsavebox{\bootarrow}

\sbox{\onesteparrow}{\smash{\scaleto{\rightarrow\mathstrut}{5pt}}}
\sbox{\bootarrow}{\smash{\scaleto{\curvearrowright\mathstrut}{5pt}}}

\NewDocumentCommand{\fixlimits}{e{^_}}{%
  \IfValueT{#1}{^{\mkern4mu#1}}%
  \IfValueT{#2}{_{#2}}%
}
\xapptocmd{\bootterm}{\fixlimits}{}{}
\xapptocmd{\onestepterm}{\fixlimits}{}{}

\newcommand{\sm}{m}

\newcommand{\setfont}[1]{\mathsf{#1}}

\newtheorem{lemma}{Lemma}
\newtheorem{definition}{Definition}

\definecolor{hl}{RGB}{205, 232, 248}
\definecolor{best}{RGB}{199, 221, 236}

\definecolor{sci1}{HTML}{4165c0}
\definecolor{sci2}{HTML}{e770a2}
\definecolor{sci3}{HTML}{5ac3be}
\definecolor{sci4}{HTML}{696969}
\definecolor{sci5}{HTML}{f79a1e}
\definecolor{sci6}{HTML}{ba7dcd}

\definecolor{ocind7}{HTML}{4263eb}
\definecolor{ocgreen7}{HTML}{37b24d}
\definecolor{ocorange7}{HTML}{f76707}

\newcommand{\mhl}[1]{\sethlcolor{hl}\hl{#1}}
\newcommand{\mdec}[1]{\sethlcolor{gray!25}\hl{#1}}
\newcommand{\mdecc}[1]{\sethlcolor{red!25}\hl{#1}}

\usepackage{etoolbox}
\makeatletter
\patchcmd{\hyper@makecurrent}{%
    \ifx\Hy@param\Hy@chapterstring
        \let\Hy@param\Hy@chapapp
    \fi
}{%
    \iftoggle{inappendix}{%
        \@checkappendixparam{chapter}%
        \@checkappendixparam{section}%
        \@checkappendixparam{subsection}%
        \@checkappendixparam{subsubsection}%
        \@checkappendixparam{paragraph}%
        \@checkappendixparam{subparagraph}%
    }{}%
}{}{\errmessage{failed to patch}}

\newcommand*{\@checkappendixparam}[1]{%
    \def\@checkappendixparamtmp{#1}%
    \ifx\Hy@param\@checkappendixparamtmp
        \let\Hy@param\Hy@appendixstring
    \fi
}
\makeatletter

\newtoggle{inappendix}
\togglefalse{inappendix}

\apptocmd{\appendix}{\toggletrue{inappendix}}{}{\errmessage{failed to patch}}
\apptocmd{\subappendices}{\toggletrue{inappendix}}{}{\errmessage{failed to patch}}

\makeatletter
\renewcommand{\sectionautorefname}{\protect\textsection\@gobble}
\renewcommand{\subsectionautorefname}{\protect\textsection\@gobble}
\renewcommand{\subsubsectionautorefname}{\protect\textsection\@gobble}
\makeatother

\def\Vhrulefill{\leavevmode\leaders\hrule height 0.7ex depth \dimexpr0.4pt-0.7ex\hfill\kern0pt}

\xpretocmd{\appendixpagename}{\LARGE}{}{}

\title{Compositional Planning with Jumpy World Models}

\author[2,3,*]{Jesse Farebrother}
\author[1]{Matteo Pirotta}
\author[1]{Andrea Tirinzoni}
\author[2,3,\dagger]{Marc G. Bellemare}
\author[1]{Alessandro Lazaric}
\author[1]{\\Ahmed Touati}

\affiliation[1]{FAIR at Meta}
\affiliation[2]{Mila -- Qu\'ebec AI Institute}
\affiliation[3]{McGill University}

\contribution[*]{Work done at Meta}
\contribution[\dagger]{CIFAR AI Chair}

\abstract{
The ability to plan with temporal abstractions is central to intelligent decision-making.
Rather than reasoning over primitive actions, we study agents that compose pre-trained policies as temporally extended actions, enabling solutions to complex tasks that no constituent alone can solve.
Such compositional planning remains elusive as compounding errors in long-horizon predictions make it challenging to estimate the visitation distribution induced by sequencing policies.
Motivated by the \textit{geometric policy composition} framework introduced in~\citet{thakoor22ghm}, we address these challenges by learning predictive models of multi-step dynamics --- so-called \emph{jumpy world models} --- that capture state occupancies induced by pre-trained policies across multiple timescales in an off-policy manner.
Building on Temporal Difference Flows \citep{farebrother2025temporal}, we enhance these models with a novel consistency objective that aligns predictions across timescales, improving long-horizon predictive accuracy.
We further demonstrate how to combine these generative predictions to estimate the value of executing arbitrary sequences of policies over varying timescales.
Empirically, we find that compositional planning with jumpy world models significantly improves zero-shot performance across a wide range of base policies on challenging manipulation and navigation tasks, yielding, on average, a $200\%$ relative improvement over planning with primitive actions on long-horizon tasks.

}

\begin{document}

\maketitle

\section{Introduction}

In recent years, the success of large-scale foundation models in domains such as computer vision \citep[e.g.,][]{radford21clip,ravi25sam2,assran25vjepa2} and natural language processing \citep[e.g.,][]{llama3,openai24o1,gemini25} has inspired a similar shift in Reinforcement~Learning~(RL). 
Foundation policies pre-trained on diverse unlabeled data or via unsupervised objectives can now generalize to a wide range of downstream tasks, without additional training or explicit planning. This has led to remarkable progress in areas like humanoid control \citep[e.g.,][]{peng21amp,peng22ase,tessler23calm,tessler24masked,tirinzoni25zeroshot,alegre25amor} and real-world robotics~\citep[e.g.,][]{rt-1,rt-2,Luo2024universal,
ghosh24octo,
Black2024pizero,
pi05,
pi06,
nvidia25groot,
li26bfm}.

Despite these advances, a key limitation persists: while foundation policies can handle many tasks out of the box, they often fall short when faced with complex, long-horizon problems that require reasoning over extended sequences of decisions. 
In such cases, the planning horizon of a single policy is insufficient, and agents must compose policies to achieve their goals \citep{schmidhuber91learning,singh92transfer,dayan92feudal,kaelbling93learning,kaelbling93hierarchical,parr97reinforcement,dietterich98maxq,sutton99between,precup00temporal}.

Hierarchical RL \citep{barto03recent,klissarov25discovering}, including the options framework~\citep{sutton99between,precup00temporal}, aims to achieve compositionality by training task-specific high-level policies to leverage manual \citep[e.g.,][]{nachum18data,barreto19option,carvalho23sfkeyboard,park23hiql,park25horizon} or automatic \citep[e.g.,][]{bacon17option,machado17laplacian,machado18eigopts,machado23sr,bagaria21skill,sutton23reward} task decompositions.
In this paper, we take a fundamentally different approach: instead of learning task-specific hierarchies, we develop a framework for direct compositional planning over parameterized policies, requiring \emph{no task-specific} training. By learning ``jumpy'' multi-step dynamics models -- also known as ``jumpy world models'' \citep{murphy24reinforcement} -- we enable flexible composition of existing policies at planning time, transforming how agents tackle novel tasks through intelligent recombination of existing behavior.

To operationalize this idea, we propose learning a \emph{policy} and \emph{horizon} conditioned \emph{jumpy world model} that captures the distribution of future states for all parameterized policies over a continuum of geometrically decaying time horizons \citep{janner20gmodel,thakoor22ghm}.
To make this possible, we first generalize the recent  Temporal Difference Flow \citep{farebrother2025temporal} framework using a novel consistency objective that enforces coherence between predictions at different timescales, consistently improving long-horizon predictions.
Additionally, we develop a novel estimator of the value of executing an arbitrary sequence of policies, each with its own variable timescale.
With these two pieces in place, we demonstrate how to plan over a wide range of parameterized policies, allowing us to flexibly compose behaviors to solve complex, long-horizon tasks without the need for further environment interaction or fine-tuning.  Empirically, across multiple classes of base policies evaluated on a suite of OGBench navigation and manipulation tasks~\citep{park25ogbench}, behavior-level planning consistently improves over zero-shot performance, often by a large margin. Finally, our approach outperforms state-of-the-art hierarchical baselines as well as alternative planning methods; in particular, planning with jumpy world models achieves a 
$200\%$ relative improvement over action-level planning with a one-step world model on long-horizon tasks.
These results demonstrate that planning with jumpy world models offers a powerful complement that is particularly effective for long-horizon decision-making.

\section{Preliminaries}
We model the environment as a reward-free discounted Markov Decision Process (MDP) defined as the $4$-tuple $\mathcal{M} = \left( \setfont{S}, \setfont{A}, P, \gamma \right)$. Here, $\setfont{S}$ and $\setfont{A}$ represent the state and action spaces, respectively; $P : \setfont{S} \times \setfont{A} \to \mathscr{P}(\setfont{S})$ characterizes the distribution over next states; and $\gamma \in [0, 1)$ is the discount factor.
At each step $k$, the agent follows a policy $\pi : \setfont{S} \to \mathscr{P}(\setfont{A})$, generating a trajectory of state-action pairs $(S_k, A_k)_{k \geq 0}$ where $A_k \sim \pi(\,\cdot \!\mid\! S_k)$ and $S_k \sim P(\,\cdot\!\mid\!S_{k-1},A_{k-1})$. 
When unambiguous, we use $S'$ as the immediate next state of $S$, and $S^{+}$ for a successor state at some future step $k > 0$.
We use $\Pr(\,\cdot \!\mid\! S_0=s, A_0=a, \pi)$ and $\mathbb{E}[\,\cdot \!\mid\! S_0=s, A_0=a, \pi]$ to denote the probability and expectation over sequences induced by starting from $(s, a)$ and following~$\pi$~thereafter.

\paragraph{\textbf{Successor Measure}} For a policy $\pi$ and an initial state-action pair $(s, a)$, \textit{the (normalized) successor measure}~\citep{dayan93sr,blier21successor}, denoted by $\sm_\gamma^\pi(\,\cdot \mid s,a\,)$, is a probability measure over the state~space~$\setfont{S}$.
For any subset $\setfont{X} \subseteq \setfont{S}$, it represents the cumulative discounted probability of visiting a state in
$\setfont{X}$, with each visit discounted geometrically by its time of arrival. Formally, this is defined as:
\ifmetatemplate
\begin{equation*}%
\sm^\pi_\gamma(\,\setfont{X} \mid s,a\,)= (1 - \gamma) \sum_{k=0}^\infty \gamma^k\, \mathrm{Pr}(S_{k+1} \in \setfont{X} \mid S_0 = s,\,A_0 = a,\, \pi)\,.
\end{equation*}
\fi
The normalization factor $1 - \gamma$ ensures $m_\gamma^\pi$ is a probability distribution. This admits an intuitive interpretation: rather than viewing $\gamma$ as a discount factor, one may equivalently consider an \emph{auxiliary} process with a geometrically distributed lifetime -- halting at each step with probability $1 - \gamma$ \citep{derman70finite}.
Under this view, $m_\gamma^\pi(\setfont{X} \mid s, a)$ equivalently characterizes the probability that the state visited at the halting time~lies~in~$\setfont{X}$.
This yields a convenient reparameterization of the action-value function as:
\begin{equation}\label{eq:q-ghm}
Q^\pi_{\gamma}(s, a) = \mathbb{E} \Big[\, \sum_{k=0}^\infty \gamma^k r(S_{k+1}) \mid S_0 = s, A_0 = a, \pi \,\Big] \equiv
(1-\gamma)^{-1} \mathbb{E}_{S^{+} \sim m_\gamma^\pi(\cdot\mid s, a)}\left[ r(S^{+}) \right],
\end{equation}
expressing value as the expected reward at the geometrically distributed halting time scaled by the average lifetime $(1 - \gamma)^{-1}$.
Note that this equivalence is purely mathematical\footnote{This preserves expected cumulants and occupancy measures, but not trajectory-level statistics \citep{bellemare23distrl}}; the underlying MDP remains unchanged.
\paragraph{\textbf{Geometric Horizon Model}} A Geometric Horizon Model \citep[GHM;][]{janner20gmodel, thakoor22ghm} instantiates a jumpy world model as a \emph{generative model} of the successor measure.
It can be learned off-policy via temporal-difference learning, exploiting the fact that $\sm^\pi_\gamma$ is a fixed point of the Bellman~equation:
\ifmetatemplate
\begin{equation}\label{eq:ghm-tar}
    \sm^\pi_\gamma(\,\cdot \mid s,a\,) =
    (1-\gamma) P(\,\cdot \mid s, a) + \gamma \mathbb{E}_{S' \sim P(\cdot \mid s, a),\, A' \sim \pi(\cdot \mid S')}\left[ \sm^\pi_\gamma(\,\cdot \mid S', A')\right ] .
\end{equation}
\fi
Due to the Bellman equation's reliance on bootstrapping, the choice of generative model is critical for maintaining stability.
\citet{farebrother2025temporal} demonstrate that prior approaches suffer from systemic bias at long horizons due to these bootstrapped predictions. To address this, \citet{farebrother2025temporal} propose the use of flow-matching techniques \citep{lipman2022flow,albergo23building}, which construct probability paths that evolve smoothly from a source distribution to the desired target distribution. By designing these paths to exploit structure in the temporal difference target distribution, they show that bootstrapping bias can be controlled, enabling more accurate long-horizon predictions.

In this framework, the GHM models a $d$-dimensional continuous state space\footnote{While we assume a continuous state space $\setfont{S} \subseteq \mathbb{R}^d$ for ease of exposition, this flow-matching approach readily extends to non-Euclidean and discrete spaces \citep[e.g.,][]{huang22riemannian,chen24flow,gat24discrete,kapusniak2024metric}.} as an ordinary differential equation (ODE) parameterized by a time-dependent vector field $v_t: \mathbb{R}^{d} \times \setfont{S}\times \setfont{A} \to \mathbb{R}^{d}$. Sampling from the GHM begins by drawing initial noise $X_0 \in \mathbb{R}^d$ from a prior distribution $p_0 \in \mathscr{P}(\mathbb{R}^{d})$ and subsequently following the flow $\psi_t: \mathbb{R}^{d} \times \setfont{S} \times \setfont{A} \to \mathbb{R}^{d}$, defined by the following Initial Value Problem (IVP) for $t \in [0, 1]$:
\ifmetatemplate
\begin{equation*}
\begin{cases}
\;\; \displaystyle \frac{\mathrm{d}}{\mathrm{d}t} \psi_t(X_0\mid s,a) = v_t\big(\psi_t(X_0\mid s, a) \mid s, a\big) \\[.5em]
\;\; \displaystyle \psi_0(X_0\mid s, a) = X_0
\end{cases}
\!\!\!\!\Longleftrightarrow\,\,
\psi_t(X_0 \mid s, a) = X_0 + \int_0^t v_\tau\big(\psi_\tau(X_0\mid s,a)\mid s, a\big)\, \mathrm{d}\tau .
\end{equation*} 
\fi
We can solve this IVP using standard numerical integration techniques \citep{butcher16numerical}. In doing so, we obtain an ODE-induced probability path defined as the pushforward $p_t := \psi_t(\,\cdot\mid S, A)_{\sharp}p_0(\cdot)$, i.e., the distribution of $\psi_t(X_0\mid S, A)$ where $X_0 \sim p_0(\cdot)$.
To ensure that $p_1$ coincides with the successor measure $\sm^\pi_\gamma$, \citet{farebrother2025temporal} propose to learn the parameterized vector field $v_t(\cdots; \theta)$ by minimizing the \tdfl loss%
\ifmetatemplate
:%
\begin{equation}\label{eq: td-flow}
\begin{aligned}[c]
    \ell_{\tdfl}(\theta;\, \bar\theta) &= (1-\gamma)\, \mathbb{E}_{\substack{t\sim\mathcal{U}([0,1]),\,(S, A, S') \sim \mathcal{D}\\X_0\sim p_0(\cdot),\, 
    X_t = (1-t) X_0 + t S'}} \left[ \Big \| v_t(X_t \mid S, A; \theta) - (S' - X_0) \Big \|^2 \right] \\
    & + \gamma\, \mathbb{E}_{\substack{t\sim\mathcal{U}([0,1]),\,(S, A, S')\sim\mathcal{D},\,A'\sim\pi(\cdot\mid S')\\ X_t \sim \psi_t(\cdot \mid S', A';\, \bar{\theta})_{\sharp}p_0(\cdot)}} \left[ \Big \| v_t(X_t \mid S, A; \theta) - v_t(X_t \mid S', A'; \bar{\theta}) \Big \|^2 \right],
\end{aligned}
\end{equation}
over transitions sampled from the dataset $\mathcal{D}$ with non-trainable target parameters $\bar{\theta}$ \citep{mnih15dqn}.
\fi
Recall that the Bellman equation \eqref{eq:ghm-tar} defines the successor measure as a mixture distribution with weights $1 - \gamma$~and~$\gamma$. The \tdfl objective reflects this: the first term is a conditional flow-matching loss~\citep{lipman2022flow} targeting the one-step transition kernel $P(\,\cdot\mid S, A)$, while the second is a marginal flow-matching term targeting the bootstrapped successor measure $\sm^\pi_\gamma(\,\cdot\mid S', A')$. \citet{farebrother2025temporal} show that jointly optimizing these components with the mixture weighting recovers the successor measure at convergence.

\section{Planning via Geometric Policy Composition}
In the sequel, we consider an agent equipped with a repertoire of pretrained policies $\{\pi_z\}_{z \in \setfont{Z}}$ indexed by $z$ (e.g., a state for goal-conditioned policies or, more generally, a latent variable parameterizing diverse behaviors). Our objective is to learn a predictive model of the policies' behaviors that enables planning for arbitrary downstream tasks, without requiring further online interaction with the environment or additional fine-tuning.

To this end, we first formalize how GHM predictions compose to evaluate plans that stochastically switch among a subset of policies, showing how the successor measures of constituent policies combine to yield the successor measure of the composite policy.
We then address the challenge of learning accurate GHMs across timescales by introducing a consistency objective that enforces coherence across horizons, improving long-horizon predictions.
Finally, with these tools we detail our complete compositional planning procedure.

\subsection{Evaluating Geometric Switching Policies}
A natural way to chain policies is through \textit{geometric switching}: for each policy $\pi_{z_{i}}$ in a sequence, execution continues with probability $ 1- \alpha_i \in [0, 1]$, or switches to the subsequent policy $\pi_{z_{i+1}}$ with probability $\alpha_i$. Such a policy can be written as:
$\nu := \pi_{z_1} \xrightarrow{\alpha_1} \pi_{z_2} \cdots \xrightarrow{\alpha_{n-1}} \pi_{z_n}\,.$
By definition, the final policy $\pi_{z_n}$ is absorbing, meaning the agent commits to it for the remainder of the episode, so its switching probability is $\alpha_n = 0$.
These non-Markovian policies are called Geometric Switching Policies~\citep[GSPs;][]{thakoor22ghm}.
The term ``geometric'' captures that each policy $\pi_{z_i}$ is followed for a geometrically distributed duration $T_i \sim \mathrm{Geom}(\alpha_i)$.

To analyze these policies, we must understand how this switching mechanism interacts with the MDP's \emph{global} discount factor, $\gamma$. Recall that $\gamma$ can be interpreted as the probability the episode continues to the next step, while $1 - \gamma$ gives the probability of halting. When following policy $\pi_{z_k}$ within a GSP, there are two reasons it might stop executing that policy:
(1) the episode halts, with probability $1 - \gamma$; or 
(2) the policy switches to $\pi_{z_{k+1}}$, with probability $\alpha_k$.
Thus, continuing to follow $\pi_{z_k}$ for one more step requires that neither event occur, which happens with probability $\beta_k := \gamma\,(1-\alpha_k)$.
This quantity acts as an effective discount factor for the duration spent executing $\pi_{z_k}$.
Note that $\beta_n = \gamma$, since the final policy is absorbing.

We can now characterize the successor measure of a GSP. Note that the agent might reach a successor state $s^{+}$ through multiple paths: arriving while still executing $\pi_{z_1}$, or switching to $\pi_{z_2}$ and reaching $s^{+}$ from there, and so on. Each path contributes to the overall successor measure and must be weighted~appropriately.
\ifmetatemplate
\begin{metaframe}
\fi
\begin{definition}\label{def:gsp-weights}
Let
$
\nu := \pi_{z_1} \xrightarrow{\alpha_1} \pi_{z_2} \cdots \xrightarrow{\alpha_{n-1}} \pi_{z_n}
$
be a geometric switching policy with global discount factor $\gamma \in (0,1)$ and effective discount factors $\beta_k \coloneqq \gamma(1-\alpha_k)$ for $k\in\llbracket n\rrbracket$. The \textit{weight} of the $k$-th~policy~is
\vspace{-0.25em}
\[
w_k \;\coloneqq\; \frac{1-\gamma}{1-\beta_k}\;\prod_{i=1}^{k-1}\frac{\gamma-\beta_i}{1-\beta_i},
\]
where an empty product equals $1$ (hence $w_1 = \smash{\frac{1 - \gamma}{1 - \beta_1}}$).
\end{definition}
\ifmetatemplate
\end{metaframe}
\fi
These weights capture the relative contribution of each policy phase to the successor measure of the GSP. Intuitively, $w_k$ reflects the probability that the agent (i) survives the first $k-1$ policy phases without the episode halting, and (ii) reaches states under policy $\pi_{z_k}$ rather than having already switched to a later policy. With these weights, we now define the successor measure of the composite GSP in the following~result.
\ifmetatemplate
\begin{metaframe}
\fi
\begin{restatable}{theorem}{thmgspsm}\label{th: gsp sm}
Let $\nu := \pi_{z_1} \xrightarrow{\alpha_1} \pi_{z_2} \cdots \xrightarrow{\alpha_{n-1}} \pi_{z_n}$ be a geometric switching policy with global discount factor $\gamma \in (0, 1)$, effective discount factors $\{\beta_k\}_{k=1}^n$, and weights $\{w_k\}_{k=1}^n$ from \autoref{def:gsp-weights}.
For any state-action pair $(s, a)$, the successor measure of $\nu$ decomposes as:
\ifmetatemplate
\vspace{-0.25em}
\begin{equation*}
    \sm^\nu_\gamma(\mathrm{d}s^{+} \mid s, a) =
     \sum_{k=1}^n w_k \int_{\substack{s_1, \ldots, s_{k-1} \\ a_1, \ldots, a_{k-1}}} m_{\beta_1}^{\pi_{z_1}}(\mathrm{d}s_1 \mid s, a) \pi_{z_2}(\mathrm{d}a_1 \mid s_1) %
    \,\cdots\, m_{\beta_k}^{\pi_{z_k}}(\mathrm{d}s^{+} \mid s_{k-1}, a_{k-1}) \, .
\end{equation*}
\fi 

\end{restatable}
\ifmetatemplate
\end{metaframe}
\fi

This theorem decomposes the successor measure of a GSP as a mixture distribution with $n$ components. The $k$-th component, weighted by $w_k$, captures the state distribution under policy $\pi_{z_k}$, having passed through all intermediate states visited under $\pi_{z_1},\ldots,\pi_{z_{k-1}}$%
\footnote{Each component can be seen as an application of the Chapman--Kolmogorov equation \citep{chapman28on,kolmogoroff1931analytischen}, retaining kernel composition via marginalization but replacing single-step dynamics $P(\cdot \!\mid\! s,a)$ with jumpy dynamics via~$m_\beta^{\pi}(\cdot \!\mid\! s,a)$.}.
Practically, such a decomposition yields a natural strategy for estimating action-value function similar to~\eqref{eq:q-ghm}:
sample a state from each component, evaluate its reward, and form a weighted sum using the weights~$w_k$.
Rather than sampling from each component independently, we can leverage their overlapping sequential structure to draw samples much more efficiently.
We achieve this through composition:
starting from $(s, a)$, we sample $S^{\scriptscriptstyle +}_1 \sim m_{\scriptstyle \beta_1}^{\pi_{\scriptscriptstyle z_{\scriptscriptscriptstyle{1}}}}(\,\cdot\mid s, a)$, then use $S^{\scriptscriptstyle +}_1$ to sample $S^{\scriptscriptstyle +}_2 \sim m_{\scriptstyle \beta_{\scriptscriptstyle 2}}^{\scriptstyle \pi_{\scriptscriptstyle z_{\scriptscriptscriptstyle{2}}}}(\,\cdot\mid S^{\scriptscriptstyle +}_1, A^{\scriptscriptstyle +}_1)$ where $A^{\scriptscriptstyle +}_1\sim \pi_{\scriptstyle z_{\scriptscriptstyle 2}}(\,\cdot\mid S^{\scriptscriptstyle +}_1)$,~and~so~on.
As the following lemma formalizes, the weighted sum $\sum_{k} w_k\, r(S_k^{\scriptscriptstyle +})$ yields an unbiased estimator of $Q^\nu_\gamma$.

\ifmetatemplate
\begin{metaframe}
\fi
\begin{lemma}\label{lem: gsp q-fucntion}Let $\nu := \pi_{z_1} \xrightarrow{\alpha_1} \pi_{z_2} \cdots \xrightarrow{\alpha_{n-1}} \pi_{z_n}$ be a geometric switching policy with global discount factor $\gamma \in (0, 1)$, effective discount factors $\{\beta_k\}_{k=1}^n$, and weights $\{w_k\}_{k=1}^n$ from \autoref{def:gsp-weights}. For any reward function $r : \setfont{S} \to \mathbb{R}$ and state-action pair $(s, a)$, set $(S^{+}_0, A^{+}_0) = (s, a)$ and for $k=1,\dots,n$ sample $S^{+}_k \sim m_{\beta_k}^{\pi_{z_k}}(\,\cdot\mid S^{+}_{k-1}, A^{+}_{k-1})$ and $A^{+}_k \sim \pi_{z_{k+1}}(\,\cdot\mid S^{+}_k)$. Then the single-sample monte-carlo estimator
\vspace{-0.25em}
\begin{align*}%
     \widehat{Q}^{\nu}_{\gamma} \coloneqq (1-\gamma)^{-1} \sum_{k=1}^n w_k\, r(S^{+}_k) \,,
\end{align*}
is an unbiased estimate of $Q^\nu_\gamma(s, a)$, i.e., $\smash{\mathbb{E}\big[\, \widehat{Q}^\nu_\gamma \,\big] = Q^\nu_\gamma(s, a)}$ where the expectation is over the joint distribution of $(S^{+}_1, A^{+}_1, \ldots, S^{+}_n)$ induced by the sampling procedure above.
\end{lemma}
\ifmetatemplate
\end{metaframe}
\fi

This lemma generalizes several previous results: it extends \citet[Theorem 3.2]{thakoor22ghm}, which assumed fixed switching probabilities, i.e., $\alpha_k=\alpha, \forall\, k \in \llbracket n-1 \rrbracket$, and further generalizes results in \citet{janner20gmodel}, that additionally impose a fixed policy throughout, i.e., $\pi_{z_k} = \pi, \forall\, k \in \llbracket n \rrbracket$. By allowing both policies and switching probabilities to vary, this result enables the evaluation of a wider class of switching policies and brings us closer to the options framework \citep{sutton95td,sutton99between,precup00temporal}.

\subsection{Learning Geometric Horizon Models Across Multiple Timescales}\label{ssec:td-hc}

A core requirement of our planning framework is the ability to predict the behavior of many policies over multiple timescales, enabling us to calculate the expected return of candidate geometric switching policies.
A natural extension of the \tdfl objective \eqref{eq: td-flow} conditions the vector field $v$ on both the policy encoding $z$ and discount factor $\gamma$, yielding a single unified model across policies and horizons.
However, generalizing across many horizons is challenging: variance increases with horizon length, reducing per-horizon accuracy and destabilizing training \citep{petrik08biasing}.

To address this challenge, we propose a generalization of \tdfl that enforces consistency across horizons.
Rather than learning each horizon independently, we exploit a Bellman-like relationship between the successor measure at two discount factors ${\textcolor{ocorange7}{\boldsymbol\beta}} \leq {\textcolor{ocind7}{\boldsymbol\gamma}}$ to bootstrap longer-horizon predictions from shorter-horizon ones:
\begingroup%
\let\oldbeta\beta%
\let\oldgamma\gamma%
\renewcommand{\beta}{\textcolor{ocorange7}{\boldsymbol\oldbeta}}%
\renewcommand{\gamma}{\textcolor{ocind7}{\boldsymbol\oldgamma}}%
\begin{align}~\label{eq: hc-bellman}%
    \sm^\pi_{\gamma}(\,\cdot \mid s, a) \,&=\, (1-\gamma)\, P(\,\cdot \mid s, a) +
    \gamma \frac{1-\gamma}{1-\beta} \,\mathbb{E}_{\substack{S' \sim P(\cdot \mid s, a), A' \sim \pi(\cdot \mid S')} }\left [\, \sm^\pi_{\beta}(\,\cdot \mid S', A')\, \right ] \\
    &+\; \gamma \frac{\gamma - \beta}{1-\beta} 
    \,\mathbb{E}_{\substack{S' \sim P(\cdot \mid s, a), A' \sim \pi(\cdot \mid S') \\
    S^{+} \sim \sm^\pi_{\beta}(\cdot \mid S', A'), A^{+} \sim \pi(\cdot \mid S^{+}) } }\left [\, m^\pi_{\gamma}(\,\cdot \mid S^{+}, A^{+})\,\right ].  \nonumber
\end{align}%
\endgroup%
This follows directly from \autoref{th: gsp sm} by considering the switching policy $\smash{\nu := \pi \xrightarrow{\alpha_1=1} \pi \xrightarrow{\alpha_2 = 1-\nicefrac{\beta}{\gamma}} \pi}$ and noting that $\sm^\pi_{\gamma (1-\alpha_1)} = m^\pi_{0} = P$.
\ifmetatemplate
\fi
Building on this result, we extend the derivation of \tdfl from~\citet{farebrother2025temporal} to the new Bellman equation in~\eqref{eq: hc-bellman} (full derivation in \autoref{app:td-horizon-cons}), arriving at what we call the \emph{Temporal Difference Horizon Consistency (\tdhc)} loss%
\ifmetatemplate
:%
\begingroup
\let\oldbeta\beta
\let\oldgamma\gamma
\renewcommand{\beta}{\textcolor{ocorange7}{\boldsymbol\oldbeta}}
\renewcommand{\gamma}{\textcolor{ocind7}{\boldsymbol\oldgamma}}
\begin{align}\label{eq: td-hc}
    \ell_{\tdhc}(\theta; \bar\theta, \beta, \gamma) = (1-\gamma) &\mathbb{E}_{\substack{t\sim\mathcal{U}([0,1]),\, (S, A, S')\sim\mathcal{D} \\ X_0 \sim p_0(\cdot),\, X_t = (1-t) X_0 + t S'}} \left[ \Big \| v_t(X_t \mid S, A, \gamma; \theta) - (S' - X_0) \Big \|^2 \right] \\
    +\;\; \gamma \frac{1-\gamma}{1-\beta} &\mathbb{E}_{\substack{
    t\sim\mathcal{U}([0,1]),\,
    (S, A, S')\sim\mathcal{D},A'\sim \pi(\cdot\mid S') \\
    X_t \sim \psi_t(\cdot \mid S', A', \beta; \bar{\theta})_{\sharp} p_0(\cdot)}} \left[ \Big \| v_t(X_t \mid S, A, \gamma;\theta) - v_t(X_t \mid S', A', \beta;\,\bar{\theta}) \Big \|^2 \right] \nonumber \\
     +\; \gamma \frac{\gamma - \beta}{1-\beta} 
    &\mathbb{E}_{\substack{t\sim\mathcal{U}([0,1]),\, (S, A, S')\sim\mathcal{D}, A' \sim \pi(\cdot | S') \\
    S^{\scriptscriptstyle{+}} \!\sim \psi_1(\cdot \mid S', A', \beta; \bar{\theta})_{\sharp} p_0(\cdot),\,A^{\scriptscriptstyle{+}} \!\sim \pi(\cdot\mid S^{\scriptscriptstyle{+}}) \\
    X_t \sim \psi_t(\cdot \mid S^{+}, A^{+}, \gamma; \bar{\theta})_{\sharp} p_0(\cdot) }}\! \left[ \Big \| v_t(X_t \mid S, A, \gamma;\theta ) - v_t(X_t \mid S^{+}, A^{+}, \gamma; \bar{\theta}) \Big \|^2 \right]\!. \nonumber
\end{align}
\endgroup
\fi
When $\textcolor{ocind7}{\boldsymbol\gamma}=\textcolor{ocorange7}{\boldsymbol\beta}$, the third term vanishes, and we recover the original \tdfl loss.
In practice, we sample $\textcolor{ocind7}{\boldsymbol\gamma}$ uniformly from $[\gamma_{\min}, \gamma_{\max}]$ and $\textcolor{ocorange7}{\boldsymbol\beta}$ uniformly from $[\gamma_{\min}, \textcolor{ocind7}{\boldsymbol\gamma}]$, but only apply horizon consistency (i.e., $\textcolor{ocorange7}{\boldsymbol\beta} \ne \textcolor{ocind7}{\boldsymbol\gamma}$) to a small proportion of each mini-batch. This is motivated by the fact that the consistency term requires sampling from the model's own predictions at horizon $\textcolor{ocind7}{\boldsymbol\beta}$ and using these samples as conditioning for the longer horizon $\textcolor{ocorange7}{\boldsymbol\gamma}$, meaning errors in the model's current predictions can compound and introduce bias. Restricting the consistency term to a small fraction of the batch, we gain the benefits of horizon alignment while ensuring the majority of updates come from \tdfl. This design follows standard practice for consistency-like generative modeling \citep[e.g.,][]{frans24onestep,geng25mean,boffi25how,ai26joint}.

\subsection{Compositional Planning}
\label{sec: test-time-planning}
With the machinery for training policy-conditioned GHMs across multiple timescales now in place, we can turn to the central question: how to use them to solve downstream tasks? Given a reward function $r : \setfont{S} \to \mathbb{R}$, our goal is to find a sequence of policies that maximizes expected return.
Since the learned GHMs already capture the full complexity of how policies evolve in the environment, evaluating $Q^\nu_\gamma$ from \autoref{lem: gsp q-fucntion} requires only specifying the policy embeddings $z_1, \ldots, z_n$, hence planning reduces to the following optimization problem:
\begin{equation}\label{eq:planning.}
    \max_{a_1, z_1, \ldots, z_n} Q^{\pi_{z_1} \xrightarrow{\alpha_1} \pi_{z_2} \,\cdots\, \xrightarrow{\alpha_{n-1}} \pi_{z_n}}_{\gamma}(s,a_1) \, .
\end{equation}
The switching probabilities $\{\alpha_i\}$ control how long each policy executes before transitioning to the next, and are treated as hyperparameters.
Once the optimal sequence $(a^\ast_1, z^\ast_1, \ldots, z_n^\ast)$ is identified, we execute the first action $a_1^\ast$ followed by the policy $\pi_{z_1^\ast}$, replanning at future states as needed.

Notably, this planning objective unifies several existing approaches as special cases.
By varying the switching probabilities $\{\alpha_i\}$, one can interpolate between action-level control and planning over policies by:
\begin{itemize}[leftmargin=*]
    \item Setting $\alpha_1 = \cdots = \alpha_{n} = 1$, which reduces to optimizing over sequences of primitive actions, equivalent to Model-Predictive Control with horizon $n$.
    \item Setting $\alpha_1 = 1$ and $\alpha_2 = \cdots = \alpha_{n} = 0$, yielding~\emph{Generalized~Policy~Improvement}~\citep[\!\emph{GPI};\!][]{barreto2017sf}.
    \item Setting $\alpha_1 = \cdots = \alpha_{n-1} = \alpha$ for some fixed $\alpha \in (0,1)$, which recovers \emph{Geometric Generalized Policy Improvement}~\citep[\!\emph{GGPI};\!][]{thakoor22ghm}.
\end{itemize}
We refer to our approach -- which allows distinct switching probabilities $\alpha_1, \ldots, \alpha_{n-1} \in (0, 1)$ -- as \textsc{CompPlan}. Likewise, we refer to action-level planning as \textsc{ActionPlan}. Crucially, the same pretrained GHMs power all of these methods. By conditioning on policies and a continuum of timescales, our framework spans one-step world models ($\alpha = 1$) through long-horizon policy composition ($\alpha_i \in (0, 1)$), unifying previously disparate~paradigms.

\paragraph{\textbf{Optimization via random shooting.}} The maximization in \eqref{eq:planning.} is tractable when policies are indexed by a finite set, but becomes challenging for large or continuous $\setfont{Z}$.
In this paper, we focus on goal-conditioned policies where $\setfont{Z} = \setfont{S}$ and $z \in \setfont{S}$ represents a subgoal. Here, the key difficulty lies in proposing good candidate subgoals without searching over the entire state space.
Our solution is to use the GHMs themselves as a proposal distribution.
Given a goal $g \in \setfont{S}$, we generate waypoints from $z_0 := s$ by composing $\smash{m^{\scriptstyle \pi_{\scriptscriptstyle g}}_{\scriptstyle \beta_{\scriptscriptstyle i}}}$ over horizons $\smash{\{\beta_i\}}$ as:
\begin{equation*}
    a_i \sim \pi_g(\,\cdot \mid z_i),
    \quad z_{i+1} \sim \sm^{\pi_g}_{\beta_{i+1}}(\,\cdot \mid z_i, a_i),
    \quad \text{for } i = 0, \ldots, n-1\,.
\end{equation*}
This produces a sequence of subgoals $(z_1, \ldots, z_n)$ that naturally guides progress towards the final objective.
Alternatively, we can sample from an unconditional GHM that predicts plausible successor states under the data distribution (i.e., the behavior policy).
To enable this, we stochastically mask the policy encoding as $z = \varnothing$ during training, bootstrapping with the dataset action at the next state in \eqref{eq: td-flow} when masked.

With a proposal distribution, planning reduces to random shooting \citep{matyas1965random}: we (1) sample $m$ candidate sequences $\{(z^{\scriptscriptstyle(i)}_1, \ldots, z^{\scriptscriptstyle(i)}_n)\}_{\scriptscriptstyle{ i=1}}^{\scriptscriptstyle{m}}$ from the proposal distribution; (2) evaluate $Q^{\scriptstyle\nu^{\scriptscriptscriptstyle{(i)}}}_\gamma(s, a^{\scriptscriptstyle(i)}_1)$ for each candidate switching policy $\smash{\nu^{\scriptscriptstyle(i)} := \pi_{\scriptscriptstyle z^{\scriptscriptscriptstyle{(i)}}_{\scriptscriptscriptstyle 1}}\xrightarrow{\scriptstyle \alpha_{1}}\pi_{\scriptscriptstyle z^{\scriptscriptscriptstyle{(i)}}_{\scriptscriptscriptstyle 2}}\cdots\xrightarrow{\scriptstyle \alpha_{n-1}}\pi_{\scriptscriptstyle z^{\scriptscriptscriptstyle{(i)}}_{\scriptscriptscriptstyle n}}}$ using \autoref{lem: gsp q-fucntion} where $a^{\scriptscriptstyle(i)}_1 \sim \pi_{\scriptscriptstyle z^{\scriptscriptscriptstyle{(i)}}_{\scriptscriptscriptstyle 1}}(\,\cdot\mid s)$; and (3) select the sequence $(a^\ast_1, z^\ast_1, \ldots, z^\ast_n)$ with the highest value $Q^{\nu^\ast}_\gamma(s, a^\ast_1)$.
The full method is~summarized~in~\autoref{alg:compplan}.

\section{Experiments}
Our empirical evaluation tests the core hypothesis of this work: that learning a jumpy world model over a diverse collection of parameterized policies enables effective and efficient compositional planning. We begin by describing the experimental setting and the training procedure for the base policies.
We subsequently compare the performance of these policies against that of compositional planning over them.
Additionally, we benchmark \textsc{CompPlan} against other test-time planning approaches and hierarchical methods.
Finally, we examine the effect of the proposed consistency loss on both model accuracy and planning performance.
Additional ablations on the replanning frequency, planning objective, and proposal distribution~are~provided~in~\autoref{app:sec:results}.

\subsection{Experimental Setup}
\textbf{Benchmark and Dataset\;} All experiments are conducted using the OGBench benchmark \citep{park25ogbench}, which provides challenging long-horizon robotic manipulation and locomotion tasks structured as offline goal-conditioned reinforcement learning problems.
We focus on ant navigation tasks across different maze topologies (\textsc{medium}, \textsc{large}, and \textsc{giant}) as well as multi-cube robotic arm manipulation. For both policy and GHM training, we use the \emph{standard} \textsc{navigate} and \textsc{play} offline datasets for \textsc{antmaze} and \textsc{cube}, respectively.

\textbf{Base Policies\;} The effectiveness of compositional planning depends on the quality and characteristics of the underlying policies. We train five distinct policy types, each exhibiting different tradeoffs between GHM learning and planning performance: 1) Goal-Conditioned TD3~\citep[GC-TD3;][]{pirotta24bfmil}; 2) Goal-Conditioned 1-Step RL (GC-1S); 3) Contrastive RL~\citep[CRL;][]{eysenbach22contrastive}; 4) Goal-Conditioned Behavior Cloning \citep[GC-BC;][]{lynch19learning,ghosh21gcbc}; and 5) Hierarchical Flow Behavior Cloning~\citep[HFBC;][]{park25horizon}. Additional details are provided in \autoref{app:policies}.

\begin{table*}[b]
\setlength{\aboverulesep}{0pt}
\setlength{\belowrulesep}{0pt}
\centering
\caption{Success rate ($\uparrow$) of base policies $\pi_g$ (Zero Shot) and compositional planning with GHMs (\textsc{CompPlan}; ours) averaged over tasks. We report the mean and standard deviation over 3 seeds. We highlight relative \mhl{increases} and \mdec{decreases} in performance w.r.t.\ the base policies. Additionally, we \textbf{bold} the best performance for each domain.
}
\label{tab:success.zeroshot.plan}
\renewcommand{\arraystretch}{1.3}
\begin{adjustbox}{width=\textwidth}
\begin{tabular}{l|cc|cc|cc|cc|cc}
\toprule
\multirow{2}{*}{\textbf{Domain}} & \multicolumn{2}{c}{\textbf{CRL}}& \multicolumn{2}{c}{\textbf{GC-1S}}& \multicolumn{2}{c}{\textbf{GC-BC}}& \multicolumn{2}{c}{\textbf{GC-TD3}}& \multicolumn{2}{c}{\textbf{HFBC}}\\
 & \textbf{Zero Shot} & \textbf{\textsc{CompPlan}} & \textbf{Zero Shot} & \textbf{\textsc{CompPlan}} & \textbf{Zero Shot} & \textbf{\textsc{CompPlan}} & \textbf{Zero Shot} & \textbf{\textsc{CompPlan}} & \textbf{Zero Shot} & \textbf{\textsc{CompPlan}}\\
\midrule
\textsc{antmaze-medium} & 0.88 & \cellcolor{hl}{\textbf{0.97 {\footnotesize (0.02)}}} & 0.56 & \cellcolor{hl}{0.87 {\footnotesize (0.05)}} & 0.49 & \cellcolor{hl}{0.85 {\footnotesize (0.08)}} & 0.65 & \cellcolor{gray!25}{0.65 {\footnotesize (0.03)}} & 0.94 & \cellcolor{gray!25}{0.94 {\footnotesize (0.01)}}\\
\hhline{~|----------}
\textsc{antmaze-large} & 0.84 & \cellcolor{hl}{0.90 {\footnotesize (0.00)}} & 0.21 & \cellcolor{hl}{0.61 {\footnotesize (0.04)}} & 0.18 & \cellcolor{hl}{0.73 {\footnotesize (0.02)}} & 0.23 & \cellcolor{hl}{0.48 {\footnotesize (0.05)}} & 0.78 & \cellcolor{hl}{\textbf{0.92 {\footnotesize (0.02)}}}\\
\hhline{~|----------}
\textsc{antmaze-giant} & 0.16 & \cellcolor{hl}{0.29 {\footnotesize (0.03)}} & 0.00 & \cellcolor{hl}{0.02 {\footnotesize (0.00)}} & 0.00 & \cellcolor{hl}{0.03 {\footnotesize (0.01)}} & 0.00 & \cellcolor{hl}{0.01 {\footnotesize (0.01)}} & 0.42 & \cellcolor{hl}{\textbf{0.79 {\footnotesize (0.04)}}}\\
\hhline{~|----------}
\textsc{cube-1} & 0.28 & \cellcolor{hl}{0.86 {\footnotesize (0.02)}} & 0.37 & \cellcolor{hl}{0.66 {\footnotesize (0.02)}} & 0.90 & \cellcolor{hl}{\textbf{0.99 {\footnotesize (0.01)}}} & 0.58 & \cellcolor{hl}{0.91 {\footnotesize (0.01)}} & 0.80 & \cellcolor{hl}{0.97 {\footnotesize (0.01)}}\\
\hhline{~|----------}
\textsc{cube-2} & 0.02 & \cellcolor{hl}{0.50 {\footnotesize (0.03)}} & 0.10 & \cellcolor{hl}{0.57 {\footnotesize (0.09)}} & 0.15 & \cellcolor{hl}{\textbf{0.97 {\footnotesize (0.01)}}} & 0.12 & \cellcolor{hl}{0.82 {\footnotesize (0.01)}} & 0.76 & \cellcolor{hl}{0.77 {\footnotesize (0.02)}}\\
\hhline{~|----------}
\textsc{cube-3} & 0.01 & \cellcolor{hl}{0.73 {\footnotesize (0.02)}} & 0.01 & \cellcolor{hl}{0.67 {\footnotesize (0.02)}} & 0.09 & \cellcolor{hl}{\textbf{0.92 {\footnotesize (0.01)}}} & 0.12 & \cellcolor{hl}{0.83 {\footnotesize (0.04)}} & 0.64 & \cellcolor{hl}{0.83 {\footnotesize (0.03)}}\\
\hhline{~|----------}
\textsc{cube-4} & 0.00 & \cellcolor{hl}{0.39 {\footnotesize (0.04)}} & 0.01 & \cellcolor{hl}{0.60 {\footnotesize (0.02)}} & 0.00 & \cellcolor{hl}{\textbf{0.76 {\footnotesize (0.03)}}} & 0.00 & \cellcolor{hl}{0.57 {\footnotesize (0.03)}} & 0.24 & \cellcolor{hl}{0.67 {\footnotesize (0.03)}}\\
\bottomrule
\end{tabular}
\end{adjustbox}
\end{table*}

\subsection{How Do Policies Affect Compositional Planning?}\label{ssec:zero-shot-policy-comp}
For each policy family, we train a GHM in an off-policy manner using the \tdhc loss described in~\autoref{ssec:td-hc}. GHMs are trained for 3M gradient steps using the Adam optimizer~\citep{kingma15adam} with a batch size of 256. The model architecture follows a U-Net-style design, similar to~\citet{farebrother2025temporal}. Both the timestep $t$ and the discount factor $\gamma$ are embedded by first applying a sinusoidal embedding to increase dimensionality, followed by a two-layer MLP with mish activations~\citep{misra19mish}. For the discount embedding, we further concatenate the vector $[\gamma, 1-\gamma, -\log(1-\gamma)]$, where $-\log(1-\gamma)$ corresponds to the logarithm of the effective horizon; we find this improves the model's sensitivity to the discount factor. Other conditioning information, such as the state-action pair and policy embedding 
$z$, is processed through an additional MLP and added to both the time and discount embeddings. The network incorporates conditioning information via FiLM modulation~\citep{perez18film}.
When training each GHM, we apply the horizon-consistency objective \eqref{eq: td-hc} (i.e., $\beta \ne \gamma$) to $25\%$ of each mini-batch in \textsc{antmaze} and $12.5\%$ in \textsc{cube}. We also train the unconditional model (i.e., $z = \varnothing$) for $10\%$ of each mini-batch; these two proportions do not overlap.

During evaluation, we tailor the proposal distribution to each domain based on their distinct characteristics. In \textsc{antmaze}, states are separated by large temporal distances and physical barriers, so an unconditional proposal would waste most samples on irrelevant regions of the state space; we therefore sample 256 subgoal sequences from the goal-conditioned GHM. In \textsc{cube}, tasks consist of short pick-and-place sequences with many viable paths, making the unconditional GHM a natural fit; we sample 1024 sequences to cover the broader proposal.
Ablations of these choices are provided in~\autoref{appendix: proposal_ablation}.

We begin by comparing the zero-shot performance of the base policies against our compositional planning approach. \autoref{tab:success.zeroshot.plan} details these results (averaged over three seeds), showing that compositional planning consistently improves upon the zero-shot policies. 
This demonstrates the effectiveness of our method in selecting policy sequences that exceed the performance of the best individual policy.
While improvements are evident across all domains, the gains are particularly pronounced in complex, long-horizon tasks such as \textsc{antmaze-giant} and \textsc{cube-\{3,4\}}, where success rates can rise from $10\%$ to $90\%$ in the most extreme~cases.

Among the evaluated policy classes, HFBC emerges as the most consistent zero-shot performer.
\textsc{CompPlan} further improves HFBC, notably in \textsc{antmaze-giant} and \textsc{cube-4}. CRL, in contrast, is effective in \textsc{antmaze} but underperforms in \textsc{cube}.
We hypothesize this stems from the inductive bias in CRL's representation learning, which approximates the goal-conditioned value function as $Q(s, a, g) \approx \phi(s, a)^\top \psi(g)$.
This factorization effectively captures the ant's spatial position -- yielding robust navigation in \textsc{antmaze} -- but fails to encode complex object-related features, resulting in weaker \textsc{cube} policies.
Notably, incorporating planning not only improves upon CRL's strong baseline in \textsc{antmaze}, but also achieves non-trivial success rates in \textsc{cube}, demonstrating that our approach can extract utility from base policies otherwise limited by their~inductive~bias.

Interestingly, the other policies exhibit the opposite trend. When evaluated zero-shot, they show weak performance in medium to long-horizon tasks -- \textsc{antmaze-medium}, \textsc{antmaze-large}, and \textsc{cube} tasks all prove challenging, with \textsc{cube} success rates falling below $10\!-\!15\%$.
However, these policies prove remarkably effective when integrated into our compositional planning framework, with success rates climbing above $70\%$ in many tasks.
Taken together, these results suggest that zero-shot metrics tell us little about how well a policy will compose; its utility may only become clear when orchestrated with other policies.

\subsection{How Does Compositional Planning Compare to Other Planning  and Hierarchical Approaches?}

Our results in \autoref{ssec:zero-shot-policy-comp} demonstrate that compositional planning unlocks capabilities beyond what any single policy achieves alone.
We now situate our approach among existing planning and hierarchical methods, revealing that \textsc{CompPlan}'s advantages stem from its unique combination of temporal abstraction and flexible~composition.

Recall from \autoref{sec: test-time-planning} that different choices of switching probabilities recover existing methods as special cases. This observation suggests a natural ablation: comparing \textsc{CompPlan} against these special cases can disentangle the contribution of its key components.
At one extreme lies \emph{Generalized Policy Improvement} \citep[GPI;][]{barreto2017sf} which sets $\alpha_1 = 1$ and commits to a single policy for the remainder of the episode. GPI leverages our GHMs to estimate each policy's value, selecting the best in-class policy at each timestep according to:
$$
\max_{z, A \sim \pi_z(\cdot\mid s)} Q^{\pi_z}_\gamma(s,a) =  (1 - \gamma)^{-1} \mathbb{E}_{S \sim m^{\pi_z}_\gamma(\cdot\mid s, a)}\left[\, r(S)\, \right] \,.
$$
GPI thus serves as a test of whether composing policies offers an advantage over merely selecting among them.

At the other extreme, we compare against action-level planning (\textsc{ActionPlan}), which sets $\alpha_1 = \cdots = \alpha_n = 1$ and operates at the granularity of individual actions rather than policies.
To fairly isolate the effect of temporal abstraction, we train a dedicated one-step world model $\tilde{p}(\cdot| s, a)$ using the same flow-matching framework, network architecture, and training procedure as our GHMs -- differing only in the removal of policy and discount conditioning.
Given this model, we optimize the following objective:
$$\argmax_{A_1, \ldots, A_n}\, \sum_{k=1}^n \gamma^k r(S_{k+1})\,,$$
where $S_{k+1} \sim \tilde{p}(\,\cdot\mid S_k, A_k)$ and $A_{k} \sim \pi_g(\,\cdot\mid S_k)$.
This comparison directly tests whether planning over sequences of policies -- enabled by jumpy predictions -- confers an advantage over action-level planning.

\autoref{fig:wm.gpi.ghm.planning} reveals a clear pattern: on long-horizon tasks, \textsc{CompPlan} substantially outperforms both alternatives, achieving an $89\%$ relative improvement over GPI and a $201\%$ gain over \textsc{ActionPlan} (averaged across policies and long-horizon domains). These gains indicate that neither policy selection alone nor action-level planning captures the full benefit of our compositional framework.
Rather, the combination of planning over sequences of policies at multiple timescales unlocks strong long-horizon capabilities.

\ifmetatemplate
\begin{wraptable}[14]{r}{0.5075\textwidth}
\vspace{-1.325em}
\fi
\setlength{\aboverulesep}{0pt}
\setlength{\belowrulesep}{0pt}
\centering
\caption{Success rate ($\uparrow$) of hierarchical baselines and compositional planning with HFBC policies (\textsc{CompPlan}; ours). For each domain, we \mhl{highlight} the best performance.}
\label{tab:success_adjustbox}
\setlength{\extrarowheight}{1pt}
\renewcommand{\arraystretch}{1.25}
\ifmetatemplate
\begin{adjustbox}{width=0.5075\textwidth}
\fi
\begin{tabular}{l|c|c|c|c}
\toprule
\textbf{Domain}& \textbf{HIQL} & \textbf{SHARSA} & \textbf{HFBC} & \textbf{\textsc{CompPlan}}\\
\midrule
\textsc{antmaze-medium} & \cellcolor{hl}{0.96 {\footnotesize (0.01)}} & 0.91 {\footnotesize (0.03)} & 0.94 & 0.94 {\footnotesize (0.01)} \\
\textsc{antmaze-large} & {0.91 {\footnotesize (0.02)}} & 0.88 {\footnotesize (0.03)} & 0.78 & \cellcolor{hl}{0.92 {\footnotesize (0.02)}}\\
\textsc{antmaze-giant} & 0.65 {\footnotesize (0.05)} & 0.56 {\footnotesize (0.07)} & 0.42 & \cellcolor{hl}{0.79 {\footnotesize (0.04)}}\\
\textsc{cube-1} & 0.15 {\footnotesize (0.03)} & 0.70 {\footnotesize (0.03)} & 0.84 & \cellcolor{hl}{0.97 {\footnotesize (0.01)}}\\
\textsc{cube-2} & 0.06 {\footnotesize (0.02)} & 0.60 {\footnotesize (0.07)} & 0.70 & \cellcolor{hl}{0.77 {\footnotesize (0.02)}}\\
\textsc{cube-3} & 0.03 {\footnotesize (0.01)}& 0.50 {\footnotesize (0.09)} & 0.54 & \cellcolor{hl}{0.83 {\footnotesize (0.03)}}\\
\textsc{cube-4} & 0.00 {\footnotesize (0.00)} & 0.09 {\footnotesize (0.04)} & 0.34 & \cellcolor{hl}{0.67 {\footnotesize (0.03)}}\\
\bottomrule
\end{tabular}
\ifmetatemplate
\end{adjustbox}
\fi
\ifmetatemplate
\end{wraptable}
\fi

Furthermore, we compare against methods that learn hierarchical structure during training: HIQL \citep{park23hiql} and SHARSA \citep{park25horizon}, the current state-of-the-art on OGBench. These approaches train high-level policies to select subgoals or skills, whereas \textsc{CompPlan} performs composition at test time using pre-trained GHMs. \autoref{tab:success_adjustbox} presents this comparison where \textsc{CompPlan} employs HFBC base policies similar to SHARSA. Our approach consistently outperforms both hierarchical methods, with the largest margins on the most challenging tasks. In \textsc{cube-4}, \textsc{CompPlan} achieves $67\%$ success compared to $9\%$ for SHARSA and $0\%$ for HIQL -- demonstrating that test-time composition can surpass learned hierarchies when tasks demand flexible, long-horizon reasoning. Notably, \textsc{CompPlan} requires no task-specific training, suggesting a promising alternative to task-specific hierarchical methods.

\begin{figure*}[t]
    \centering
    \includegraphics[width=\linewidth]{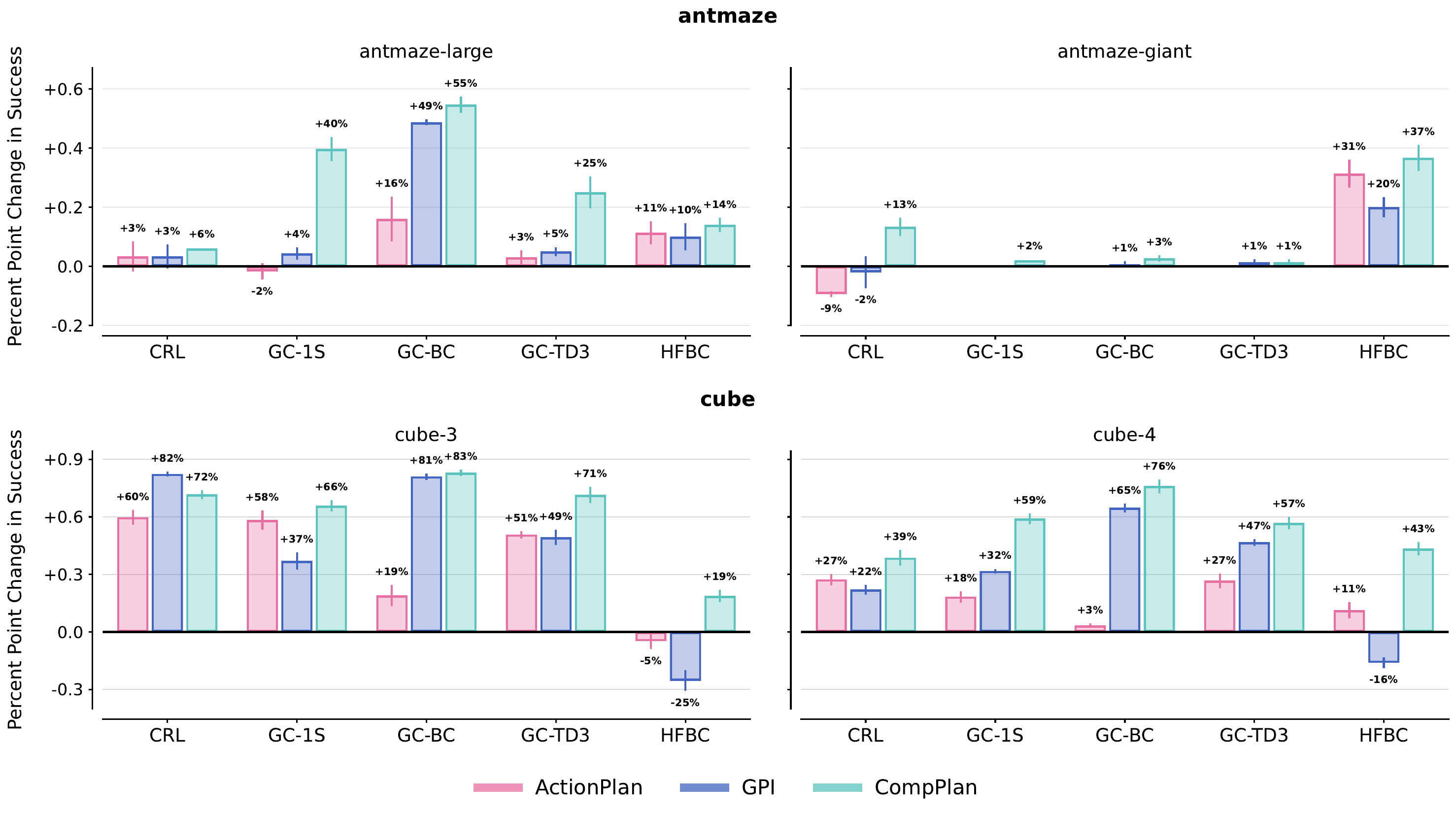}
    \caption{Percentage point change ($\uparrow$) over zero-shot policies. We compare: action-level planning (\textsc{ActionPlan}) with a world model; generalized policy improvement (GPI) and compositional planning (\textsc{CompPlan}; ours) with GHMs.}
    \label{fig:wm.gpi.ghm.planning}
\end{figure*}

\subsection{How Does Horizon Consistency Affect GHM Learning?}\label{ssec:td-hc-exp}

Having established that compositional planning yields strong empirical gains, we now turn inward to examine one of our methodological contributions: the horizon consistency objective from \autoref{ssec:td-hc}. We investigate its impact through two complementary lenses -- generative fidelity and downstream planning -- revealing that consistency plays distinct roles at different stages of the pipeline.

\ifmetatemplate
\begin{wraptable}[14]{r}{0.5075\textwidth}
\vspace{-1.325em}
\fi
\setlength{\aboverulesep}{0pt}
\setlength{\belowrulesep}{0pt}
\centering
\caption{Accuracy (EMD; $\downarrow$) of GHMs trained with our horizon consistency loss (\tdhc) and without (\tdfl) for discount factor $\gamma=0.995$. We \mhl{highlight} the best performing method.}
\label{tab:consistency_ablation_emd_short}
\setlength{\extrarowheight}{1pt}
\renewcommand{\arraystretch}{1.25}
\ifmetatemplate
\begin{adjustbox}{width=0.5075\textwidth}
\fi
\begin{tabular}{l|ll|ll}
\toprule
\multirow[c]{2}{*}{\textbf{Domain}} & \multicolumn{2}{c}{\textbf{CRL}}& \multicolumn{2}{c}{\textbf{GC-1S}}\\
& \multicolumn{1}{c}{\textbf{\tdfl (\xmark)}} & \multicolumn{1}{c|}{\textbf{\tdhc (\cmark)}} & \multicolumn{1}{c}{\textbf{\tdfl (\xmark)}} & \multicolumn{1}{c}{\textbf{\tdhc (\cmark)}} \\
\midrule
\textsc{antmaze-medium} & 4.41 {\footnotesize (0.05)} & \cellcolor{hl}{4.22 {\footnotesize (0.06)}} & 4.40 {\footnotesize (0.02)} & \cellcolor{hl}{4.22 {\footnotesize (0.03)}} \\
\hhline{~|----}
\textsc{antmaze-large} & 5.24 {\footnotesize (0.07)} & \cellcolor{hl}{4.81 {\footnotesize (0.03)}} & 5.12 {\footnotesize (0.18)} & \cellcolor{hl}{4.67 {\footnotesize (0.04)}} \\
\hhline{~|----}
\textsc{antmaze-giant} & 6.77 {\footnotesize (0.49)} & \cellcolor{hl}{5.74 {\footnotesize (0.06)}} & 7.29 {\footnotesize (0.69)} & \cellcolor{hl}{5.25 {\footnotesize (0.08)}} \\
\hhline{~|----}
\textsc{cube-1} & 1.60 {\footnotesize (0.02)} & \cellcolor{hl}{1.57 {\footnotesize (0.03)}} & 1.43 {\footnotesize (0.00)} & \cellcolor{hl}{1.33 {\footnotesize (0.03)}} \\
\hhline{~|----}
\textsc{cube-2} & 2.36 {\footnotesize (0.03)} & \cellcolor{hl}{2.23 {\footnotesize (0.02)}} & 1.86 {\footnotesize (0.04)} & \cellcolor{hl}{1.71 {\footnotesize (0.01)}} \\
\hhline{~|----}
\textsc{cube-3} & 2.15 {\footnotesize (0.02)} & \cellcolor{hl}{2.10 {\footnotesize (0.02)}} & 1.80 {\footnotesize (0.04)} & \cellcolor{hl}{1.71 {\footnotesize (0.03)}} \\
\hhline{~|----}
\textsc{cube-4} & 2.41 {\footnotesize (0.03)} & \cellcolor{hl}{2.34 {\footnotesize (0.01)}} & 2.13 {\footnotesize (0.03)} & \cellcolor{hl}{2.05 {\footnotesize (0.03)}} \\
\bottomrule
\end{tabular}
\ifmetatemplate
\end{adjustbox}
\fi
\ifmetatemplate
\end{wraptable}

We begin by asking whether enforcing consistency across timescales produces more accurate GHM predictions.
To isolate this effect, we train two GHMs -- one using \tdfl and the other using \tdhc\ -- keeping all else fixed.
We obtain ground-truth samples by executing $64$ policy rollouts from $256$ randomly selected (state, goal) pairs and resampling $2048$ visited states according to $t \sim \mathrm{Geom}(1-\gamma)$.
We then draw an equal number of samples from each GHM and compute the Earth Mover's Distance \citep[EMD;][]{RubnerTG00} between the two sets.
As \autoref{tab:consistency_ablation_emd_short} demonstrates, consistency systematically improves accuracy at long horizons. The effect is particularly striking in \textsc{antmaze}, where the maze structure imposes hard constraints on reachability. For example, employing \tdhc with GC-1S policies in $\textsc{antmaze-giant}$ results in a $28\%$ reduction in EMD. Qualitatively, \autoref{fig:qual-ghm} shows that \tdhc on \textsc{antmaze-giant} leads to fewer samples erroneously traversing walls, a failure mode that compounds over long horizons.

Given these accuracy improvements, one might expect correspondingly large gains in planning. Surprisingly, \autoref{tab:success_adjustbox2} tells a different story: planning success rates are nearly identical with and without consistency, averaging only a $5\%$ relative improvement. These findings are not contradictory but rather illuminate when consistency matters most. Our planning procedure evaluates candidate sequences using effective horizons $\{\beta_i\}$ in the range of $50-100$ steps (i.e., $\beta_i \in [0.98, 0.99]$), not the $200+$ step horizons where consistency provides its largest accuracy gains. At these moderate timescales, the base \tdfl objective already learns sufficiently accurate models to rank policies correctly.
The consistency objective thus provides a margin of safety for long-horizon predictions without being strictly necessary for the planning horizons we employ in OGBench.
This suggests that practitioners facing longer-horizon tasks would benefit most from the consistency~loss~herein.

\section{Discussion}
This work reframes pre-trained policies not as isolated controllers but as composable primitives -- building blocks to be sequenced.
Jumpy world models provide the mechanism: by predicting successor states for many policies across a continuum of horizons, they enable planning over behavior rather than primitive actions.
Empirically, compositional planning consistently outperforms individual policies, hierarchical methods, and action-level planning, with striking gains at long horizons.
Looking ahead, we see many promising directions: learning state-dependent switching probabilities, jointly learning policies and predictive models, employing more sample-efficient model-predictive control methods, and exploring jumpy world models in learned latent~spaces.

\section*{Acknowledgements}
The authors thank Harley Wiltzer, Arnav Jain, Pierluca D'oro, Nate Rahn, Michael Rabbat, Yann Ollivier, Marlos C. Machado, Michael Bowling, Adam White, and Doina Precup for useful discussions that helped improve this work.
MGB is supported by the Canada CIFAR AI Chair program and NSERC.
Finally, this work was made possible by the Python community, particularly NumPy \citep{harris20numpy}, Matplotlib \citep{hunter07mpl}, Seaborn \citep{waskom2021seaborn}, Einops \citep{rogozhnikov22einops}, and Mujoco \citep{todorov2012mujoco}.

\bibliography{bookends}
\bibliographystyle{./meta/assets/plainnat}

\clearpage

\appendix
\onecolumn

\begin{appendices}
\vspace{-0.95em}

\startcontents[sections]
\printcontents[sections]{l}{1}{\setcounter{tocdepth}{2}}

\section{Extended Related Works}
\paragraph{\textbf{Successor Measure}}

Methods that learn (discounted) state occupancies employing temporal difference learning \citep{sutton95td} date back to the successor representation \citep{dayan93sr} with more recent extensions like successor features \citep{barreto2017sf} and the successor measure \citep{blier21successor,blier22thesis}. \citet{janner20gmodel} was the first to introduce a generative model of the successor measure with $\gamma$-models also referred to as geometric horizon models \citep{thakoor22ghm}. \citet{wiltzer24dsm} additionally introduced $\delta$-models that learn a distribution over $\gamma$-models enabling applications in distributional RL \citep{bellemare17distrl,bellemare23distrl,dabney18qrdqn}. Many generative modeling techniques have been applied to learn these models, including GANs \citep{janner20gmodel,wiltzer24dsm}, normalizing flows \citep{janner20gmodel}, VAEs \citep{thakoor22ghm,tomar2024video}, flow matching \citep{farebrother2025temporal,zheng25intention}, and diffusion \citep{schramm24belldiff,farebrother2025temporal}. Closely related is work on distributional successor features also known as multi-variate distributional RL \citep{freirich19distributional,gimelfarb21risk,zhang21multidim,wu23distributional,wiltzer24mvdrl,zhu24gom}, that involves modeling the distribution over cumulative finite-dimensional features induced by a policy.
From a generative modeling perspective, our work generalizes temporal difference flows \citep{farebrother2025temporal} and shows how long-horizon predictions can be improved by training across multiple timescales with a novel horizon consistency objective.

Additionally, our compositional framework can be viewed as a generalization of Geometric Generalized Policy Improvement \citep[GGPI;][]{thakoor22ghm} to arbitrary switching probabilities; however, our concrete formulation and empirical evaluation differ substantially from \citet{thakoor22ghm}.
First, \citet{thakoor22ghm} considers only four pre-trained policies in practice, and learn two separate GHMs with effective horizons of $5$ and $10$ steps respectively. Their experiments also restrict composition to sequences of only two policies. In contrast, we learn GHM models conditioned on a continuous family of policies and timescales with horizons up to $25\times$ longer, and evaluate GSPs with lengths ranging from $3$ to $24$ policies.
As a result, our work not only enables a richer class of geometric switching policies but also performs a more comprehensive empirical validation of these techniques on challenging long-horizon tasks where temporal abstraction is most important.

\paragraph{\textbf{Planning With Temporal Abstractions}}

Several prior works have explored planning over subgoals or waypoints to solve long-horizon tasks \citep[e.g.,][]{nasiriany19planning,eysenbach19search,nair20hierarchical,chanesane21goal,fang22planning, hafner2022deep,lo24goal,gurtler25long}. These methods typically involve learning sub-goal conditioned policies together with a high-level dynamics model that predicts the outcomes of reaching these subgoals $k$ steps in the future, then employ model predictive control to select subgoal sequences.
In contrast, GHMs model the entire state-occupancy distribution rather than a fixed $k$-step lookahead, allowing planning over arbitrary reward functions rather than just goal-conditioned~tasks.

A parallel line of work employs diffusion models \citep{vincent2011connection,sohl2015deep,ho20ddpm,song21sde,chieh25the} for trajectory-level planning \citep[e.g.,][]{janner22planning,ajay23is,zheng23guided,chen25extendable,yoon25monte,yoon25fast,luo25generative,lee25state}.
Rather than modeling policy-induced dynamics, these methods train generative models over trajectory segments and perform planning within the denoising process. Execution then relies on inverse dynamics models to extract actions from planned state sequences.
While powerful, this paradigm has notable limitations: (i) it requires learning accurate inverse dynamics; (ii) planning quality depends heavily on the trajectory distribution in the training data rather than on the capabilities of any particular policy; and (iii) these methods often assume access to oracle goal representations during the denoising process.
In contrast, our work is policy-grounded: it directly composes the behaviors of pre-trained policies by predicting their induced state occupancies, rather than planning over abstract trajectory segments. This makes our approach agnostic to the policy class and avoids inverse dynamics entirely since actions are sampled directly from the base policies.

Closer to our approach, are methods that learn dynamics models over temporally extended behaviors \citep[e.g.,][]{xie21latent,shi2022skill,zhang23leveraging, mishra23generative, gurtler25long}. These methods learn general latent skills together with a high-level dynamics model predicting the outcomes of their execution (more precisely the states reached a fixed number of steps in the future), and use MPC to plan over these skills. However, these approaches suffer from the same limitations as the $k$-step subgoal methods discussed above.
While we only report experiments for goal-based policies, \textsc{CompPlan} also enables planning on top of any set of parameterized skills or policies, for example, those learned via unsupervised RL methods \citep{borsa19usfa,touati21fwdbwd,touati23zeroshot,park24hilbert,frans24unsupervised,cetin25finer,agarwal25psm,tirinzoni25zeroshot,sikchi25fast,sikchi2025rlzero,TD-JEPA}.

Finally, our work connects to both the options framework \citep{sutton99between,precup00temporal} and hierarchical RL \citep{schmidhuber91learning,kaelbling93hierarchical,parr97reinforcement,barto03recent,klissarov25discovering}, which share a focus on temporal abstraction and multi-level decision making. Prior work has explored planning with options \citep[e.g.,][]{silver12compositional,jinnai19finding,barreto19option,carvalho23sfkeyboard,rodriguez24learning} or learning policies and value functions at multiple levels of abstraction \citep[e.g.,][]{precup97multi,precup97planning,precup98theoretical,dayan92feudal,dietterich98maxq,vezhnevets17feudal,kulkarni16hierarchical,gurtler21hierarchical,nachum18data,levy2018hierarchical,park23hiql,park25horizon}.
Our compositional planning method can be viewed both as planning over options with a simple termination condition -- where each policy terminates after a random number of steps -- and as a hierarchical RL method that replaces the high-level policy with a test-time planning procedure.

\clearpage
\section{Algorithms}
We present the full \tdhc algorithm along with \textsc{CompPlan} and the two proposal distributions outlined herein.

\begin{algorithm}[H]
    \caption{Temporal Difference Flows with Horizon Consistency}\label{alg:td-flow-hc}
    \addcontentsline{toc}{subsection}{Algorithm \ref{alg:td-flow-hc}: Temporal Difference Flows with Horizon Consistency}
\begin{small}
    \begingroup
    \let\oldbeta\beta
    \let\oldgamma\gamma
    \renewcommand{\beta}{\textcolor{ocorange7}{\boldsymbol\oldbeta}}
    \renewcommand{\gamma}{\textcolor{ocind7}{\boldsymbol\oldgamma}}
    \begin{algorithmic}[1]
\State \textbf{Inputs}: offline dataset $\mathcal{D}$, policy $\pi$, batch size $K$, Polyak coefficient $\zeta$, randomly initialized weights $\theta$, learning rate $\eta$, maximum discount factor $\oldgamma_{\mathrm{max}} \in [0, 1)$, horizon consistency proportion $\tau_c \in [0, 1]$.
    \For{$n = 1, \dots$}
    \State Sample mini-batch $\{ (S_k, A_k, S_k', A'_k)\}_{k=1}^K$ from $\mathcal{D}$
    \For{$k = 1, \ldots, K$}
    \State $t_k \sim \mathcal{U}([0,1])$
    \State $\gamma_k \sim \mathcal{U}([0, \oldgamma_{\mathrm{max}}])$
    \State

    \State {\bf \textcolor{gray!50!blue}{\# \texttt{One-Step Term}}}
    \State $X_0 \sim p_0(\cdot)$
    \State $\onestepterm{X}_{t_k} \gets (1 - t_k) X_0 + t_k S'_k$
    \State $\onestepterm{\ell}_k(\theta) = \big\| v_{t_k}(\onestepterm{X}_{t_k}\mid S_k, A_k, \gamma_k; \theta) - (S_k' - X_0) \big\|^2$
    \State

    \If{$k \le \lceil K \cdot \tau_c \rceil$}
    \State $\beta_k \sim \mathcal{U}([0, \gamma_k])$
    \State {\bf \textcolor{gray!50!blue}{\# \texttt{$\boldsymbol\oldbeta$-Bootstrap Term}}}
    \State $X_0 \sim p_0(\cdot)$
    \State $\bootterm{X}^\beta_{t_k} \gets \psi_{t_k}(X_0 \mid S'_k, A'_k, \beta_k; \bar\theta)$
    \State $\bootterm{\ell}^{\beta}_k(\theta) =  
    \big\| v_{t_k}(\bootterm{X}^\beta_{t_k}\mid S_k, A_k, \gamma_k; \theta) - v_{t_k}(\bootterm{X}^\beta_{t_k}\mid S'_k, A'_k, \beta_k; \bar\theta) \big\|^2$
    \State {\bf \textcolor{gray!50!blue}{\# \texttt{$\boldsymbol\oldgamma$-Bootstrap Term}}}
    \State $(X_0, X''_0) \sim p_0(\cdot)$
    \State $S''_k \gets \psi_1(X''_0 \mid S'_k, A'_k, \beta_k; \bar\theta)$
    \State $A''_k \sim \pi(\cdot\mid S''_k)$
    \State $\bootterm{X}^\gamma_{t_k} \gets \psi_{t_k}(X_0\mid S''_k, A''_k, \gamma_k; \bar\theta)$
    \State $\bootterm{\ell}^\gamma_k(\theta) =  
    \big\| v_{t_k}(\bootterm{X}^\gamma_{t_k}\mid S_k, A_k, \gamma_k; \theta) - v_{t_k}(\bootterm{X}^\gamma_{t_k}\mid S''_k, A''_k, \gamma_k; \bar\theta) \big\|^2$
    \State {\bf \textcolor{gray!50!blue}{\# \texttt{Mixture Loss}}}
    \State $\ell_k(\theta) = (1 - \gamma_k) \onestepterm{\ell}_k(\theta) + \gamma_k \frac{1 - \gamma_k}{1 - \beta_k} \bootterm{\ell}^\beta_k(\theta) + \gamma_k \frac{\gamma_k - \beta_k}{1 - \beta_k} \bootterm{\ell}^\gamma_k(\theta)$
    \Else{}
    \State {\bf \textcolor{gray!50!blue}{\# \texttt{$\boldsymbol\oldgamma$-Bootstrap Term}}}
    \State $X_0 \sim p_0(\cdot)$
    \State $\bootterm{X}^\gamma_{t_k} \gets \psi_{t_k}(X_0 \mid S'_k, A'_k, \gamma_k; \bar\theta)$
    \State $\bootterm{\ell}^{\gamma}_k(\theta) =  
    \big\| v_{t_k}(\bootterm{X}^\gamma_{t_k}\mid S_k, A_k, \gamma_k; \theta) - v_{t_k}(\bootterm{X}^\gamma_{t_k}\mid S'_k, A'_k, \gamma_k; \bar\theta) \big\|^2$
    \State {\bf \textcolor{gray!50!blue}{\# \texttt{Mixture Loss}}}
    \State $\ell_k(\theta) = (1 - \gamma_k) \onestepterm{\ell}_k(\theta) + \gamma_k \bootterm{\ell}^\gamma_k(\theta)$
    \EndIf

    \EndFor
    \State {\bf \textcolor{gray!50!blue}{\# \texttt{Compute loss}}}
    \State $\ell(\theta) = \frac{1}{K} \sum_{k=1}^K \ell_k(\theta)$

    \State {\bf  \textcolor{gray!50!blue}{\# \texttt{Perform gradient step}}}
    \State $\theta \leftarrow \theta - \eta \nabla_{\theta} \ell(\theta)$
    \State {\bf  \textcolor{gray!50!blue}{\# \texttt{Update parameters of target vector field}}}
    \State  $\bar\theta \leftarrow \zeta \bar\theta + (1-\zeta) \theta$ 
    \EndFor
\end{algorithmic}
\endgroup
 \end{small}
\end{algorithm}

\begin{algorithm}[H]
    \caption{Compositional Planning with Jumpy World Models}\label{alg:compplan}
    \addcontentsline{toc}{subsection}{Algorithm \ref{alg:compplan}: Compositional Planning with Jumpy World Models}
\begin{small}
    \begin{algorithmic}[1]
\State \textbf{Inputs}: parameterized class of policies $\{\pi_z\}_{z \in \setfont{Z}}$, geometric horizon model $m^{\pi_z}_\gamma$, policy sequence length $K$, proposal distribution $\rho : \setfont{S} \to \mathscr{P}(\setfont{Z}^K)$, number of proposals $M$, number of monte-carlo samples $N$, reward function $r$, effective discount factors $\{\beta_k\}_{k=1}^{K}$, mixture weights $\{w_k\}_{k=1}^{K}$
\Function{\textsc{CompPlan}}{$s$}
\For{$i = 1, \ldots, M$}
    \State {\bf \textcolor{gray!50!blue}{\# \texttt{Sample policy sequence from proposal distribution}}}
    \State $(z_1^{(i)}, \ldots, z_K^{(i)}) \sim \rho(\cdot \mid s)$
    \State
    \State {\bf \textcolor{gray!50!blue}{\# \texttt{Sample initial action}}}
    \State $a_1^{(i)} \sim \pi_{z_1^{(i)}}(\cdot \mid s)$
    \State
    \State {\bf \textcolor{gray!50!blue}{\# \texttt{Monte Carlo Q-value estimation (\autoref{lem: gsp q-fucntion})}}}
    \For{$j = 1, \ldots, N$}
        \State $(S_0, A_0) \gets (s, a_1^{(i)})$
        \For{$k = 1, \ldots, K$}
            \State $S_k \sim m^{\pi_{z_k^{(i)}}}_{\beta_k}(\cdot \mid S_{k-1}, A_{k-1})$
            \State $A_k \sim \pi_{z_{k+1}^{(i)}}(\cdot \mid S_k)$
        \EndFor
        \State $\widehat{Q}^{(i,j)} \gets (1 - \gamma)^{-1} \sum_{k=1}^{K} w_k \cdot r(S_k)$
    \EndFor
    \State $\widehat{Q}^{(i)} \gets \frac{1}{N} \sum_{j=1}^{N} \widehat{Q}^{(i,j)}$
\EndFor
\State
\State {\bf \textcolor{gray!50!blue}{\# \texttt{Select best candidate}}}
\State $i^* \gets \arg\max_{i \in \llbracket M \rrbracket} \widehat{Q}^{(i)}$
\State
\State {\bf \textcolor{gray!50!blue}{\# \texttt{Return optimal action and policy}}}
\State \Return $(a_1^{(i^*)}, z_1^{(i^*)})$
\EndFunction
\end{algorithmic}
 \end{small}
\end{algorithm}

\begin{figure}[H]
    \vspace{-2.5mm}
    \begin{minipage}[t]{0.48\textwidth}
    \begin{algorithm}[H]
    \caption{Goal-Conditioned Proposal}\label{alg:goal-cond-proposal}
    \addcontentsline{toc}{subsection}{Algorithm \ref{alg:goal-cond-proposal}: Goal-Conditioned Proposal}
    \begin{small}
        \begin{algorithmic}[1]
    \State \textbf{Inputs}: geometric horizon model $m^{\pi_z}_\gamma$, policy sequence length $K$, effective discount factors $\{\beta_k\}_{k=1}^{K}$
    \Function{\textsc{GoalCondProposal}}{$s, g$}
        \State {\bf \textcolor{gray!50!blue}{\# \texttt{Chain GHM samples toward goal}}}
        \State $z_0 \gets s$
        \For{$k = 1, \ldots, K$}
            \State $A_{k-1} \sim \pi_g(\cdot\mid z_{k-1})$
            \State $z_k \sim m^{\pi_g}_{\beta_{k}}(\cdot \mid z_{k-1}, A_{k-1})$
        \EndFor
        \State \Return $(z_1, \ldots, z_K)$
    \EndFunction
    \end{algorithmic}
     \end{small}
    \end{algorithm}
    \end{minipage}
    \hfill
    \begin{minipage}[t]{0.48\textwidth}
    \begin{algorithm}[H]
    \caption{Unconditional Proposal}\label{alg:uncond-proposal}
    \addcontentsline{toc}{subsection}{Algorithm \ref{alg:uncond-proposal}: Unconditional Proposal}
    \begin{small}
        \begin{algorithmic}[1]
    \State \textbf{Inputs}: unconditional (behavior policy $\mu$) geometric horizon model $m^\mu_\gamma$, policy sequence length $K$, effective discount factors $\{\beta_k\}_{k=1}^{K}$
    \Function{\textsc{UncondProposal}}{$s$}
        \State {\bf \textcolor{gray!50!blue}{\# \texttt{Chain unconditional GHM samples}}}
        \State $z_0 \gets s$
        \For{$k = 1, \ldots, K$}
            \State $z_k \sim m^{\mu}_{\beta_{k}}(\cdot \mid z_{k-1})$
        \EndFor
        \State \Return $(z_1, \ldots, z_K)$
    \EndFunction
    \end{algorithmic}
     \end{small}
    \end{algorithm}
    \end{minipage}
    \caption{Two proposal distributions for goal-conditioned compositional planning. \textbf{Left:} \textsc{GoalCondProposal} samples subgoal sequences by chaining GHM predictions conditioned on the goal $g \in \setfont{S}$, guiding the agent toward the target. \textbf{Right:} Unconditional proposal samples from the behavior policy's GHM that can be trained alongside $m^{\pi_z}_\gamma$ by periodically setting $z = \varnothing, a = \varnothing$.}
    \label{fig:proposal-distributions}
\end{figure}

\clearpage
\section{Theoretical Results}\label{app:sec:theoretical-results}
\ifmetatemplate
\begin{metaframe}
\fi
\thmgspsm*
\ifmetatemplate
\end{metaframe}
\fi
\begin{proof}
    Let's denote $\nu_{l:n} = \pi_{z_l} \xrightarrow{\alpha_l} \pi_{z_{l+1}} \ldots \xrightarrow{\alpha_{n-1}} \pi_{z_n}$ the geometric switching policy that starts by $\pi_{z_l}$.
    we will proceed by induction over $l \in \{ n, n-1, \ldots 1\}$ to show that: 
    \begin{equation}\label{eq:induction_hyp}
        \sm^{\nu_{l:n}}_\gamma(\mathrm{d}s' \mid s, a) =
     \sum_{k=l}^n \frac{1-\gamma}{1 - \beta_k} \left( \prod_{i=l}^{k-1} \frac{\gamma - \beta_i}{1-\beta_i} \right)  \int_{\substack{s_l, \ldots, s_{k-1} \\ a_l, \ldots, a_{k-1}}}    m_{\beta_l}^{\pi_{z_l}}(\mathrm{d}s_l \mid s, a) \pi_{z_{l+1}}(\mathrm{d}a_l \mid s_l) %
    \ldots m_{\beta_k}^{\pi_{z_k}}(\mathrm{d}s' \mid s_{k-1}, a_{k-1}),
    \end{equation}
    where $(s_{l-1}, a_{l-1}) = (s, a)$.

For the case $l=n$, it is straightforward to see that the induction hypothesis~\eqref{eq:induction_hyp} is satisfied since $\sm^{\nu_{n:n}}_\gamma =\sm^{\pi_{z_n}}_\gamma$. 

Let us now assume that the induction hypothesis~\eqref{eq:induction_hyp} holds for $l+1 \in \{ n, n-1, \ldots ,2\}$. Our goal is to demonstrate that it also holds for $l$.

After executing a single step of $\nu_{l:n}$, two outcomes are possible:: with probability $(1-\alpha_l)$, we remain committed to $\nu_{l:n}$, or with probability $\alpha_l$, we switch to the next policy $\pi_{z_{l+1}}$, thereby continuing the episode with $\nu_{l+1:n}$. This leads to the following Bellman-equation:
\begin{align*}
    \sm^{\nu_{l:n}} = (1-\gamma) P + \gamma (1-\alpha_{l}) P^{\pi_{z_l}}\sm^{\nu_{l:n}} + \gamma \alpha_l  P^{\pi_{z_{l+1}}} \sm^{\nu_{l+1:n}}
\end{align*}
which implies
\begin{align*}
    & (I - \gamma \beta_l P^{\pi_{z_l}} )\sm^{\nu_{l:n}} = (1-\gamma) P + \gamma \alpha_l  P^{\pi_{z_{l+1}}} \sm^{\nu_{l+1:n}} \\
    & \Longrightarrow \sm^{\nu_{l:n}}  = (1-\gamma) (I - \gamma \beta_l P^{\pi_{z_l}} )^{-1} P + \gamma \alpha_l (I - \gamma \beta_l P^{\pi_{z_l}} )^{-1}  P^{\pi_{z_{l+1}}} \sm^{\nu_{l+1:n}} \\
    & \Longrightarrow \sm^{\nu_{l:n}}(\mathrm{d}s' \mid s, a)  = \frac{1-\gamma}{1-\beta_l}\sm^{\pi_{l}}_{\beta_l}(\mathrm{d}s' \mid s, a) + \frac{\gamma - \beta_l}{1-\beta_l} \int_{s_l} \sm^{\pi_{l}}_{\beta_l}(\mathrm{d}s_l \mid s, a) \pi_{\pi_{z_{l+1}}} (\mathrm{d}a_l \mid s_l) \sm^{\nu_{l+1:n}}( \mathrm{d}s' \mid s_l, a_l)\,. \\
\end{align*}
Using the induction hypothesis for $l+1$, we find: 
\begin{align*}
    & \sm^{\nu_{l:n}}(\mathrm{d}s' \mid s, a) \\
    & = \frac{1-\gamma}{1-\beta_l}\sm^{\pi_{l}}_{\beta_l}(\mathrm{d}s' \mid s, a) \\
    & + 
     \frac{\gamma - \beta_l}{1-\beta_l} \int_{s_l} \sm^{\pi_{l}}_{\beta_l}(\mathrm{d}s_l \mid s, a) \pi_{\pi_{z_{l+1}}} (\mathrm{d}a_l \mid s_l) \\
     & \times 
     \sum_{k=l+1}^n 
     \frac{1-\gamma}{1 - \beta_k} \left( \prod_{i=l+1}^{k-1} \frac{\gamma - \beta_i}{1-\beta_i} \right)  \int_{\substack{s_{l+1}, \ldots, s_{k-1} \\ a_{l+1}, \ldots, a_{k-1}}}    m_{\beta_{l+1}}^{\pi_{z_{l+1}}}(\mathrm{d}s_{l+1} \mid s, a) \pi_{z_{l+2}}(\mathrm{d}a_{l+1} \mid s_{l+1}) %
    \ldots m_{\beta_k}^{\pi_{z_k}}(\mathrm{d}s' \mid s_{k-1}, a_{k-1}) \\
    & = \frac{1-\gamma}{1-\beta_l}\sm^{\pi_{l}}_{\beta_l}(\mathrm{d}s' \mid s, a) \\
    & +
     \sum_{k=l+1}^n 
     \frac{1-\gamma}{1 - \beta_k} \left( \prod_{i=l}^{k-1} \frac{\gamma - \beta_i}{1-\beta_i} \right)  \int_{\substack{s_{l}, \ldots, s_{k-1} \\ a_{l}, \ldots, a_{k-1}}}    m_{\beta_l}^{\pi_{z_l}}(\mathrm{d}s_l \mid s, a) \pi_{z_{l+1}}(\mathrm{d}a_l \mid s_l) %
    \ldots m_{\beta_k}^{\pi_{z_k}}(\mathrm{d}s' \mid s_{k-1}, a_{k-1}) \\
    & = \sum_{k=l}^n 
     \frac{1-\gamma}{1 - \beta_k} \left( \prod_{i=l}^{k-1} \frac{\gamma - \beta_i}{1-\beta_i} \right)  \int_{\substack{s_{l}, \ldots, s_{k-1} \\ a_{l}, \ldots, a_{k-1}}}    m_{\beta_l}^{\pi_{z_l}}(\mathrm{d}s_l \mid s, a) \pi_{z_{l+1}}(\mathrm{d}a_l \mid s_l) %
    \ldots m_{\beta_k}^{\pi_{z_k}}(\mathrm{d}s' \mid s_{k-1}, a_{k-1})
\end{align*}
which shows the desired result.
\end{proof}

\subsection{Multi-Timescale Temporal Difference Flows with Horizon Consistency}\label{app:td-horizon-cons}
The results in this section generalize those found in \citet{farebrother2025temporal} to arbitrary mixture distributions.

\ifmetatemplate
\begin{metaframe}
\fi
\begin{lemma}\label{lem:mix_vf}
    Let $\{ v^i_t \}_{ i \in \llbracket N \rrbracket}$ a family of $N \in \mathbb{N}$ vector fields that generate the probability paths $\{ p^i_t \}_{i \in \llbracket N \rrbracket}$. Then, the mixture probability path $p_t = \sum_i \lambda_i p^i_t$, where $\{ \lambda_i \}_{i \in \llbracket N \rrbracket} \in [0,1]$ and $\sum_i \lambda_i = 1$ is generated by the vector field
    \begin{align}\label{eq:mix-vf}
        v_t := \frac{\sum_i \lambda_i p^i_t v^i_t }{\sum_i \lambda_i p^i_t}.
    \end{align}
\end{lemma}
\ifmetatemplate
\end{metaframe}
\fi

    \begin{proof}
        Since $v_i^t$ generates $p^i_t$, we know from the continuity equation that:
        \begin{align*}
            \forall i \in \llbracket N \rrbracket, \frac{\partial p^i_t}{\partial t} = \text{div} (p^i_t v^i_t) 
        \end{align*}
        where $\text{div}$ denotes the divergence operator. Then, by linearity of $\text{div}$,
        \begin{align*}
             \frac{\partial p_t }{\partial t}& = \frac{\partial \left( \sum_i \lambda_i p^i_t \right)}{\partial t} \\
             & = \sum_i \lambda_i \text{div} (p^i_t v^i_t)\\
            & = \text{div} \left( \sum_i \lambda_i p^i_t v^i_t  \right) \\
            & = \text{div} \left( \frac{\sum_i \lambda_i p^i_t v^i_t }{\sum_i \lambda_i p^i_t}  \sum_i \lambda_i p^i_t \right)
            \\ &= \text{div} (v_t p_t).
        \end{align*}
        Hence, $(v_t, p_t)$ satisfies the continuity equation, which implies that $v_t$ generates $p_t$.
\end{proof}

\ifmetatemplate
\begin{metaframe}
\fi
\begin{lemma}\label{lem:mix-vf-minimization}
    Let $\{ v^i_t \}_{ i \in \llbracket N \rrbracket}$ a family of $N \in \mathbb{N}$ vector fields that generate the probability paths $\{ p^i_t \}_{i \in \llbracket N \rrbracket}$. For $\lambda_i \in [0,1]$ such that $\sum_i \lambda_i =1$, the vector field $v_t = \frac{\sum_i \lambda_i p^i_t v^i_t }{\sum_i \lambda_i p^i_t}$ satisfies
    \begin{align*}
        v_t = \argmin_{v: \mathbb{R}^d\rightarrow \mathbb{R}^d} \Big \{ \sum_i \lambda_i \E_{x_t \sim p^1_t}{\| v_t(x_t) - v^i_t(x_t)\|^2} \Big \}.
    \end{align*}
\end{lemma}
\ifmetatemplate
\end{metaframe}
\fi

\begin{proof}
    Let $\ell_t(v) := \sum_i \lambda_i \E_{x_t \sim p^i_t}{\| v_t(x_t) - v^i_t(x_t)\|^2}$. The functional derivative of this quantity wrt $v$ evaluated at some point $x$ is
    \begin{align*}
        \nabla_v \ell_t(v)(x) = \sum_i \lambda_i p_i^t(x) ( v_t(x) - v^i_t(x)).
    \end{align*}
    Setting this to zero and solving for $v_t(x)$ yields the result.
\end{proof}

The consistency operator in equation~\eqref{eq: hc-bellman} combines three different distributions.
Lemmas~\ref{lem:mix_vf} and~\ref{lem:mix-vf-minimization}  indicate that we can construct distinct probability paths for each distribution as follows
\begin{enumerate}[leftmargin=*]
    \item For the first distribution \textit{i.e} $P(\cdot \mid s, a)$, We apply the standard Conditional Flow Matching (CFM) approach, where the probability path is defined as the marginal over a simple conditional path (specifically, we use the Optimal Transport (OT) path):
    \begin{equation}
        q_t(x | s,a) = \mathbb{E}_{S' \sim P(\cdot \mid s, a)} \left [ \mathcal{N}(x; t S', (1-t)^2) \right ]
    \end{equation}
    where $\mathcal{N}(x; t S', (1-t)^2)$ is the gaussian distribution of mean $t S'$ and variance $(1-t)^2$. This leads to the standard CFM objective:
    \begin{equation}
        \mathbb{E}_{\substack{t, (S, A, S') \\ X_0 \sim p_0, X_t = t S' + (1-t) X_0}} \left[ \Big \| v_t(X_t \mid S, A; \theta, \gamma) - (S' - X_0) \Big \|^2 \right] 
    \end{equation}
    \item For the second distribution, \textit{i.e.,} $\mathbb{E}_{\substack{S' \sim P(\cdot \mid s, a), A' \sim \pi(\cdot \mid S')} }\left [ \sm^\pi_\beta(\cdot \mid S', A')\right ] = (P^\pi \sm^\pi_\beta)(\cdot \mid s, a)$, we leverage that $\sm^\pi_\beta$ is parametrized by a flow matching model to define the probability path: 
\begin{equation}
    q_t(x \mid s, a) = \mathbb{E}_{\substack{S' \sim P(\cdot \mid s, a), A' \sim \pi(\cdot \mid S')} }\left [ {p_0}_{\#}\psi_t(\cdot \mid S', A'; \theta, \beta ) \right ]
\end{equation}
$q_t$ is a valid probability path, satisfying the boundary conditions: $q_0(x \mid s, a) = \mathbb{E}_{\substack{S' \sim P(\cdot \mid s, a), A' \sim \pi(\cdot \mid S')} }\left [ p_0(x) \right] = p_0(x)$ and $q_1 = P^\pi \sm_\beta^\pi$.
$q_t$ can be interpreted as aggregation of conditional paths ${p_0}_{\#}\psi_t(\cdot \mid S', A', \beta, \theta)$, for which we have access to their vector field $v_t(X_t \mid S', A';\theta, \beta)$. Using the equivalence between marginal flow matching and conditional flow matching~\citep{lipman2022flow}, we arrive at the following objective.
\begin{equation}
    \mathbb{E}_{\substack{
    t,
    (S, A, S'),A'\sim \pi(\cdot\mid S') \\
    X_t \sim {p_0}_{\#}\psi_t(\cdot \mid S', A', \beta; \bar{\theta} )}} \left[ \Big \| v_t(X_t \mid S, A,\gamma;\theta) - v_t(X_t \mid S', A', \beta;\bar{\theta}) \Big \|^2 \right]
\end{equation}
\item Similarly, for the third distribution
$\mathbb{E}_{\substack{S' \sim P(\cdot \mid s, a), A' \sim \pi(\cdot \mid S') \\
    S'' \sim \sm^\pi_\beta(\cdot \mid S', A'), A'' \sim \pi(\cdot \mid S'') } }\left [ m^\pi_\gamma(\cdot \mid S'', A'')\right ] $, we again leverage that $\sm^\pi_\gamma$ is parametrized by flow matching model to define the following probability path: 
    \begin{equation}
        q_t(x \mid s, a) = \mathbb{E}_{\substack{S' \sim P(\cdot \mid s, a), A' \sim \pi(\cdot \mid S') \\
    S'' \sim {p_0}_{\#}\psi_1(\cdot \mid S', A', \beta; \bar{\theta} ), A'' \sim \pi(\cdot \mid S'') } }\left [ {p_0}_{\#}\psi_t(\cdot \mid S', A'; \theta,\gamma )\right ] 
    \end{equation}
$q_t$ can be interpreted as aggregation of conditional paths ${p_0}_{\#}\psi_t(\cdot \mid S', A', \gamma, \theta)$, for which we have access to their vector field $v_t(X_t \mid S', A', \gamma; \theta)$. Using the equivalence between marginal flow matching and conditional flow matching~\citep{lipman2022flow}, we arrive at the following objective.

\begin{equation}
    \mathbb{E}_{\substack{t, (S, A, S'), A' \sim \pi(\cdot | S') \\
    S'' \sim {p_0}_{\#}\psi_1(\cdot \mid S', A', \beta; \bar{\theta} ) \\
    X_t \sim {p_0}_{\#}\psi_t(\cdot \mid S'', A'', \gamma; \bar{\theta} ) }} \left[ \Big \| v_t(X_t \mid S, A, \gamma;\theta) - v_t(X_t \mid S'', A'', \gamma; \bar{\theta}) \Big \|^2 \right]
\end{equation}

\end{enumerate}

\clearpage
\section{Additional Results}\label{app:sec:results}

\subsection{Compositional Planning / Zero-Shot Results}
We report the full planning results here. Table~\ref{tab:per_task_success} summarizes the success rate for each task. We evaluate 7 domains with 5 tasks per domain, for a total of 35 tasks. For each task we evaluate the base policy and \textsc{CompPlan} by rolling out 10 trajectories. We report standard deviation over the 3 seeds used for GHM training, for the base policy we do not have multiple seeds. We see that \textsc{CompPlan} is better than the base policy in almost all the tasks.
Figure~\ref{fig:planning_vs_zeroshot_success} shows a bar plot comparing the domain-averaged success rates of \textsc{CompPlan} and the zero-shot baseline.

\begin{table}[ht]
\setlength{\aboverulesep}{0pt}
\setlength{\belowrulesep}{0pt}
\centering
\caption{Success rate ($\uparrow$) per task for base policies $\pi_g$ (Zero Shot) and \textsc{CompPlan}. Mean and standard deviation reported over 3 seeds. \mhl{Blue} and \mdecc{red} denote an increase and decrease w.r.t. zero-shot with \mdec{gray} indicating no significant~difference.}
\label{tab:per_task_success}
\renewcommand{\arraystretch}{1.3}
\begin{adjustbox}{width=\textwidth}
\begin{tabular}{rc|ll|ll|ll|ll|ll}
\toprule
\multirow{2}{*}{\textbf{Domain}} & \multirow{2}{*}{\textbf{Task}} & \multicolumn{2}{c}{\textbf{CRL}}& \multicolumn{2}{c}{\textbf{GC-1S}}& \multicolumn{2}{c}{\textbf{GC-BC}}& \multicolumn{2}{c}{\textbf{GC-TD3}}& \multicolumn{2}{c}{\textbf{HFBC}}\\
 & &\textbf{Zero Shot} & \textbf{\textsc{CompPlan}} & \textbf{Zero Shot} & \textbf{\textsc{CompPlan}} & \textbf{Zero Shot} & \textbf{\textsc{CompPlan}} & \textbf{Zero Shot} & \textbf{\textsc{CompPlan}} & \textbf{Zero Shot} & \textbf{\textsc{CompPlan}}\\
\midrule
\multirow{5}{*}{\textsc{antmaze-medium}} & 1 & 0.950 & \cellcolor{hl}{0.970 {\footnotesize (0.030)}} & 0.400 & \cellcolor{hl}{0.930 {\footnotesize (0.070)}} & 0.400 & \cellcolor{hl}{0.870 {\footnotesize (0.090)}} & 0.650 & \cellcolor{hl}{0.800 {\footnotesize (0.060)}} & 0.900 & \cellcolor{red!25}{0.870 {\footnotesize (0.030)}}\\
 & 2 & 0.950 & \cellcolor{red!25}{0.900 {\footnotesize (0.100)}} & 0.800 & \cellcolor{hl}{1.000 {\footnotesize (0.000)}} & 0.550 & \cellcolor{hl}{0.830 {\footnotesize (0.090)}} & 0.900 & \cellcolor{red!25}{0.530 {\footnotesize (0.120)}} & 0.900 & \cellcolor{hl}{0.970 {\footnotesize (0.030)}}\\
 & 3 & 0.650 & \cellcolor{hl}{0.970 {\footnotesize (0.030)}} & 0.100 & \cellcolor{hl}{0.500 {\footnotesize (0.120)}} & 0.600 & \cellcolor{hl}{0.900 {\footnotesize (0.060)}} & 0.450 & \cellcolor{hl}{0.770 {\footnotesize (0.030)}} & 1.000 & \cellcolor{red!25}{0.970 {\footnotesize (0.030)}}\\
 & 4 & 0.900 & \cellcolor{hl}{1.000 {\footnotesize (0.000)}} & 0.650 & \cellcolor{hl}{0.930 {\footnotesize (0.030)}} & 0.100 & \cellcolor{hl}{0.770 {\footnotesize (0.090)}} & 0.550 & \cellcolor{hl}{0.600 {\footnotesize (0.100)}} & 1.000 & \cellcolor{red!25}{0.970 {\footnotesize (0.030)}}\\
 & 5 & 0.950 & \cellcolor{hl}{1.000 {\footnotesize (0.000)}} & 0.850 & \cellcolor{hl}{0.970 {\footnotesize (0.030)}} & 0.800 & \cellcolor{hl}{0.870 {\footnotesize (0.090)}} & 0.700 & \cellcolor{red!25}{0.570 {\footnotesize (0.070)}} & 0.900 & \cellcolor{hl}{0.930 {\footnotesize (0.030)}}\\
\midrule
\multirow{5}{*}{\textsc{antmaze-large}} & 1 & 0.800 & \cellcolor{hl}{0.830 {\footnotesize (0.030)}} & 0.100 & \cellcolor{hl}{0.770 {\footnotesize (0.070)}} & 0.150 & \cellcolor{hl}{0.570 {\footnotesize (0.070)}} & 0.300 & \cellcolor{hl}{0.500 {\footnotesize (0.120)}} & 0.800 & \cellcolor{hl}{0.970 {\footnotesize (0.030)}}\\
 & 2 & 0.600 & \cellcolor{hl}{0.770 {\footnotesize (0.030)}} & 0.150 & \cellcolor{hl}{0.630 {\footnotesize (0.120)}} & 0.100 & \cellcolor{hl}{0.730 {\footnotesize (0.030)}} & 0.100 & \cellcolor{hl}{0.530 {\footnotesize (0.090)}} & 0.600 & \cellcolor{hl}{0.800 {\footnotesize (0.000)}}\\
 & 3 & 0.850 & \cellcolor{hl}{0.930 {\footnotesize (0.030)}} & 0.800 & \cellcolor{hl}{0.830 {\footnotesize (0.090)}} & 0.550 & \cellcolor{hl}{0.900 {\footnotesize (0.060)}} & 0.750 & \cellcolor{red!25}{0.500 {\footnotesize (0.100)}} & 1.000 & \cellcolor{red!25}{0.970 {\footnotesize (0.030)}}\\
 & 4 & 0.950 & \cellcolor{hl}{1.000 {\footnotesize (0.000)}} & 0.000 & \cellcolor{hl}{0.430 {\footnotesize (0.030)}} & 0.000 & \cellcolor{hl}{0.670 {\footnotesize (0.090)}} & 0.000 & \cellcolor{hl}{0.400 {\footnotesize (0.060)}} & 0.900 & \cellcolor{gray!25}{0.900 {\footnotesize (0.060)}}\\
 & 5 & 1.000 & \cellcolor{red!25}{0.970 {\footnotesize (0.030)}} & 0.000 & \cellcolor{hl}{0.370 {\footnotesize (0.120)}} & 0.100 & \cellcolor{hl}{0.770 {\footnotesize (0.090)}} & 0.000 & \cellcolor{hl}{0.470 {\footnotesize (0.090)}} & 0.600 & \cellcolor{hl}{0.970 {\footnotesize (0.030)}}\\
\midrule
\multirow{5}{*}{\textsc{antmaze-giant}} & 1 & 0.000 & \cellcolor{hl}{0.070 {\footnotesize (0.030)}} & 0.000 & \cellcolor{gray!25}{0.000 {\footnotesize (0.000)}} & 0.000 & \cellcolor{gray!25}{0.000 {\footnotesize (0.000)}} & 0.000 & \cellcolor{gray!25}{0.000 {\footnotesize (0.000)}} & 0.400 & \cellcolor{hl}{0.730 {\footnotesize (0.090)}}\\
 & 2 & 0.000 & \cellcolor{hl}{0.670 {\footnotesize (0.070)}} & 0.000 & \cellcolor{gray!25}{0.000 {\footnotesize (0.000)}} & 0.000 & \cellcolor{gray!25}{0.000 {\footnotesize (0.000)}} & 0.000 & \cellcolor{hl}{0.030 {\footnotesize (0.030)}} & 0.300 & \cellcolor{hl}{0.830 {\footnotesize (0.030)}}\\
 & 3 & 0.000 & \cellcolor{hl}{0.100 {\footnotesize (0.060)}} & 0.000 & \cellcolor{gray!25}{0.000 {\footnotesize (0.000)}} & 0.000 & \cellcolor{gray!25}{0.000 {\footnotesize (0.000)}} & 0.000 & \cellcolor{gray!25}{0.000 {\footnotesize (0.000)}} & 0.200 & \cellcolor{hl}{0.800 {\footnotesize (0.060)}}\\
 & 4 & 0.500 & \cellcolor{red!25}{0.230 {\footnotesize (0.070)}} & 0.000 & \cellcolor{gray!25}{0.000 {\footnotesize (0.000)}} & 0.000 & \cellcolor{gray!25}{0.000 {\footnotesize (0.000)}} & 0.000 & \cellcolor{gray!25}{0.000 {\footnotesize (0.000)}} & 0.500 & \cellcolor{hl}{0.730 {\footnotesize (0.030)}}\\
 & 5 & 0.300 & \cellcolor{hl}{0.400 {\footnotesize (0.120)}} & 0.000 & \cellcolor{hl}{0.100 {\footnotesize (0.000)}} & 0.000 & \cellcolor{hl}{0.130 {\footnotesize (0.030)}} & 0.000 & \cellcolor{hl}{0.030 {\footnotesize (0.030)}} & 0.700 & \cellcolor{hl}{0.830 {\footnotesize (0.030)}}\\
\midrule
\multirow{5}{*}{\textsc{cube-1}} & 1 & 0.200 & \cellcolor{hl}{0.930 {\footnotesize (0.030)}} & 0.300 & \cellcolor{hl}{0.470 {\footnotesize (0.030)}} & 1.000 & \cellcolor{gray!25}{1.000 {\footnotesize (0.000)}} & 0.500 & \cellcolor{hl}{0.900 {\footnotesize (0.060)}} & 0.800 & \cellcolor{hl}{0.970 {\footnotesize (0.030)}}\\
 & 2 & 0.100 & \cellcolor{hl}{0.830 {\footnotesize (0.090)}} & 0.400 & \cellcolor{hl}{0.470 {\footnotesize (0.030)}} & 0.950 & \cellcolor{hl}{1.000 {\footnotesize (0.000)}} & 0.850 & \cellcolor{hl}{1.000 {\footnotesize (0.000)}} & 0.700 & \cellcolor{hl}{1.000 {\footnotesize (0.000)}}\\
 & 3 & 0.500 & \cellcolor{hl}{0.900 {\footnotesize (0.000)}} & 0.300 & \cellcolor{hl}{0.900 {\footnotesize (0.060)}} & 0.950 & \cellcolor{hl}{1.000 {\footnotesize (0.000)}} & 0.600 & \cellcolor{hl}{0.900 {\footnotesize (0.000)}} & 0.900 & \cellcolor{hl}{1.000 {\footnotesize (0.000)}}\\
 & 4 & 0.250 & \cellcolor{hl}{0.930 {\footnotesize (0.030)}} & 0.500 & \cellcolor{hl}{0.770 {\footnotesize (0.030)}} & 1.000 & \cellcolor{gray!25}{1.000 {\footnotesize (0.000)}} & 0.550 & \cellcolor{hl}{0.900 {\footnotesize (0.060)}} & 0.700 & \cellcolor{hl}{0.970 {\footnotesize (0.030)}}\\
 & 5 & 0.350 & \cellcolor{hl}{0.700 {\footnotesize (0.120)}} & 0.350 & \cellcolor{hl}{0.700 {\footnotesize (0.100)}} & 0.600 & \cellcolor{hl}{0.970 {\footnotesize (0.030)}} & 0.400 & \cellcolor{hl}{0.870 {\footnotesize (0.070)}} & 0.900 & \cellcolor{gray!25}{0.900 {\footnotesize (0.060)}}\\
\midrule
\multirow{5}{*}{\textsc{cube-2}} & 1 & 0.100 & \cellcolor{hl}{0.870 {\footnotesize (0.030)}} & 0.200 & \cellcolor{hl}{0.970 {\footnotesize (0.030)}} & 0.750 & \cellcolor{hl}{1.000 {\footnotesize (0.000)}} & 0.300 & \cellcolor{hl}{0.970 {\footnotesize (0.030)}} & 1.000 & \cellcolor{gray!25}{1.000 {\footnotesize (0.000)}}\\
 & 2 & 0.000 & \cellcolor{hl}{0.500 {\footnotesize (0.060)}} & 0.200 & \cellcolor{hl}{0.500 {\footnotesize (0.170)}} & 0.000 & \cellcolor{hl}{1.000 {\footnotesize (0.000)}} & 0.150 & \cellcolor{hl}{0.930 {\footnotesize (0.070)}} & 0.900 & \cellcolor{hl}{0.930 {\footnotesize (0.030)}}\\
 & 3 & 0.000 & \cellcolor{hl}{0.570 {\footnotesize (0.030)}} & 0.050 & \cellcolor{hl}{0.600 {\footnotesize (0.100)}} & 0.000 & \cellcolor{hl}{1.000 {\footnotesize (0.000)}} & 0.100 & \cellcolor{hl}{0.870 {\footnotesize (0.030)}} & 0.900 & \cellcolor{hl}{1.000 {\footnotesize (0.000)}}\\
 & 4 & 0.000 & \cellcolor{hl}{0.070 {\footnotesize (0.030)}} & 0.000 & \cellcolor{hl}{0.400 {\footnotesize (0.170)}} & 0.000 & \cellcolor{hl}{0.870 {\footnotesize (0.070)}} & 0.000 & \cellcolor{hl}{0.430 {\footnotesize (0.130)}} & 0.300 & \cellcolor{red!25}{0.200 {\footnotesize (0.060)}}\\
 & 5 & 0.000 & \cellcolor{hl}{0.500 {\footnotesize (0.100)}} & 0.050 & \cellcolor{hl}{0.370 {\footnotesize (0.070)}} & 0.000 & \cellcolor{hl}{0.970 {\footnotesize (0.030)}} & 0.050 & \cellcolor{hl}{0.900 {\footnotesize (0.060)}} & 0.700 & \cellcolor{gray!25}{0.700 {\footnotesize (0.100)}}\\
\midrule
\multirow{5}{*}{\textsc{cube-3}} & 1 & 0.050 & \cellcolor{hl}{1.000 {\footnotesize (0.000)}} & 0.050 & \cellcolor{hl}{0.930 {\footnotesize (0.030)}} & 0.450 & \cellcolor{hl}{1.000 {\footnotesize (0.000)}} & 0.400 & \cellcolor{hl}{1.000 {\footnotesize (0.000)}} & 0.900 & \cellcolor{hl}{1.000 {\footnotesize (0.000)}}\\
 & 2 & 0.000 & \cellcolor{hl}{0.970 {\footnotesize (0.030)}} & 0.000 & \cellcolor{hl}{1.000 {\footnotesize (0.000)}} & 0.000 & \cellcolor{hl}{1.000 {\footnotesize (0.000)}} & 0.050 & \cellcolor{hl}{1.000 {\footnotesize (0.000)}} & 0.900 & \cellcolor{hl}{1.000 {\footnotesize (0.000)}}\\
 & 3 & 0.000 & \cellcolor{hl}{0.930 {\footnotesize (0.030)}} & 0.000 & \cellcolor{hl}{0.870 {\footnotesize (0.090)}} & 0.000 & \cellcolor{hl}{1.000 {\footnotesize (0.000)}} & 0.150 & \cellcolor{hl}{0.930 {\footnotesize (0.070)}} & 0.800 & \cellcolor{hl}{0.970 {\footnotesize (0.030)}}\\
 & 4 & 0.000 & \cellcolor{hl}{0.270 {\footnotesize (0.090)}} & 0.000 & \cellcolor{hl}{0.300 {\footnotesize (0.060)}} & 0.000 & \cellcolor{hl}{0.700 {\footnotesize (0.060)}} & 0.000 & \cellcolor{hl}{0.700 {\footnotesize (0.060)}} & 0.300 & \cellcolor{hl}{0.370 {\footnotesize (0.120)}}\\
 & 5 & 0.000 & \cellcolor{hl}{0.470 {\footnotesize (0.030)}} & 0.000 & \cellcolor{hl}{0.230 {\footnotesize (0.030)}} & 0.000 & \cellcolor{hl}{0.900 {\footnotesize (0.000)}} & 0.000 & \cellcolor{hl}{0.530 {\footnotesize (0.170)}} & 0.300 & \cellcolor{hl}{0.800 {\footnotesize (0.000)}}\\
\midrule
\multirow{5}{*}{\textsc{cube-4}} & 1 & 0.000 & \cellcolor{hl}{0.530 {\footnotesize (0.070)}} & 0.050 & \cellcolor{hl}{0.970 {\footnotesize (0.030)}} & 0.000 & \cellcolor{hl}{1.000 {\footnotesize (0.000)}} & 0.000 & \cellcolor{hl}{1.000 {\footnotesize (0.000)}} & 0.600 & \cellcolor{hl}{1.000 {\footnotesize (0.000)}}\\
 & 2 & 0.000 & \cellcolor{hl}{0.870 {\footnotesize (0.090)}} & 0.000 & \cellcolor{hl}{0.900 {\footnotesize (0.060)}} & 0.000 & \cellcolor{hl}{1.000 {\footnotesize (0.000)}} & 0.000 & \cellcolor{hl}{0.970 {\footnotesize (0.030)}} & 0.400 & \cellcolor{hl}{0.830 {\footnotesize (0.120)}}\\
 & 3 & 0.000 & \cellcolor{hl}{0.270 {\footnotesize (0.120)}} & 0.000 & \cellcolor{hl}{0.800 {\footnotesize (0.000)}} & 0.000 & \cellcolor{hl}{1.000 {\footnotesize (0.000)}} & 0.000 & \cellcolor{hl}{0.470 {\footnotesize (0.030)}} & 0.200 & \cellcolor{hl}{0.770 {\footnotesize (0.030)}}\\
 & 4 & 0.000 & \cellcolor{hl}{0.030 {\footnotesize (0.030)}} & 0.000 & \cellcolor{hl}{0.100 {\footnotesize (0.060)}} & 0.000 & \cellcolor{hl}{0.170 {\footnotesize (0.170)}} & 0.000 & \cellcolor{hl}{0.170 {\footnotesize (0.120)}} & 0.000 & \cellcolor{hl}{0.200 {\footnotesize (0.100)}}\\
 & 5 & 0.000 & \cellcolor{hl}{0.230 {\footnotesize (0.130)}} & 0.000 & \cellcolor{hl}{0.230 {\footnotesize (0.070)}} & 0.000 & \cellcolor{hl}{0.630 {\footnotesize (0.130)}} & 0.000 & \cellcolor{hl}{0.230 {\footnotesize (0.030)}} & 0.000 & \cellcolor{hl}{0.570 {\footnotesize (0.130)}}\\
\bottomrule
\end{tabular}
\end{adjustbox}
\end{table}

\begin{figure}[H]
    \centering
    \includegraphics[width=\textwidth]{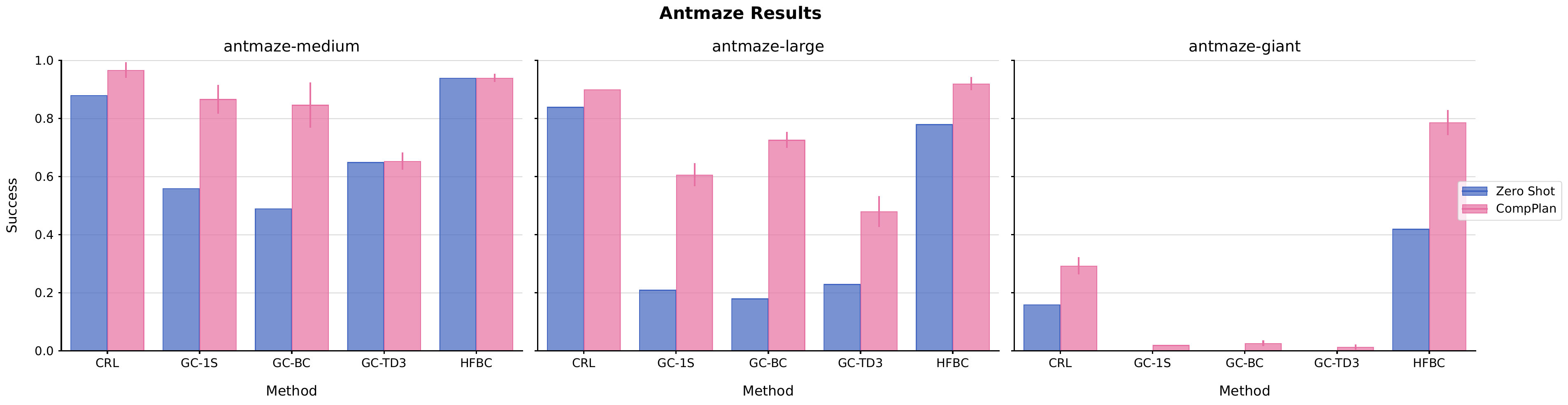}
    \includegraphics[width=\textwidth]{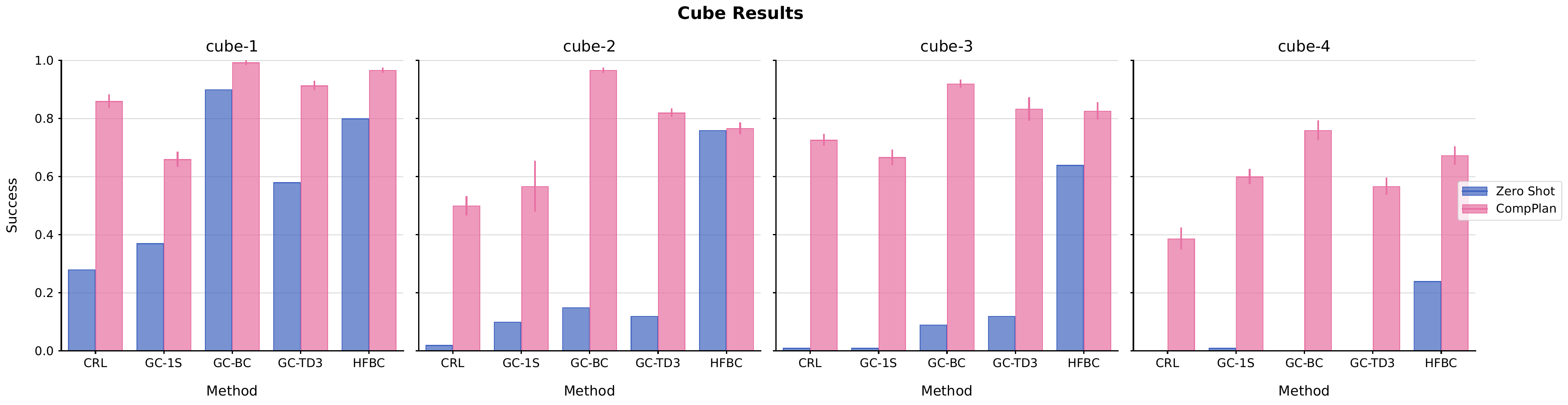}
    \caption{Success rate ($\uparrow$) of base policies $\pi_g$ (Zero Shot) and compositional planning as in Equation~\eqref{eq:planning.} with GHMs (\textsc{CompPlan}, ours) averaged over tasks.}
    \label{fig:planning_vs_zeroshot_success}
\end{figure}

\clearpage
\subsection{Compositional Planning / GPI / Action Planning Results}

\begin{figure*}[h!]
    \centering
    \includegraphics[width=\linewidth]{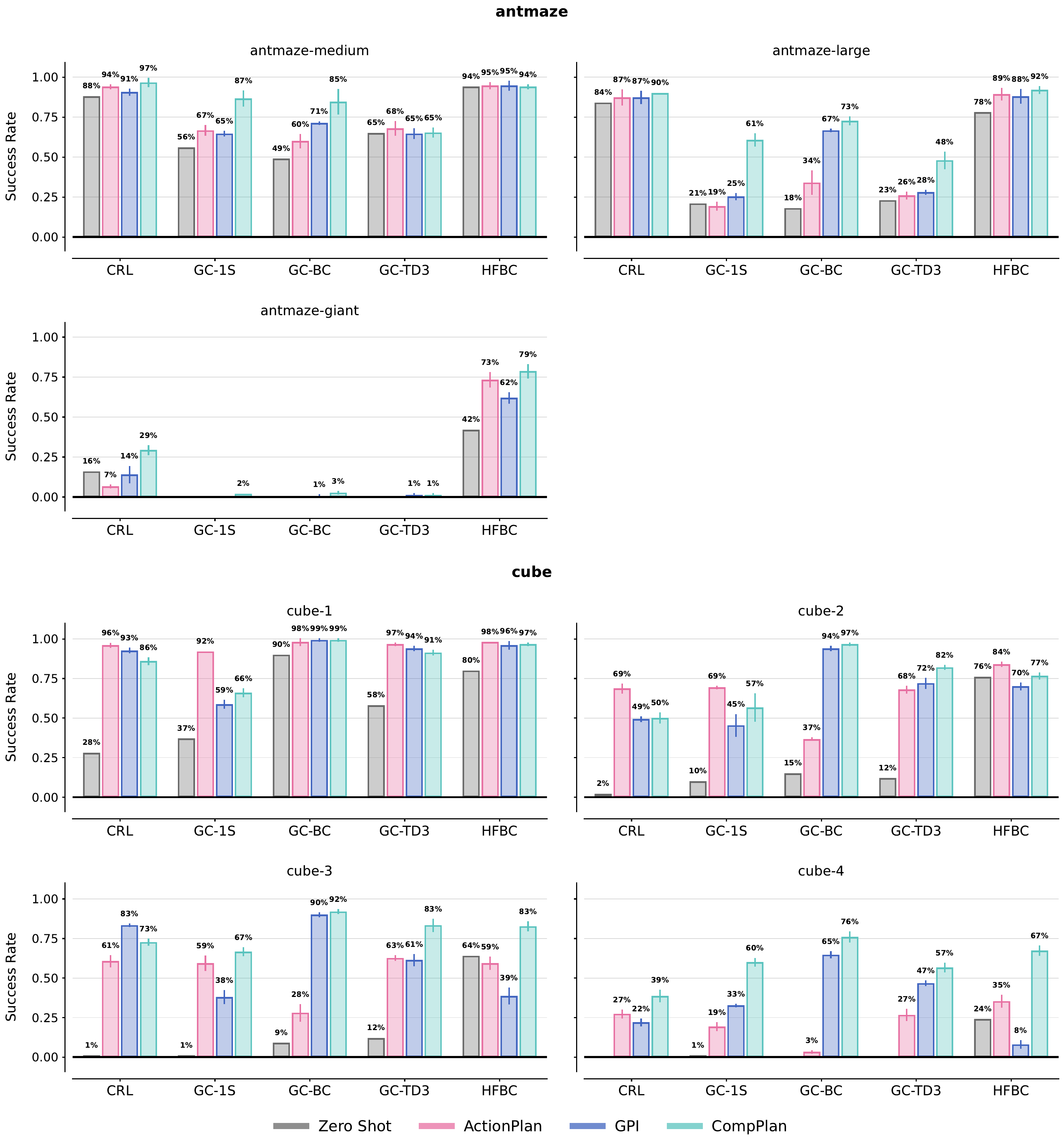}
    \caption{Success rate ($\uparrow$) of planning vs. zero shot.
    We consider action-level planning (\textsc{ActionPlan}) with a world model; generalized policy improvement (GPI) and compositional planning (\textsc{CompPlan}; ours) with GHMs.}
    \label{fig:wm.gpi.ghm.planning.absolute.all}
\end{figure*}

\begin{figure*}[h!]
    \centering
    \includegraphics[width=\linewidth]{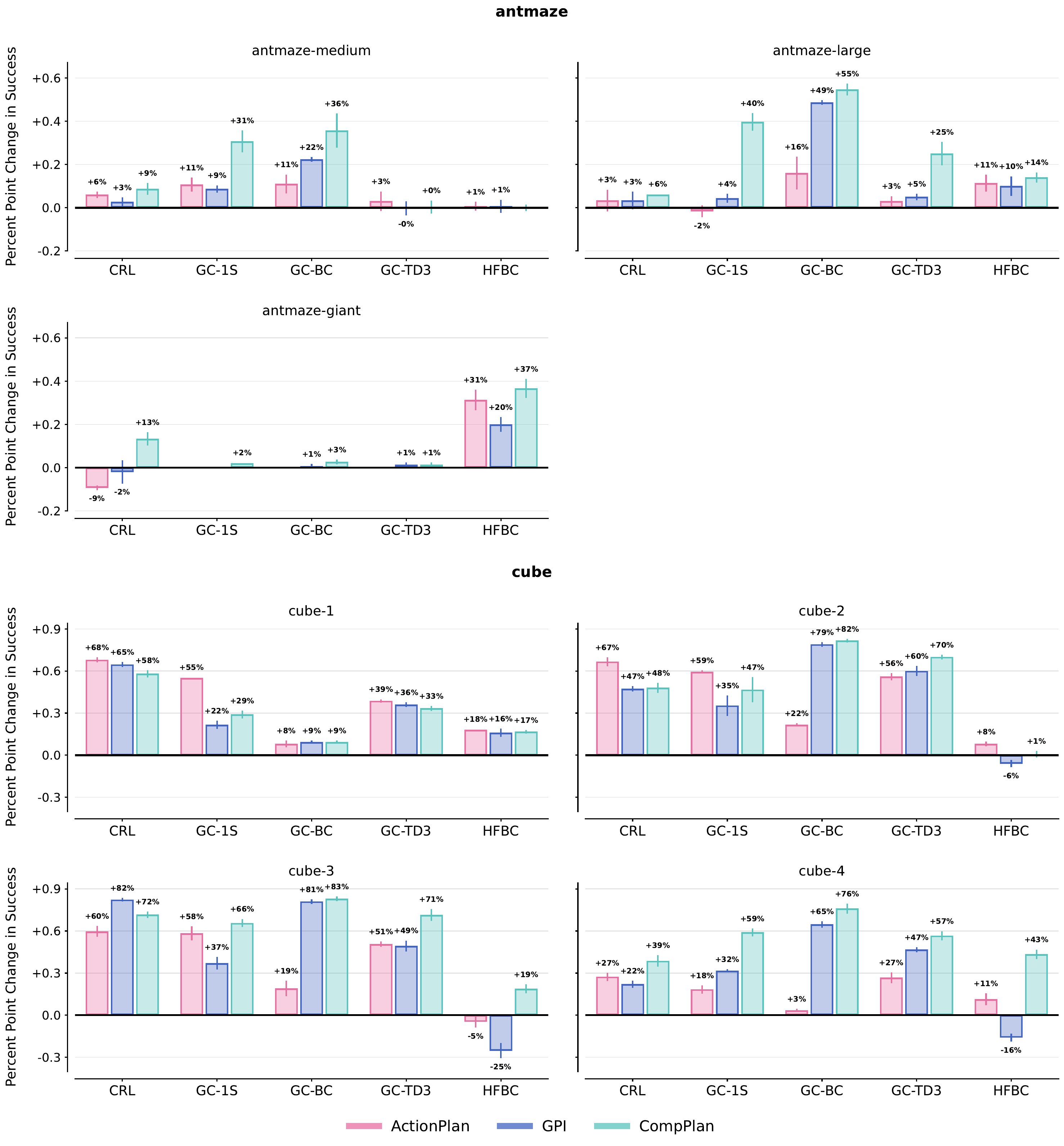}
    \caption{Percent point change ($\uparrow$) of planning over zero shot.
    We consider action-level planning (\textsc{ActionPlan}) with a world model; generalized policy improvement (GPI) and compositional planning (\textsc{CompPlan}; ours) with GHMs.}
    \label{fig:wm.gpi.ghm.planning.all}
\end{figure*}

\clearpage

\subsection{Ablation on the Planning Frequency}
In this section, we investigate whether the planning cost can be amortized over time. Specifically, we compare the performance of planning at each time step with planning every N steps. In the latter case, we execute the action maximizing the \textsc{CompPlan} objective for the first step and then follow the policy $\pi_{z_1}$ (i.e., the trajectory is $s, a^\star, s_1, \pi_{z_1}(s_1), s_2, \ldots, s_{N-1}, \pi_{z_1}(s_{N-1}), s_N$).
Figure~\ref{fig:replan_every_ablation} reports the average success rate for each domain and method. On average, planning at every step leads to about a 20\% improvement compared to planning every 5 steps. This is mostly due to a few cases—most notably, CRL policies in Cube—in which planning every 5 steps results in a large performance drop or complete unlearning. In the majority of experiments, planning every 5 steps does not substantially degrade overall performance, while significantly reducing planning time. Overall, this is a lever that can be used to trade off speed and performance. Finally, Table~\ref{tab:replan.ablation} provides a tabular summary of these results.

\begin{figure}[ht]
    \centering
    \includegraphics[width=\textwidth]
    {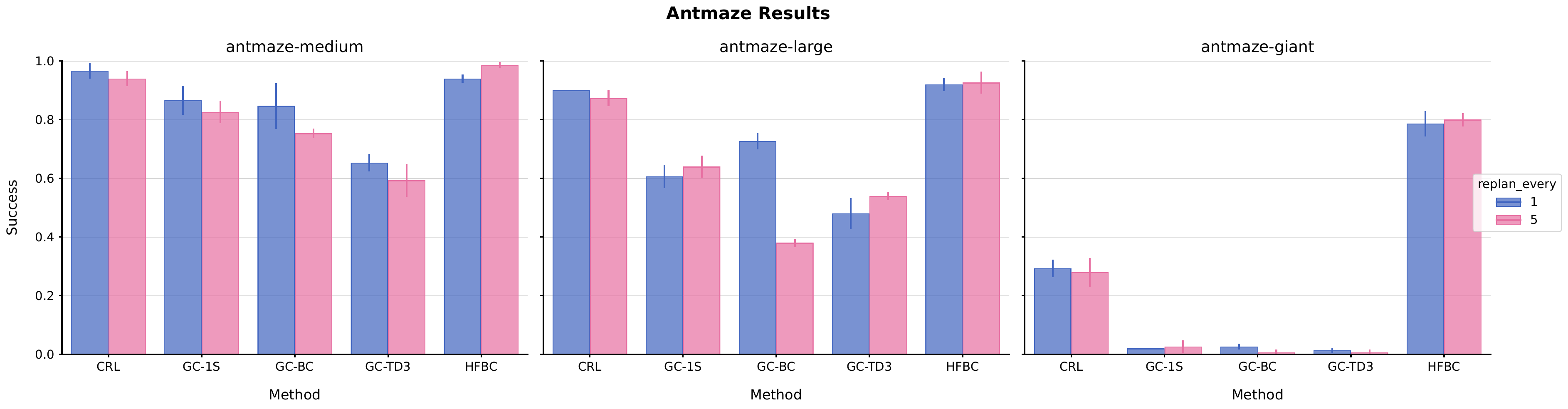}
    \includegraphics[width=\textwidth]
    {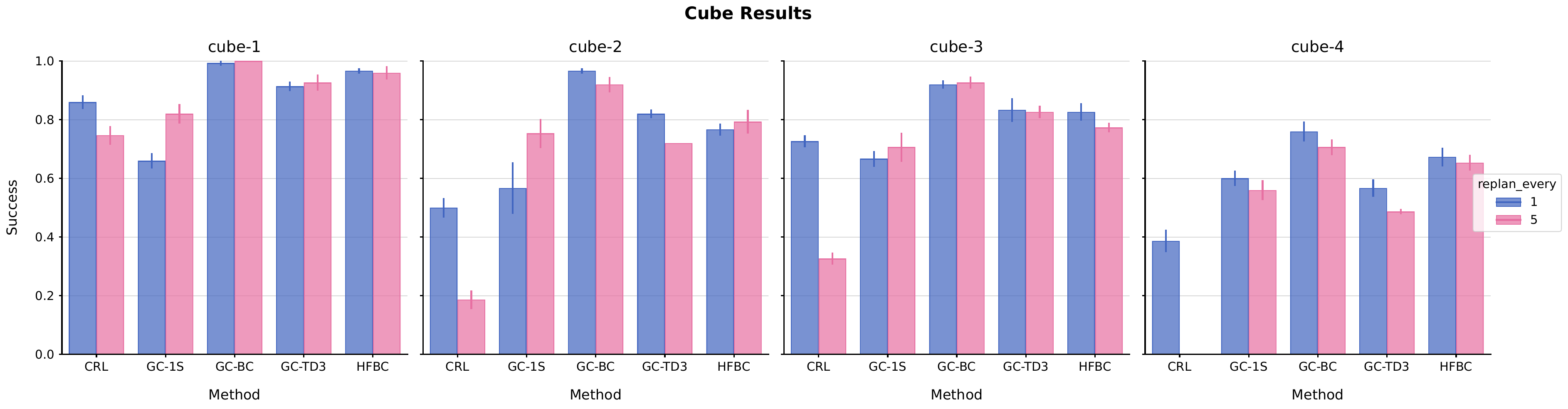}
    \caption{Success rate ($\uparrow$) of \textsc{CompPlan} with different base policies when planning at every step of every 5 steps. We report mean and standard deviation 3 seeds for GHM training.}
    \label{fig:replan_every_ablation}
\end{figure}
\begin{table}[ht]
\setlength{\aboverulesep}{0pt}
\setlength{\belowrulesep}{0pt}
\centering
\caption{Success rate ($\uparrow$) of \textsc{CompPlan} with different base policies when replanning every $1$ or $5$ steps. We report mean and standard deviation over 3 seeds.
Best method highlighted in \mhl{blue}; \mdec{gray} indicates no significant difference.}
\label{tab:replan.ablation}
\renewcommand{\arraystretch}{1.3}
\begin{adjustbox}{width=\textwidth}
\begin{tabular}{l|ll|ll|ll|ll|ll|}
\toprule
\multirow[c]{2}{*}{\textbf{Domain}} & \multicolumn{2}{c}{\textbf{CRL}}& \multicolumn{2}{c}{\textbf{GC-1S}}& \multicolumn{2}{c}{\textbf{GC-BC}}& \multicolumn{2}{c}{\textbf{GC-TD3}}& \multicolumn{2}{c}{\textbf{HFBC}}\\ & \multicolumn{1}{c}{1 Step} & \multicolumn{1}{c|}{5 Steps} & \multicolumn{1}{c}{1 Step} & \multicolumn{1}{c|}{5 Steps} & \multicolumn{1}{c}{1 Step} & \multicolumn{1}{c|}{5 Steps} & \multicolumn{1}{c}{1 Step} & \multicolumn{1}{c|}{5 Steps} & \multicolumn{1}{c}{1 Step} & \multicolumn{1}{c}{5 Steps}\\
\midrule
\textsc{antmaze-medium} & \cellcolor{hl}{0.97 {\footnotesize (0.02)}} & 0.94 {\footnotesize (0.02)} & \cellcolor{hl}{0.87 {\footnotesize (0.05)}} & 0.83 {\footnotesize (0.04)} & \cellcolor{hl}{0.85 {\footnotesize (0.08)}} & 0.75 {\footnotesize (0.01)} & \cellcolor{hl}{0.65 {\footnotesize (0.03)}} & 0.59 {\footnotesize (0.05)} & 0.94{\footnotesize (0.01)} & \cellcolor{hl}{0.99 {\footnotesize (0.01)}}\\
\hhline{~|----------}
\textsc{antmaze-large} & \cellcolor{hl}{0.90 {\footnotesize (0.00)}} & 0.87 {\footnotesize (0.02)} & 0.61{\footnotesize (0.04)} & \cellcolor{hl}{0.64 {\footnotesize (0.03)}} & \cellcolor{hl}{0.73 {\footnotesize (0.02)}} & 0.38 {\footnotesize (0.01)} & 0.48{\footnotesize (0.05)} & \cellcolor{hl}{0.54 {\footnotesize (0.01)}} & 0.92{\footnotesize (0.02)} & \cellcolor{hl}{0.93 {\footnotesize (0.04)}}\\
\hhline{~|----------}
\textsc{antmaze-giant} & \cellcolor{hl}{0.29 {\footnotesize (0.03)}} & 0.28 {\footnotesize (0.05)} & 0.02{\footnotesize (0.00)} & \cellcolor{hl}{0.03 {\footnotesize (0.02)}} & \cellcolor{hl}{0.03 {\footnotesize (0.01)}} & 0.01 {\footnotesize (0.01)} & \cellcolor{gray!25}{0.01 {\footnotesize (0.01)}} & \cellcolor{gray!25}{0.01 {\footnotesize (0.01)}} & 0.79{\footnotesize (0.04)} & \cellcolor{hl}{0.80 {\footnotesize (0.02)}}\\
\hhline{~|----------}
\textsc{cube-1} & \cellcolor{hl}{0.86 {\footnotesize (0.02)}} & 0.75 {\footnotesize (0.03)} & 0.66{\footnotesize (0.02)} & \cellcolor{hl}{0.82 {\footnotesize (0.03)}} & 0.99{\footnotesize (0.01)} & \cellcolor{hl}{1.00 {\footnotesize (0.00)}} & 0.91{\footnotesize (0.01)} & \cellcolor{hl}{0.93 {\footnotesize (0.02)}} & \cellcolor{hl}{0.97 {\footnotesize (0.01)}} & 0.96 {\footnotesize (0.02)}\\
\hhline{~|----------}
\textsc{cube-2} & \cellcolor{hl}{0.50 {\footnotesize (0.03)}} & 0.19 {\footnotesize (0.03)} & 0.57{\footnotesize (0.09)} & \cellcolor{hl}{0.75 {\footnotesize (0.05)}} & \cellcolor{hl}{0.97 {\footnotesize (0.01)}} & 0.92 {\footnotesize (0.02)} & \cellcolor{hl}{0.82 {\footnotesize (0.01)}} & 0.72 {\footnotesize (0.00)} & 0.77{\footnotesize (0.02)} & \cellcolor{hl}{0.79 {\footnotesize (0.04)}}\\
\hhline{~|----------}
\textsc{cube-3} & \cellcolor{hl}{0.73 {\footnotesize (0.02)}} & 0.33 {\footnotesize (0.02)} & 0.67{\footnotesize (0.02)} & \cellcolor{hl}{0.71 {\footnotesize (0.05)}} & 0.92{\footnotesize (0.01)} & \cellcolor{hl}{0.93 {\footnotesize (0.02)}} & \cellcolor{gray!25}{0.83 {\footnotesize (0.02)}} & \cellcolor{gray!25}{0.83 {\footnotesize (0.02)}} & \cellcolor{hl}{0.83 {\footnotesize (0.03)}} & 0.77 {\footnotesize (0.01)}\\
\hhline{~|----------}
\textsc{cube-4} & \cellcolor{hl}{0.39 {\footnotesize (0.04)}} & 0.00 {\footnotesize (0.00)} & \cellcolor{hl}{0.60 {\footnotesize (0.02)}} & 0.56 {\footnotesize (0.03)} & \cellcolor{hl}{0.76 {\footnotesize (0.03)}} & 0.71 {\footnotesize (0.02)} & \cellcolor{hl}{0.57 {\footnotesize (0.03)}} & 0.49 {\footnotesize (0.01)} & \cellcolor{hl}{0.67 {\footnotesize (0.03)}} & 0.65 {\footnotesize (0.02)}\\
\bottomrule
\end{tabular}
\end{adjustbox}
\end{table}

\clearpage

\subsection{Ablation on the Planning Objective}
\label{appendix: planning_ablation}

As mentioned in the main paper, we can leverage the GHM in several different planning approaches. In this section, we compare two strategies. The first is the approach presented in the paper, in which we optimize both the first action and the policy sequence $(z_1, \ldots, z_n)$:
\begin{equation}\label{eq:plan_a_then_z}
\max_{a_1, z_1, \ldots, z_n} Q^{\pi_{z_1} \xrightarrow{\alpha_1} \pi_{z_2} \ldots \xrightarrow{\alpha_{n-1}} \pi_{z_n}}_\gamma(s, a_1)
\end{equation}
while the second does not involve action optimization
\begin{equation}\label{eq:plan_z}
\max_{z_1, \ldots, z_n} V^{\pi_{z_1} \xrightarrow{\alpha_1} \pi_{z_2} \ldots \xrightarrow{\alpha_{n-1}} \pi_{z_n}}_\gamma(s)
\end{equation}
The difference is that, when using~\eqref{eq:plan_a_then_z}, the first action is deterministic, whereas in~\eqref{eq:plan_z} it is sampled according to $\pi_{z_1}$.
Figure~\ref{fig:action_opt_ablation} shows that maximizing over the first action yields more consistent performance overall, with an average improvement of about 70\%.\footnote{This is mostly due to unsuccessful outcomes (i.e., near-zero performance) on certain tasks when \autoref{eq:plan_z}.} In contrast, planning over the $z$ sequence can fail when the base policy is diffuse (i.e., highly stochastic). This occurs, for example, for CRL policies in the Cube domains and for GC-BC policies in AntMaze. Finally, Table~\ref{tab:action_opt_ablation} provides a tabular summary of these results.
\begin{figure}[h!]
    \centering
    \includegraphics[width=0.9\textwidth]
    {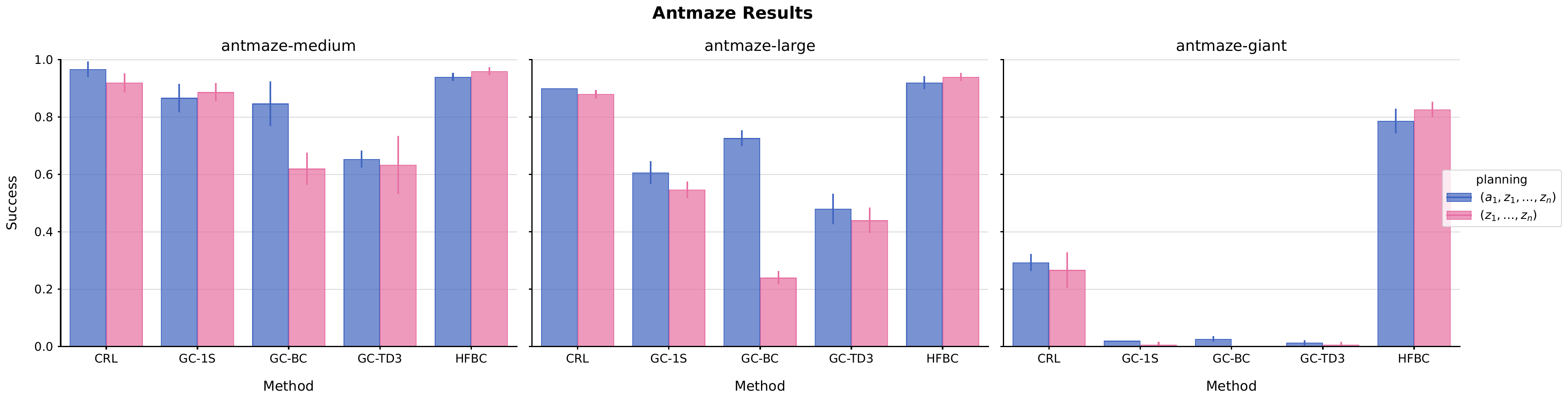}
    \includegraphics[width=0.9\textwidth]
    {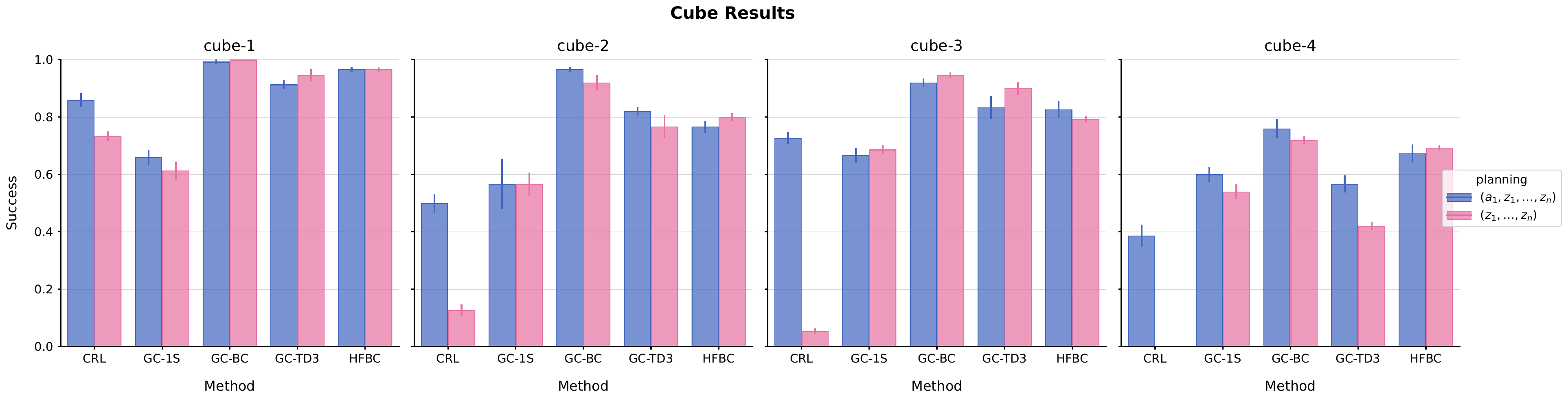}
    \caption{Success rate ($\uparrow$) of \textsc{CompPlan} with different base policies when maximizing over $(a_1, z_1, \ldots z_n)$ (Eq.~\ref{eq:plan_a_then_z}) or $(z_1, \ldots, z_n)$ (Eq.~\ref{eq:plan_z}). We report mean and standard deviation 3 seeds for GHM training.}
    \label{fig:action_opt_ablation}
\end{figure}
\begin{table}[h!]
\setlength{\aboverulesep}{0pt}
\setlength{\belowrulesep}{0pt}
\centering
\caption{Success rate ($\uparrow$) of \textsc{CompPlan} with (Eq.~\ref{eq:plan_a_then_z}) and without (Eq.~\ref{eq:plan_z}) action maximization. We report mean and standard deviation over 3 seeds. Best method highlighted in \mhl{blue}; \mdec{gray} indicates no significant difference.}
\label{tab:action_opt_ablation}
\renewcommand{\arraystretch}{1.3}
\begin{adjustbox}{width=\textwidth}
\begin{tabular}{l|ll|ll|ll|ll|ll}
\toprule
\multirow[c]{2}{*}{\textbf{Domain}} & \multicolumn{2}{c}{\textbf{CRL}} & \multicolumn{2}{c}{\textbf{GC-1S}}& \multicolumn{2}{c}{\textbf{GC-BC}}& \multicolumn{2}{c}{\textbf{GC-TD3}}& \multicolumn{2}{c}{\textbf{HFBC}} \\ & \multicolumn{1}{c}{$\bm{\max Q}$} & \multicolumn{1}{c|}{$\bm{\max V}$} & \multicolumn{1}{c}{$\bm{\max Q}$} & \multicolumn{1}{c|}{$\bm{\max V}$} & \multicolumn{1}{c}{$\bm{\max Q}$} & \multicolumn{1}{c|}{$\bm{\max V}$} & \multicolumn{1}{c}{$\bm{\max Q}$} & \multicolumn{1}{c|}{$\bm{\max V}$} & \multicolumn{1}{c}{$\bm{\max Q}$} & \multicolumn{1}{c}{$\bm{\max V}$}\\
\midrule
\textsc{antmaze-medium} & \cellcolor{hl}{0.97 {\footnotesize (0.02)}} & 0.92 {\footnotesize (0.03)} & 0.87 {\footnotesize (0.05)} & \cellcolor{hl}{0.89 {\footnotesize (0.03)}} & \cellcolor{hl}{0.85 {\footnotesize (0.08)}} & 0.62 {\footnotesize (0.05)} & \cellcolor{hl}{0.65 {\footnotesize (0.03)}} & 0.63 {\footnotesize (0.10)} & 0.94 {\footnotesize (0.01)} & \cellcolor{hl}{0.96 {\footnotesize (0.01)}}\\
\hhline{~|----------}
\textsc{antmaze-large} & \cellcolor{hl}{0.90 {\footnotesize (0.00)}} & 0.88 {\footnotesize (0.01)} & \cellcolor{hl}{0.61 {\footnotesize (0.04)}} & 0.55 {\footnotesize (0.03)} & \cellcolor{hl}{0.73 {\footnotesize (0.02)}} & 0.24 {\footnotesize (0.02)} & \cellcolor{hl}{0.48 {\footnotesize (0.05)}} & 0.44 {\footnotesize (0.04)} & 0.92 {\footnotesize (0.02)} & \cellcolor{hl}{0.94 {\footnotesize (0.01)}}\\
\hhline{~|----------}
\textsc{antmaze-giant} & \cellcolor{hl}{0.29 {\footnotesize (0.03)}} & 0.27 {\footnotesize (0.06)} & \cellcolor{hl}{0.02 {\footnotesize (0.00)}} & 0.01 {\footnotesize (0.01)} & \cellcolor{hl}{0.03 {\footnotesize (0.01)}} & 0.00 {\footnotesize (0.00)} & \cellcolor{gray!25}{0.01 {\footnotesize (0.01)}} & \cellcolor{gray!25}{0.01 {\footnotesize (0.01)}} & 0.79 {\footnotesize (0.04)} & \cellcolor{hl}{0.83 {\footnotesize (0.02)}}\\
\hhline{~|----------}
\textsc{cube-1} & \cellcolor{hl}{0.86 {\footnotesize (0.02)}} & 0.73 {\footnotesize (0.01)} & \cellcolor{hl}{0.66 {\footnotesize (0.02)}} & 0.61 {\footnotesize (0.03)} & 0.99 {\footnotesize (0.01)} & \cellcolor{hl}{1.00 {\footnotesize (0.00)}} & 0.91 {\footnotesize (0.01)} & \cellcolor{hl}{0.95 {\footnotesize (0.02)}} & \cellcolor{gray!25}{0.97 {\footnotesize (0.01)}} & \cellcolor{gray!25}{0.97 {\footnotesize (0.01)}}\\
\hhline{~|----------}
\textsc{cube-2} & \cellcolor{hl}{0.50 {\footnotesize (0.03)}} & 0.13 {\footnotesize (0.02)} & \cellcolor{gray!25}{0.57 {\footnotesize (0.04)}} & \cellcolor{gray!25}{0.57 {\footnotesize (0.04)}} & \cellcolor{hl}{0.97 {\footnotesize (0.01)}} & 0.92 {\footnotesize (0.02)} & \cellcolor{hl}{0.82 {\footnotesize (0.01)}} & 0.77 {\footnotesize (0.04)} & 0.77 {\footnotesize (0.02)} & \cellcolor{hl}{0.80 {\footnotesize (0.01)}}\\
\hhline{~|----------}
\textsc{cube-3} & \cellcolor{hl}{0.73 {\footnotesize (0.02)}} & 0.05 {\footnotesize (0.01)} & 0.67 {\footnotesize (0.02)} & \cellcolor{hl}{0.69 {\footnotesize (0.01)}} & 0.92 {\footnotesize (0.01)} & \cellcolor{hl}{0.95 {\footnotesize (0.01)}} & 0.83 {\footnotesize (0.04)} & \cellcolor{hl}{0.90 {\footnotesize (0.02)}} & \cellcolor{hl}{0.83 {\footnotesize (0.03)}} & 0.79 {\footnotesize (0.01)}\\
\hhline{~|----------}
\textsc{cube-4} & \cellcolor{hl}{0.39 {\footnotesize (0.04)}} & 0.00 {\footnotesize (0.00)} & \cellcolor{hl}{0.60 {\footnotesize (0.02)}} & 0.54 {\footnotesize (0.02)} & \cellcolor{hl}{0.76 {\footnotesize (0.03)}} & 0.72 {\footnotesize (0.01)} & \cellcolor{hl}{0.57 {\footnotesize (0.03)}} & 0.42 {\footnotesize (0.01)} & 0.67 {\footnotesize (0.03)} & \cellcolor{hl}{0.69 {\footnotesize (0.01)}}\\
\bottomrule
\end{tabular}
\end{adjustbox}
\end{table}

\clearpage
\subsection{Ablation on the Proposal Distribution}
\label{appendix: proposal_ablation}
In this section, we study the effect of the sampling distribution on the planning procedure. As mentioned in the main paper, GHMs are trained either with the policy condition $z$ or with a learnable token $\varnothing$. In the latter case, the resulting GHM corresponds to the behavioral policy. We compare two planning strategies: (1) sampling from the GHM conditioned on the policy associated with the goal (\emph{conditional proposal}); and (2) sampling from the GHM of the behavioral policy (\emph{unconditional proposal}). Figure~\ref{fig:proposal_ablation} shows that, in AntMaze, sampling from the unconditional distribution performs only marginally worse than the conditional proposal, highlighting the robustness of our planning procedure. Table~\ref{tab:conditional_ablation} provides a tabular summary of these results.

\begin{figure}[ht]
    \centering
    \includegraphics[width=\textwidth]
    {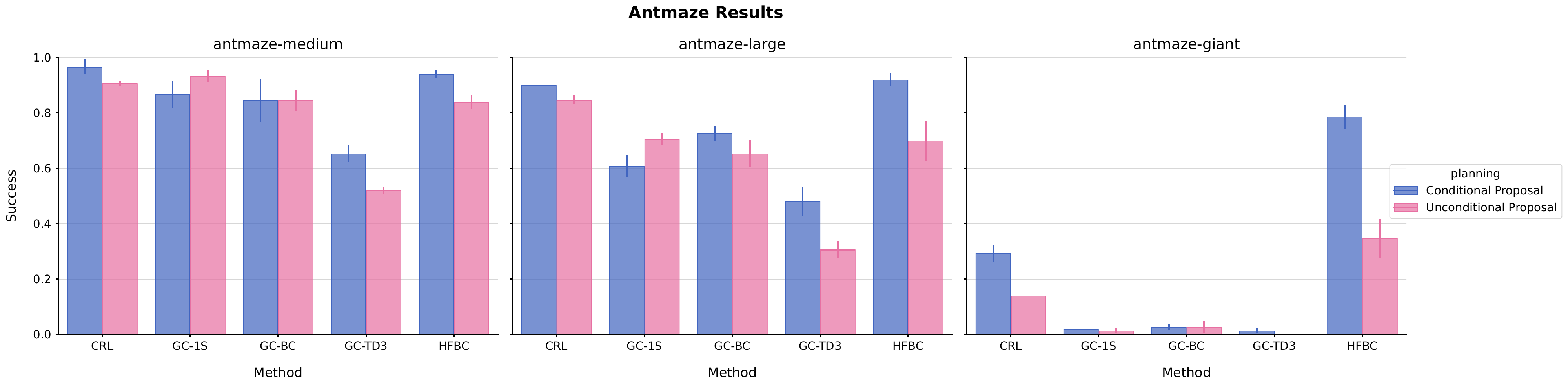}
    \includegraphics[width=\textwidth]{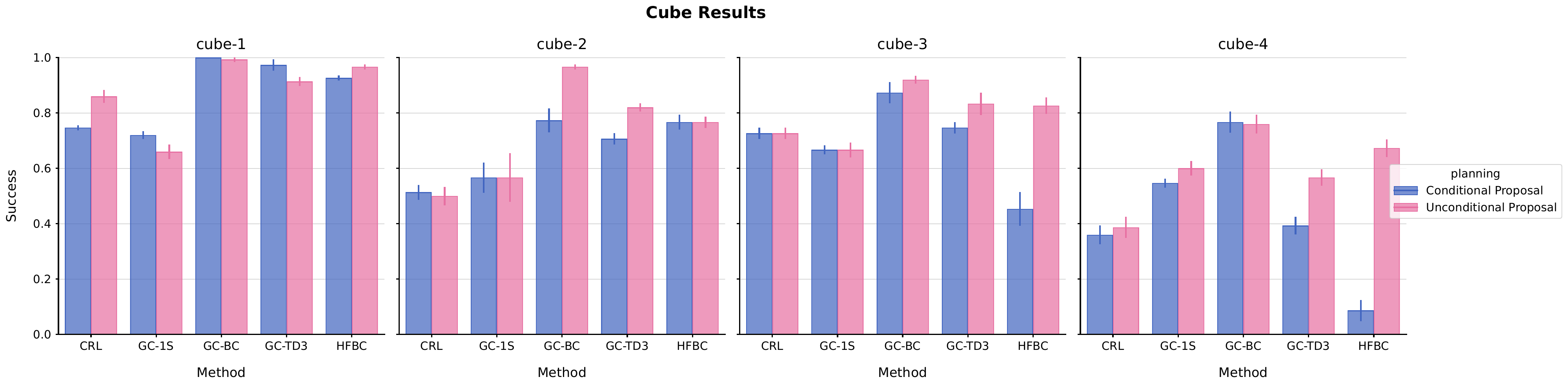}
    \caption{Success rate ($\uparrow$) of \textsc{CompPlan} with different base policies when using the conditional or unconditional GHMs to propose subgoals. We report mean and standard deviation 3 seeds for GHM training.}
    \label{fig:proposal_ablation}
\end{figure}

\begin{table}[ht]
\setlength{\aboverulesep}{0pt}
\setlength{\belowrulesep}{0pt}
\centering
\caption{Success rate ($\uparrow$) of \textsc{CompPlan} with conditional and unconditional proposal. We report mean and standard deviation 3 seeds. Best method highlighted in \mhl{blue}; \mdec{gray} indicates no significant difference.}
\label{tab:conditional_ablation}
\renewcommand{\arraystretch}{1.3}
\begin{adjustbox}{width=\textwidth}
\begin{tabular}{l|ll|ll|ll|ll|ll|}
\toprule
\multirow[c]{2}{*}{\textbf{Domain}} & \multicolumn{2}{c}{\textbf{CRL}}& \multicolumn{2}{c}{\textbf{GC-1S}}& \multicolumn{2}{c}{\textbf{GC-BC}}& \multicolumn{2}{c}{\textbf{GC-TD3}}& \multicolumn{2}{c}{\textbf{HFBC}}\\
  & \multicolumn{1}{c}{\textbf{Cond}} & \multicolumn{1}{c|}{\textbf{Uncond}} & \multicolumn{1}{c}{\textbf{Cond}} & \multicolumn{1}{c|}{\textbf{Uncond}} & \multicolumn{1}{c}{\textbf{Cond}} & \multicolumn{1}{c|}{\textbf{Uncond}} & \multicolumn{1}{c}{\textbf{Cond}} & \multicolumn{1}{c|}{\textbf{Uncond}} & \multicolumn{1}{c}{\textbf{Cond}} & \multicolumn{1}{c}{\textbf{Uncond}}\\
\midrule
\textsc{antmaze-medium} & \cellcolor{hl}{0.97 {\footnotesize (0.02)}} & 0.91 {\footnotesize (0.01)} & 0.87 {\footnotesize (0.05)} & \cellcolor{hl}{0.93 {\footnotesize (0.02)}} & \cellcolor{gray!25}{0.85 {\footnotesize (0.04)}} & \cellcolor{gray!25}{0.85 {\footnotesize (0.04)}} & \cellcolor{hl}{0.65 {\footnotesize (0.03)}} & 0.52 {\footnotesize (0.01)} & \cellcolor{hl}{0.94 {\footnotesize (0.01)}} & 0.84 {\footnotesize (0.02)}\\
\hhline{~|----------}
\textsc{antmaze-large} & \cellcolor{hl}{0.90 {\footnotesize (0.00)}} & 0.85 {\footnotesize (0.01)} & 0.61 {\footnotesize (0.04)} & \cellcolor{hl}{0.71 {\footnotesize (0.02)}} & \cellcolor{hl}{0.73 {\footnotesize (0.02)}} & 0.65 {\footnotesize (0.05)} & \cellcolor{hl}{0.48 {\footnotesize (0.05)}} & 0.31 {\footnotesize (0.03)} & \cellcolor{hl}{0.92 {\footnotesize (0.02)}} & 0.70 {\footnotesize (0.07)}\\
\hhline{~|----------}
\textsc{antmaze-giant} & \cellcolor{hl}{0.29 {\footnotesize (0.03)}} & 0.14 {\footnotesize (0.00)} & \cellcolor{hl}{0.02 {\footnotesize (0.00)}} & 0.01 {\footnotesize (0.01)} & \cellcolor{gray!25}{0.03 {\footnotesize (0.02)}} & \cellcolor{gray!25}{0.03 {\footnotesize (0.02)}} & \cellcolor{hl}{0.01 {\footnotesize (0.01)}} & 0.00 {\footnotesize (0.00)} & \cellcolor{hl}{0.79 {\footnotesize (0.04)}} & 0.35 {\footnotesize (0.07)}\\
\hhline{~|----------}
\textsc{cube-1} & 0.75 {\footnotesize (0.01)} & \cellcolor{hl}{0.86 {\footnotesize (0.02)}} & \cellcolor{hl}{0.72 {\footnotesize (0.01)}} & 0.66 {\footnotesize (0.02)} & \cellcolor{hl}{1.00 {\footnotesize (0.00)}} & 0.99 {\footnotesize (0.01)} & \cellcolor{hl}{0.97 {\footnotesize (0.02)}} & 0.91 {\footnotesize (0.01)} & 0.93 {\footnotesize (0.01)} & \cellcolor{hl}{0.97 {\footnotesize (0.01)}}\\
\hhline{~|----------}
\textsc{cube-2} & \cellcolor{hl}{0.51 {\footnotesize (0.02)}} & 0.50 {\footnotesize (0.03)} & \cellcolor{gray!25}{0.57 {\footnotesize (0.09)}} & \cellcolor{gray!25}{0.57 {\footnotesize (0.09)}} & 0.77 {\footnotesize (0.04)} & \cellcolor{hl}{0.97 {\footnotesize (0.01)}} & 0.71 {\footnotesize (0.02)} & \cellcolor{hl}{0.82 {\footnotesize (0.01)}} & \cellcolor{gray!25}{0.77 {\footnotesize (0.02)}} & \cellcolor{gray!25}{0.77 {\footnotesize (0.02)}}\\
\hhline{~|----------}
\textsc{cube-3} & \cellcolor{gray!25}{0.73 {\footnotesize (0.02)}} & \cellcolor{gray!25}{0.73 {\footnotesize (0.02)}} & \cellcolor{gray!25}{0.67 {\footnotesize (0.02)}} & \cellcolor{gray!25}{0.67 {\footnotesize (0.02)}} & 0.87 {\footnotesize (0.04)} & \cellcolor{hl}{0.92 {\footnotesize (0.01)}} & 0.75 {\footnotesize (0.02)} & \cellcolor{hl}{0.83 {\footnotesize (0.04)}} & 0.45 {\footnotesize (0.06)} & \cellcolor{hl}{0.83 {\footnotesize (0.03)}}\\
\hhline{~|----------}
\textsc{cube-4} & 0.36 {\footnotesize (0.03)} & \cellcolor{hl}{0.39 {\footnotesize (0.04)}} & 0.55 {\footnotesize (0.01)} & \cellcolor{hl}{0.60 {\footnotesize (0.02)}} & \cellcolor{hl}{0.77 {\footnotesize (0.04)}} & 0.76 {\footnotesize (0.03)} & 0.39 {\footnotesize (0.03)} & \cellcolor{hl}{0.57 {\footnotesize (0.03)}} & 0.09 {\footnotesize (0.04)} & \cellcolor{hl}{0.67 {\footnotesize (0.03)}}\\
\bottomrule
\end{tabular}
\end{adjustbox}
\end{table}

\clearpage
\subsection{Ablation on the Consistency Objective}
\label{appendix: consistency_ablation}
This section expands upon \autoref{ssec:td-hc-exp}, providing the complete empirical results for the Temporal Difference Horizon Consistency (\tdhc) objective across generative fidelity (\autoref{tab:consistency_ablation_emd}), qualitative predictions (\autoref{fig:qual-ghm}), and downstream planning performance (\autoref{tab:success_adjustbox2}).
\paragraph{\textbf{Generative Fidelity}} 
\autoref{tab:consistency_ablation_emd} shows that \tdhc systematically improves the generative accuracy of the \tdfl baseline as measured by the Earth Mover's Distance \citep[EMD;][]{RubnerTG00}.
These gains are especially pronounced in complex domains like \textsc{antmaze-giant} where bootstrapping errors easily compound.
\autoref{fig:qual-ghm} visually confirms this effect: while both models perform comparably at shorter horizons ($\gamma=0.99$), \tdfl suffers from severe compounding errors at longer horizons ($\gamma=0.998$), and fails to capture the tail of the distribution.
By anchoring its predictions with shorter-horizons, \tdhc successfully respects the topological constraints and properly captures the tail of the true successor state distribution.
\paragraph{\textbf{Planning Performance}} Despite the generative advantages of \tdhc at extreme horizons, \autoref{tab:success_adjustbox2} reveals that downstream planning success rates remain broadly similar between the two methods. Because our planning procedure evaluates candidate sequences using moderate effective horizons ($\beta_i \in [0.98, 0.99]$, or roughly $50-100$ steps), it does not query the extreme timescales where \tdfl breaks down. We expect \tdhc to enable planning over more extreme horizons in the future as benchmarks evolve towards more complex tasks.

\begin{table*}[ht]
\setlength{\aboverulesep}{0pt}
\setlength{\belowrulesep}{0pt}
\centering
\caption{Accuracy (EMD, $\downarrow$) of GHMs trained with (\tdhc) and without (\tdfl) our horizon consistency loss (\autoref{ssec:td-hc}). Best method is highlighted in \mhl{blue}.}
\label{tab:consistency_ablation_emd}
\renewcommand{\arraystretch}{1.3}
\begin{adjustbox}{width=\textwidth}
\begin{tabular}{l|ll|ll|ll|ll|ll|}
\toprule
\multirow[c]{2}{*}{\textbf{Domain}} & \multicolumn{2}{c}{\textbf{CRL}}& \multicolumn{2}{c}{\textbf{GC-1S}}& \multicolumn{2}{c}{\textbf{GC-BC}}& \multicolumn{2}{c}{\textbf{GC-TD3}}& \multicolumn{2}{c}{\textbf{HFBC}}\\
 & \multicolumn{1}{c}{\textbf{\tdfl (\xmark)}} & \multicolumn{1}{c|}{\textbf{\tdhc (\cmark)}} & \multicolumn{1}{c}{\textbf{\tdfl (\xmark)}} & \multicolumn{1}{c|}{\textbf{\tdhc (\cmark)}} & \multicolumn{1}{c}{\textbf{\tdfl (\xmark)}} & \multicolumn{1}{c|}{\textbf{\tdhc (\cmark)}} & \multicolumn{1}{c}{\textbf{\tdfl (\xmark)}} & \multicolumn{1}{c|}{\textbf{\tdhc (\cmark)}} & \multicolumn{1}{c}{\textbf{\tdfl (\xmark)}} & \multicolumn{1}{c}{\textbf{\tdhc (\cmark)}}\\
\midrule
\textsc{antmaze-medium} & 4.41 {\footnotesize (0.05)} & \cellcolor{hl}{4.22 {\footnotesize (0.06)}} & 4.40 {\footnotesize (0.02)} & \cellcolor{hl}{4.22 {\footnotesize (0.03)}} & 4.56 {\footnotesize (0.05)} & \cellcolor{hl}{4.36 {\footnotesize (0.03)}} & 4.68 {\footnotesize (0.05)} & \cellcolor{hl}{4.44 {\footnotesize (0.05)}} & 3.38 {\footnotesize (0.02)} & \cellcolor{hl}{3.22 {\footnotesize (0.02)}}\\
\hhline{~|----------}
\textsc{antmaze-large} & 5.24 {\footnotesize (0.07)} & \cellcolor{hl}{4.81 {\footnotesize (0.03)}} & 5.12 {\footnotesize (0.18)} & \cellcolor{hl}{4.67 {\footnotesize (0.04)}} & 5.32 {\footnotesize (0.03)} & \cellcolor{hl}{4.73 {\footnotesize (0.02)}} & 5.33 {\footnotesize (0.07)} & \cellcolor{hl}{4.83 {\footnotesize (0.04)}} & 3.50 {\footnotesize (0.05)} & \cellcolor{hl}{3.18 {\footnotesize (0.01)}}\\
\hhline{~|----------}
\textsc{antmaze-giant} & 6.77 {\footnotesize (0.49)} & \cellcolor{hl}{5.74 {\footnotesize (0.06)}} & 7.29 {\footnotesize (0.69)} & \cellcolor{hl}{5.25 {\footnotesize (0.08)}} & 6.46 {\footnotesize (0.10)} & \cellcolor{hl}{4.95 {\footnotesize (0.12)}} & 6.51 {\footnotesize (0.14)} & \cellcolor{hl}{5.24 {\footnotesize (0.11)}} & 3.78 {\footnotesize (0.03)} & \cellcolor{hl}{3.00 {\footnotesize (0.09)}}\\
\hhline{~|----------}
\textsc{cube-1} & 1.60 {\footnotesize (0.02)} & \cellcolor{hl}{1.57 {\footnotesize (0.03)}} & 1.43 {\footnotesize (0.00)} & \cellcolor{hl}{1.33 {\footnotesize (0.03)}} & 1.03 {\footnotesize (0.01)} & \cellcolor{hl}{0.98 {\footnotesize (0.01)}} & 1.12 {\footnotesize (0.01)} & \cellcolor{hl}{1.06 {\footnotesize (0.01)}} & 1.82 {\footnotesize (0.06)} & \cellcolor{hl}{1.27 {\footnotesize (0.01)}}\\
\hhline{~|----------}
\textsc{cube-2} & 2.36 {\footnotesize (0.03)} & \cellcolor{hl}{2.23 {\footnotesize (0.02)}} & 1.86 {\footnotesize (0.04)} & \cellcolor{hl}{1.71 {\footnotesize (0.01)}} & 1.47 {\footnotesize (0.00)} & \cellcolor{hl}{1.42 {\footnotesize (0.01)}} & 1.89 {\footnotesize (0.04)} & \cellcolor{hl}{1.80 {\footnotesize (0.03)}} & 2.29 {\footnotesize (0.09)} & \cellcolor{hl}{1.54 {\footnotesize (0.04)}}\\
\hhline{~|----------}
\textsc{cube-3} & 2.15 {\footnotesize (0.02)} & \cellcolor{hl}{2.10 {\footnotesize (0.02)}} & 1.80 {\footnotesize (0.04)} & \cellcolor{hl}{1.71 {\footnotesize (0.03)}} & 1.89 {\footnotesize (0.02)} & \cellcolor{hl}{1.85 {\footnotesize (0.02)}} & 1.91 {\footnotesize (0.06)} & \cellcolor{hl}{1.84 {\footnotesize (0.04)}} & 1.99 {\footnotesize (0.09)} & \cellcolor{hl}{1.55 {\footnotesize (0.02)}}\\
\hhline{~|----------}
\textsc{cube-4} & 2.41 {\footnotesize (0.03)} & \cellcolor{hl}{2.34 {\footnotesize (0.01)}} & 2.13 {\footnotesize (0.03)} & \cellcolor{hl}{2.05 {\footnotesize (0.03)}} & 2.33 {\footnotesize (0.03)} & \cellcolor{hl}{2.22 {\footnotesize (0.02)}} & 2.21 {\footnotesize (0.02)} & \cellcolor{hl}{2.15 {\footnotesize (0.02)}} & 2.05 {\footnotesize (0.05)} & \cellcolor{hl}{1.61 {\footnotesize (0.03)}}\\
\bottomrule
\end{tabular}
\end{adjustbox}
\end{table*}

\begin{table*}[ht]
\setlength{\aboverulesep}{0pt}
\setlength{\belowrulesep}{0pt}
\centering
\caption{Success rate ($\uparrow$) of \textsc{CompPlan} with consistency (\tdhc) and without consistency (\tdfl). Mean and standard deviation over 3 seeds. Best method highlighted in \mhl{blue}; \mdec{gray} indicates no significant difference.}
\label{tab:success_adjustbox2}
\renewcommand{\arraystretch}{1.3}
\begin{adjustbox}{width=\textwidth}
\begin{tabular}{l|ll|ll|ll|ll|ll|}
\toprule
\multirow[c]{2}{*}{\textbf{Domain}} & \multicolumn{2}{c}{\textbf{CRL}}& \multicolumn{2}{c}{\textbf{GC-1S}}& \multicolumn{2}{c}{\textbf{GC-BC}}& \multicolumn{2}{c}{\textbf{GC-TD3}}& \multicolumn{2}{c}{\textbf{HFBC}}\\
 &  \multicolumn{1}{c}{\textbf{\tdfl (\xmark)}}  &  \multicolumn{1}{c|}{\textbf{\tdhc (\cmark)}}  &  \multicolumn{1}{c}{\textbf{\tdfl (\xmark)}}  &  \multicolumn{1}{c|}{\textbf{\tdhc (\cmark)}}  &  \multicolumn{1}{c}{\textbf{\tdfl (\xmark)}}  &  \multicolumn{1}{c|}{\textbf{\tdhc (\cmark)}}  &  \multicolumn{1}{c}{\textbf{\tdfl (\xmark)}}  &  \multicolumn{1}{c|}{\textbf{\tdhc (\cmark)}}  &  \multicolumn{1}{c}{\textbf{\tdfl (\xmark)}}  &  \multicolumn{1}{c}{\textbf{\tdhc (\cmark)}} \\
\midrule
\textsc{antmaze-medium} & 0.95 {\footnotesize (0.01)} & \cellcolor{hl}{0.97 {\footnotesize (0.02)}} & 0.85 {\footnotesize (0.01)} & \cellcolor{hl}{0.87 {\footnotesize (0.05)}} & \cellcolor{hl}{0.91 {\footnotesize (0.01)}} & 0.85 {\footnotesize (0.08)} & \cellcolor{hl}{0.71 {\footnotesize (0.03)}} & 0.65 {\footnotesize (0.03)} & \cellcolor{hl}{0.95 {\footnotesize (0.02)}} & 0.94 {\footnotesize (0.01)}\\
\hhline{~|----------}
\textsc{antmaze-large} & \cellcolor{hl}{0.91 {\footnotesize (0.03)}} & 0.90 {\footnotesize (0.00)} & 0.53 {\footnotesize (0.09)} & \cellcolor{hl}{0.61 {\footnotesize (0.04)}} & \cellcolor{hl}{0.79 {\footnotesize (0.01)}} & 0.73 {\footnotesize (0.02)} & \cellcolor{hl}{0.51 {\footnotesize (0.03)}} & 0.48 {\footnotesize (0.05)} & 0.91 {\footnotesize (0.03)} & \cellcolor{hl}{0.92 {\footnotesize (0.02)}}\\
\hhline{~|----------}
\textsc{antmaze-giant} & \cellcolor{hl}{0.31 {\footnotesize (0.05)}} & 0.29 {\footnotesize (0.03)} & 0.01 {\footnotesize (0.01)} & \cellcolor{hl}{0.02 {\footnotesize (0.00)}} & 0.02 {\footnotesize (0.00)} & \cellcolor{hl}{0.03 {\footnotesize (0.01)}} & \cellcolor{gray!25}{0.01 {\footnotesize (0.01)}} & \cellcolor{gray!25}{0.01 {\footnotesize (0.01)}} & \cellcolor{hl}{0.81 {\footnotesize (0.02)}} & 0.79 {\footnotesize (0.04)}\\
\hhline{~|----------}
\textsc{cube-1} & \cellcolor{hl}{0.89 {\footnotesize (0.03)}} & 0.86 {\footnotesize (0.02)} & 0.55 {\footnotesize (0.02)} & \cellcolor{hl}{0.66 {\footnotesize (0.02)}} & \cellcolor{hl}{1.00 {\footnotesize (0.00)}} & 0.99 {\footnotesize (0.01)} & \cellcolor{hl}{0.97 {\footnotesize (0.01)}} & 0.91 {\footnotesize (0.01)} & \cellcolor{gray!25}{0.97 {\footnotesize (0.01)}} & \cellcolor{gray!25}{0.97 {\footnotesize (0.01)}}\\
\hhline{~|----------}
\textsc{cube-2} & 0.41 {\footnotesize (0.02)} & \cellcolor{hl}{0.50 {\footnotesize (0.03)}} & \cellcolor{gray!25}{0.57 {\footnotesize (0.09)}} & \cellcolor{gray!25}{0.57 {\footnotesize (0.09)}} & 0.95 {\footnotesize (0.02)} & \cellcolor{hl}{0.97 {\footnotesize (0.01)}} & \cellcolor{hl}{0.85 {\footnotesize (0.02)}} & 0.82 {\footnotesize (0.01)} & \cellcolor{hl}{0.84 {\footnotesize (0.03)}} & 0.77 {\footnotesize (0.02)}\\
\hhline{~|----------}
\textsc{cube-3} & 0.72 {\footnotesize (0.02)} & \cellcolor{hl}{0.73 {\footnotesize (0.02)}} & \cellcolor{hl}{0.72 {\footnotesize (0.01)}} & 0.67 {\footnotesize (0.02)} & 0.91 {\footnotesize (0.02)} & \cellcolor{hl}{0.92 {\footnotesize (0.01)}} & \cellcolor{gray!25}{0.83 {\footnotesize (0.04)}} & \cellcolor{gray!25}{0.83 {\footnotesize (0.04)}} & \cellcolor{gray!25}{0.83 {\footnotesize (0.03)}} & \cellcolor{gray!25}{0.83 {\footnotesize (0.03)}}\\
\hhline{~|----------}
\textsc{cube-4} & 0.36 {\footnotesize (0.00)} & \cellcolor{hl}{0.39 {\footnotesize (0.04)}} & 0.56 {\footnotesize (0.02)} & \cellcolor{hl}{0.60 {\footnotesize (0.02)}} & 0.75 {\footnotesize (0.02)} & \cellcolor{hl}{0.76 {\footnotesize (0.03)}} & 0.56 {\footnotesize (0.03)} & \cellcolor{hl}{0.57 {\footnotesize (0.03)}} & 0.64 {\footnotesize (0.04)} & \cellcolor{hl}{0.67 {\footnotesize (0.03)}}\\
\bottomrule
\end{tabular}
\end{adjustbox}
\end{table*}

\clearpage

\begin{figure}[H]
\vspace{1em}
\centering
\begin{tabular}{c c c c}
 & \shortstack{{\large \boldmath$\gamma=0.99$}\\$(\approx 100 \text{ steps})$}
 & \shortstack{{\large \boldmath$\gamma=0.995$}\\$(\approx 200 \text{ steps})$} 
 & \shortstack{{\large \boldmath$\gamma=0.998$}\\$(\approx 500 \text{ steps})$} \\[0.25cm] \cline{1-4} \\
\raisebox{1.75cm}{\rotatebox[origin=c]{90}{\tdfl}}
  & \includegraphics[height=3.5cm]{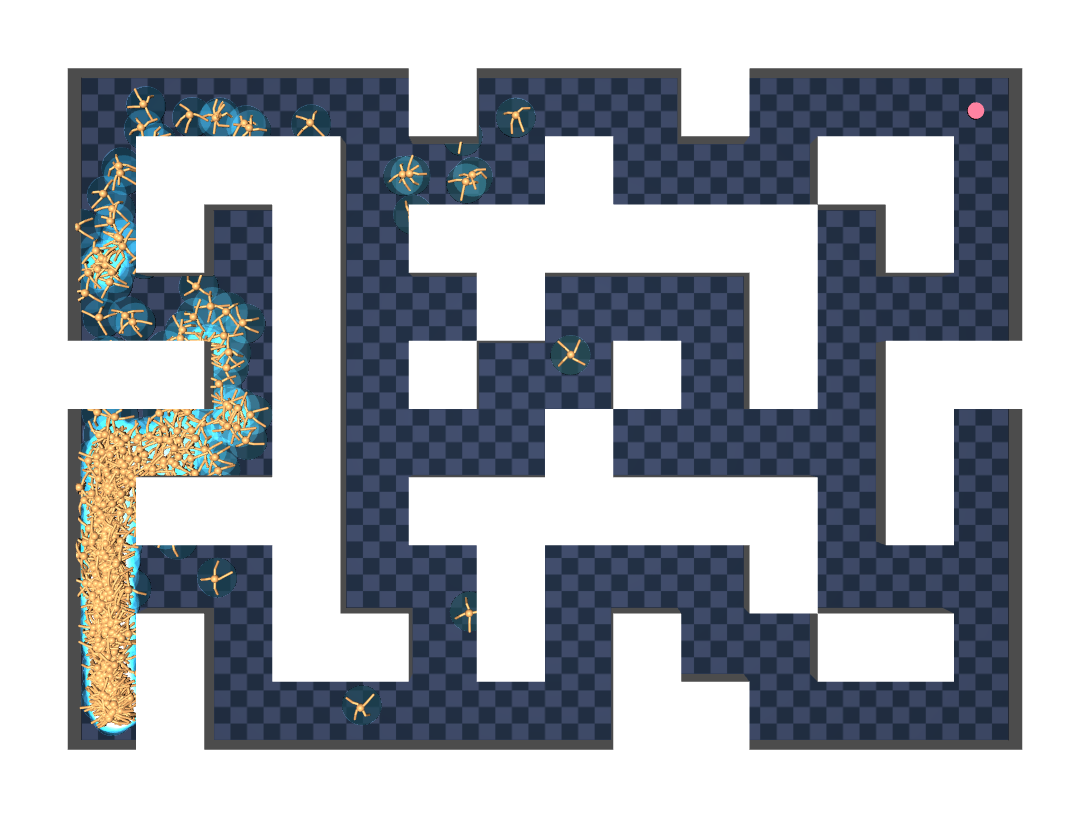}
  & \includegraphics[height=3.5cm]{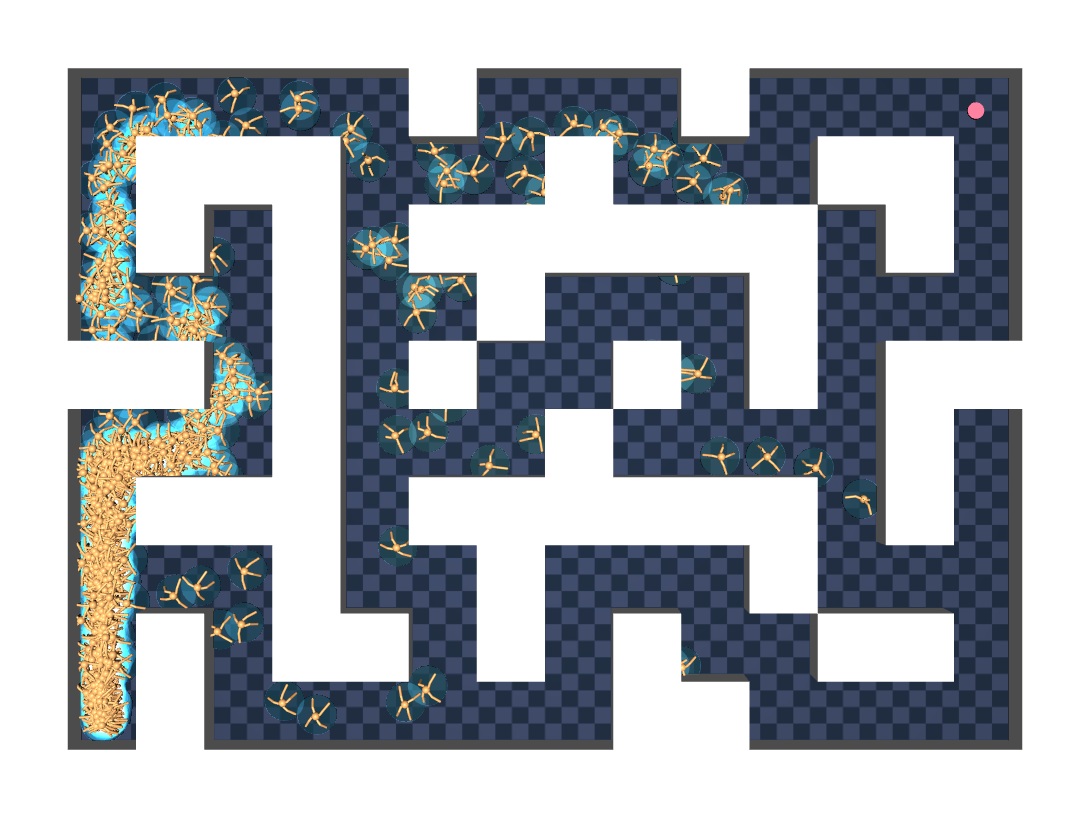}
  & \includegraphics[height=3.5cm]{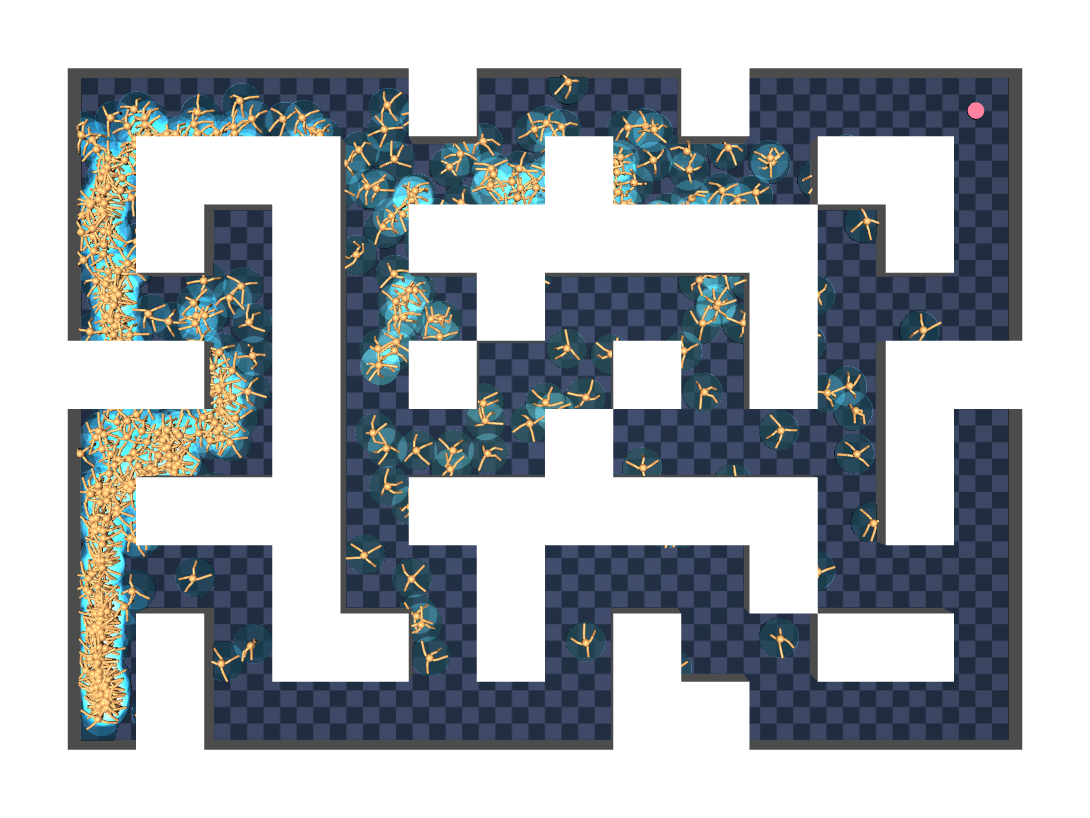} \\
\raisebox{1.75cm}{\rotatebox[origin=c]{90}{\tdhc}}
  & \includegraphics[height=3.5cm]{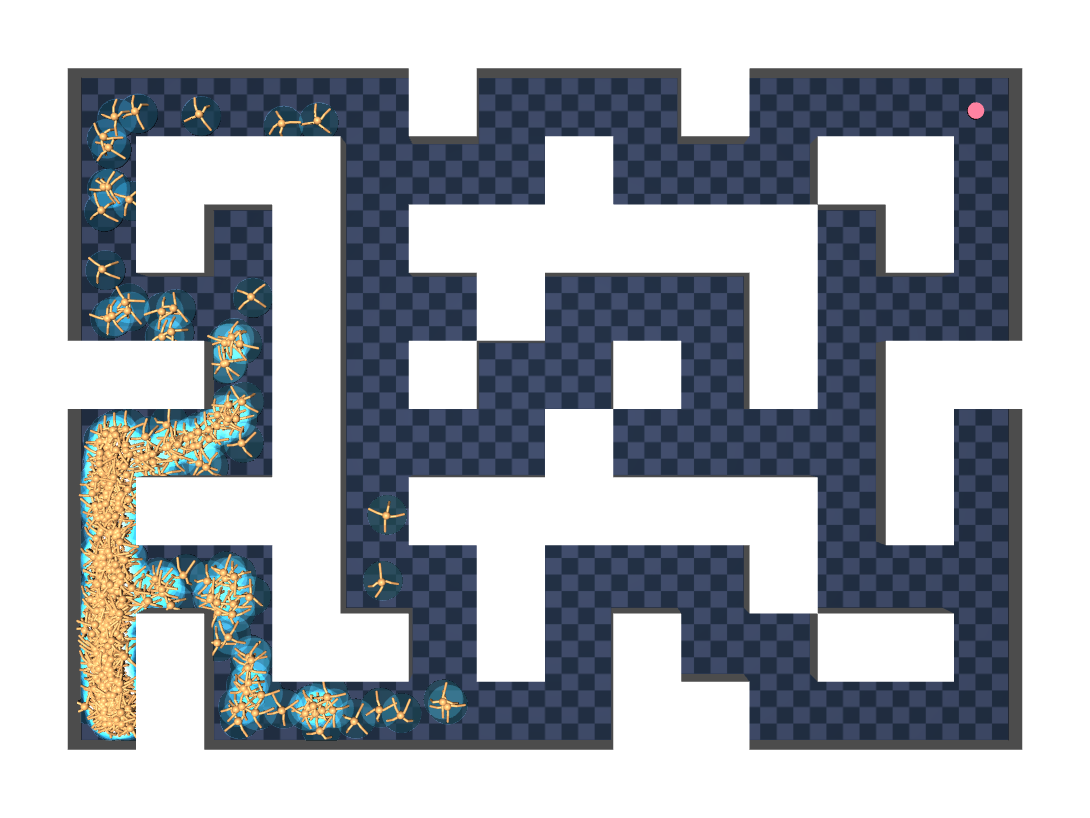}
  & \includegraphics[height=3.5cm]{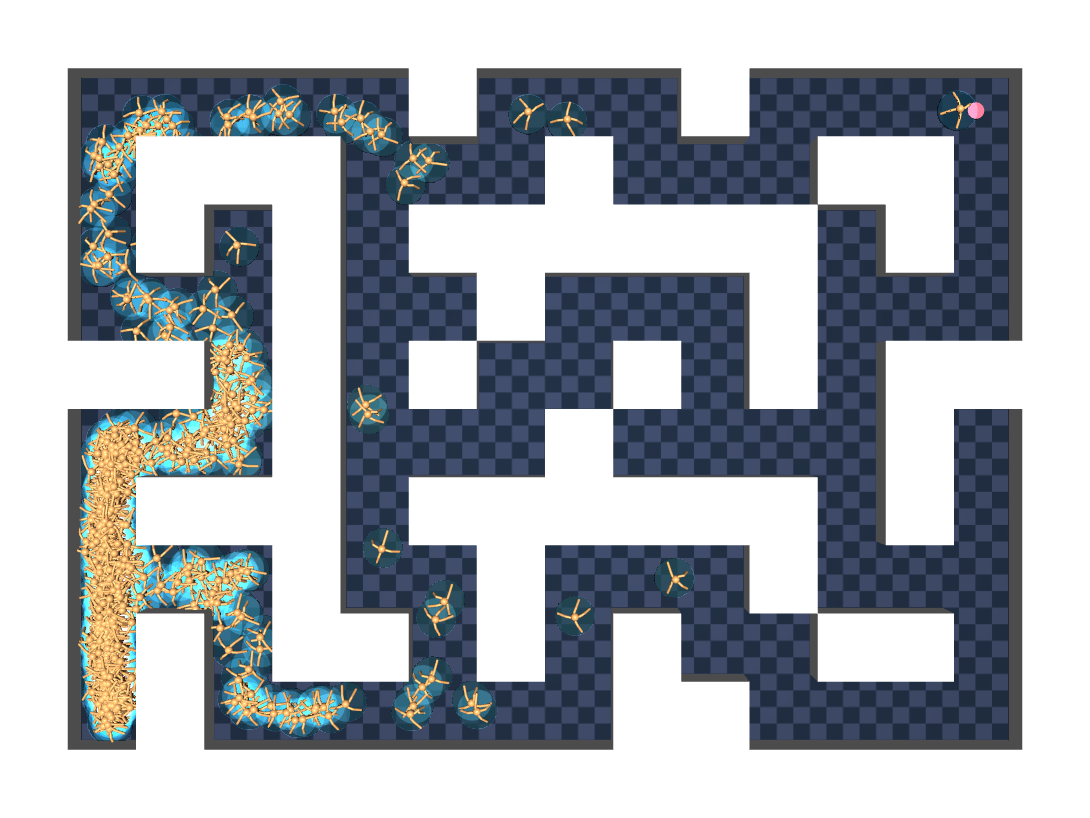}
  & \includegraphics[height=3.5cm]{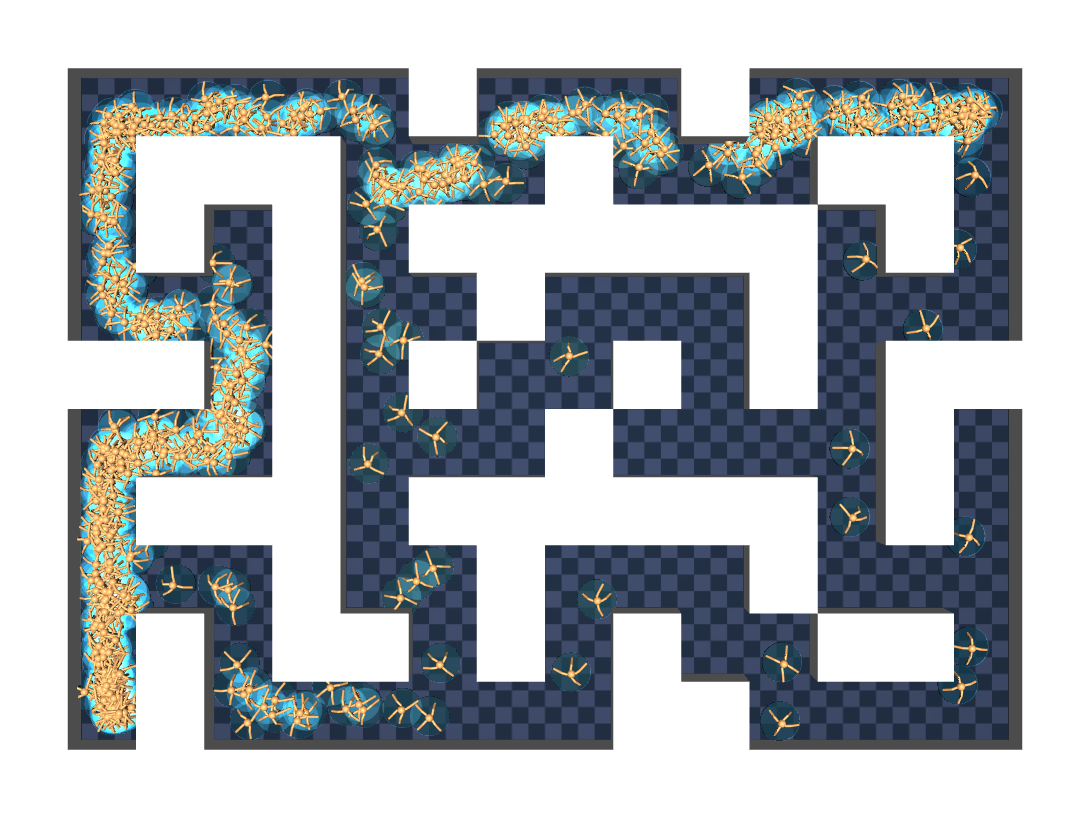} \\
  \cline{2-4} \\
\raisebox{1.2cm}{\rotatebox[origin=c]{90}{\shortstack{\textsc{ground}\\ \textsc{truth}}}}
  & \includegraphics[height=3.5cm]{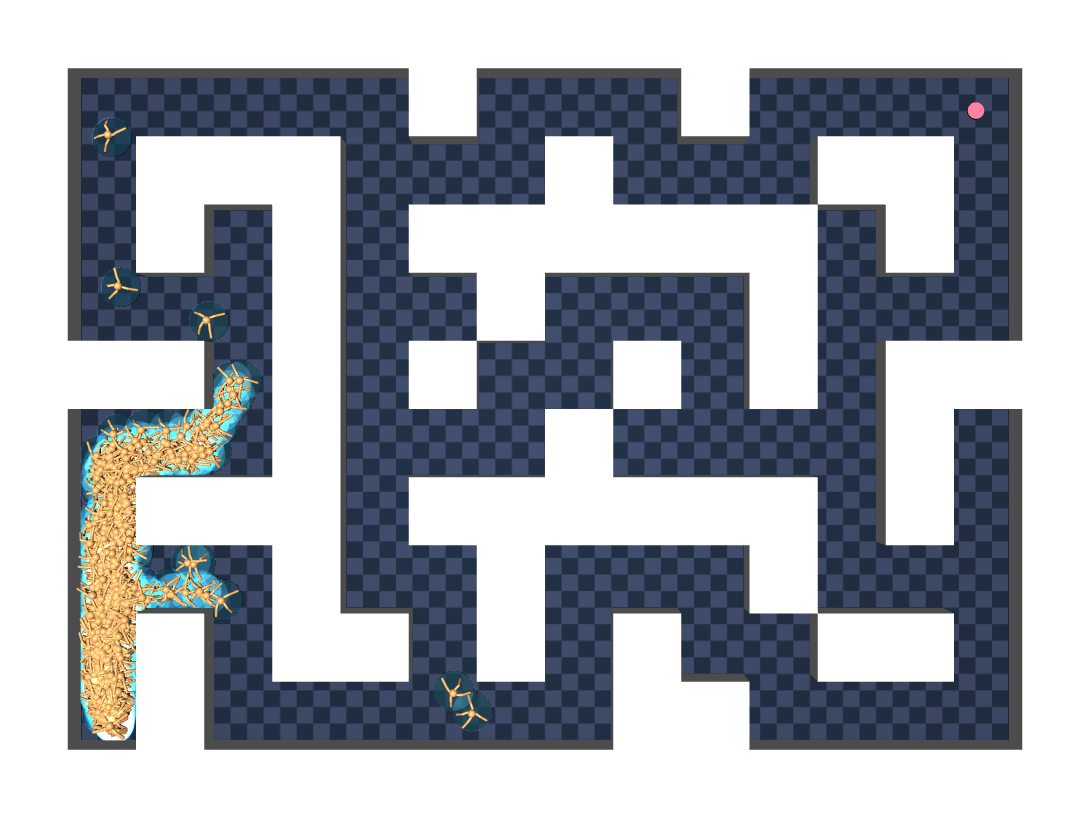}
  & \includegraphics[height=3.5cm]{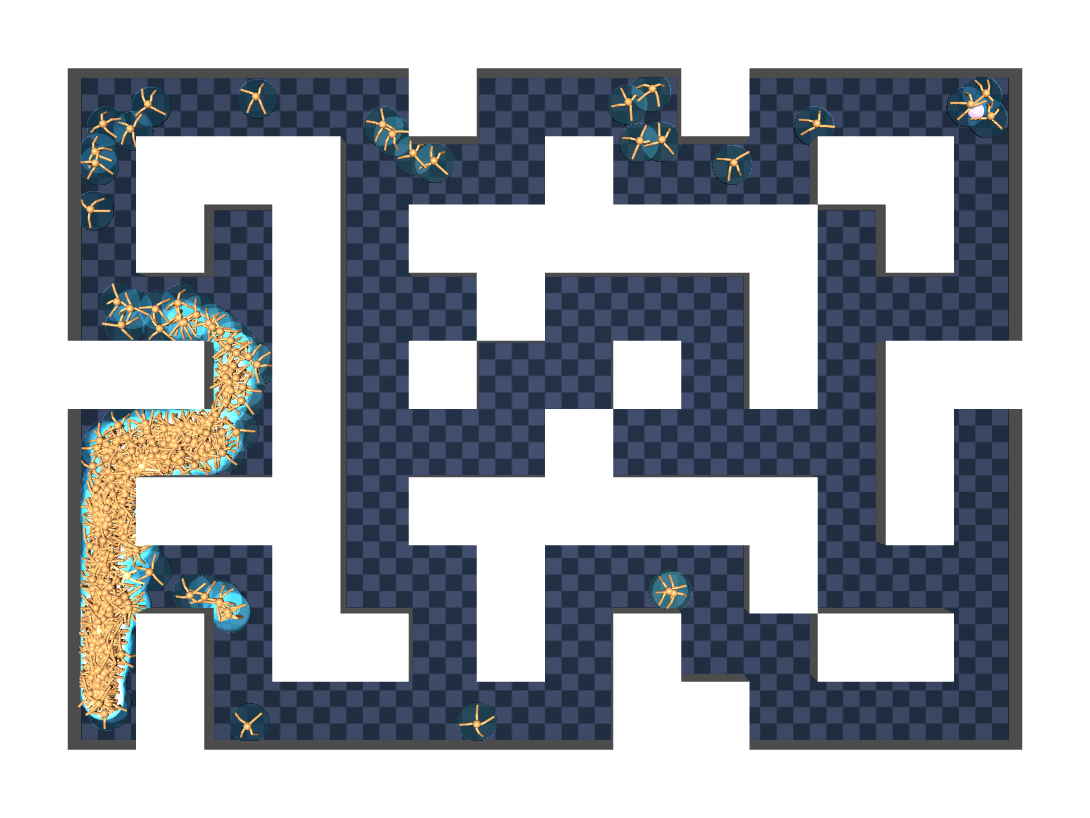}
  & \includegraphics[height=3.5cm]{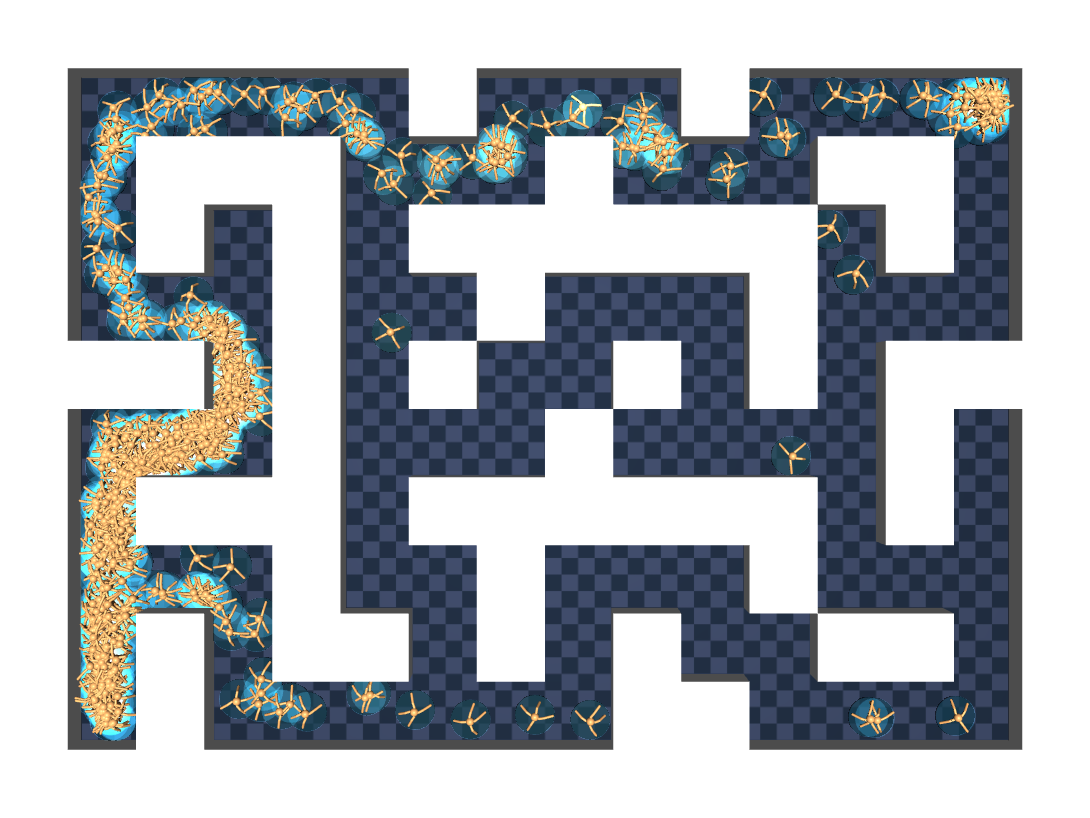} \\
\end{tabular}
\caption{Qualitative plots of GHM samples on \textsc{antmaze-giant} task 1 at different horizons $(0.99, 0.995, 0.998)$ from \tdfl and \tdhc with the last row depicting the ground truth discounted occupancy. As can be seen in the figure, performance is comparable at smaller horizons with \tdhc doing a much better job at capturing the true distribution as the horizon increases.}\label{fig:qual-ghm}
\end{figure}

\section{Qualitative Geometric Horizon Model Samples}

\begin{figure}[H]
    \centering
    \setlength{\tabcolsep}{2pt}
    \renewcommand{\arraystretch}{0.8}
    \begin{tabular}{cccccc}
        \includegraphics[width=0.15\textwidth]{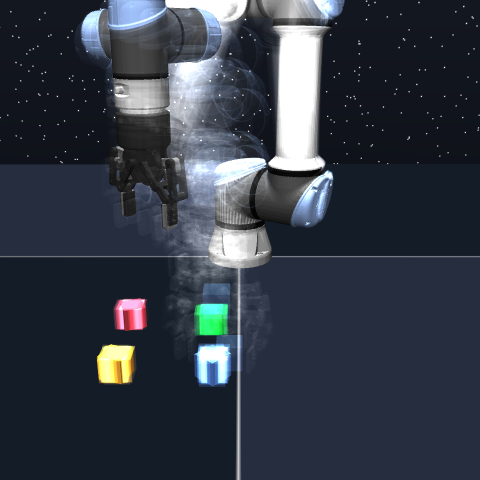} &
        \includegraphics[width=0.15\textwidth]{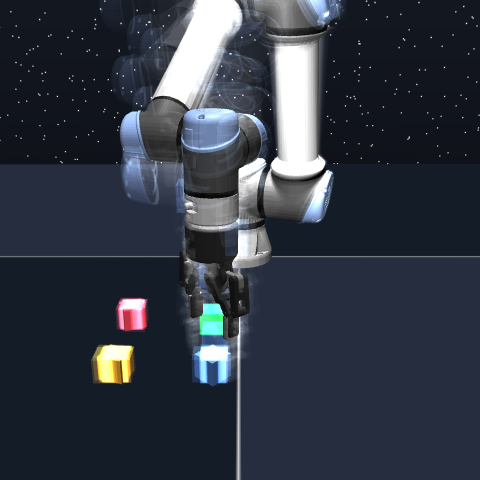} &
        \includegraphics[width=0.15\textwidth]{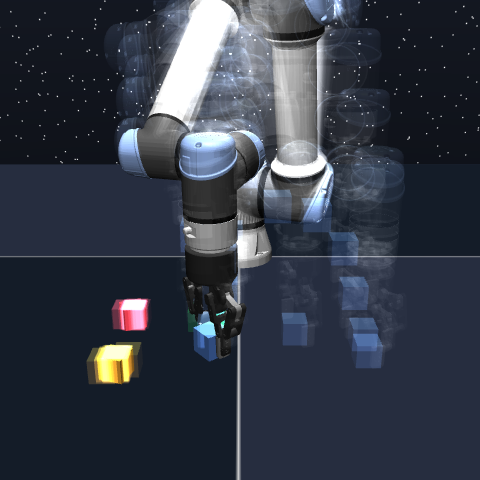} &
        \includegraphics[width=0.15\textwidth]{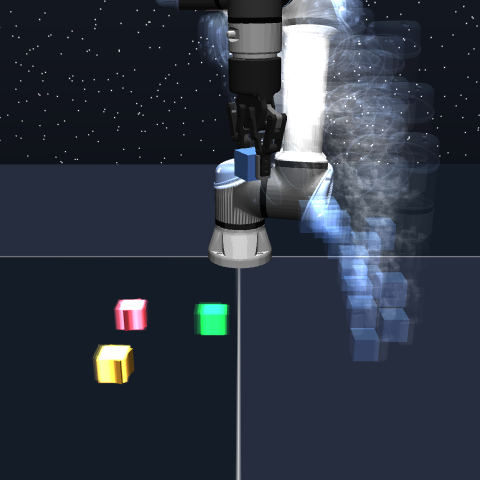} &
        \includegraphics[width=0.15\textwidth]{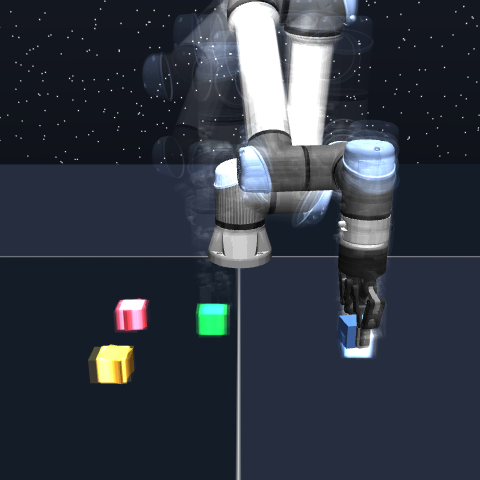} &
        \includegraphics[width=0.15\textwidth]{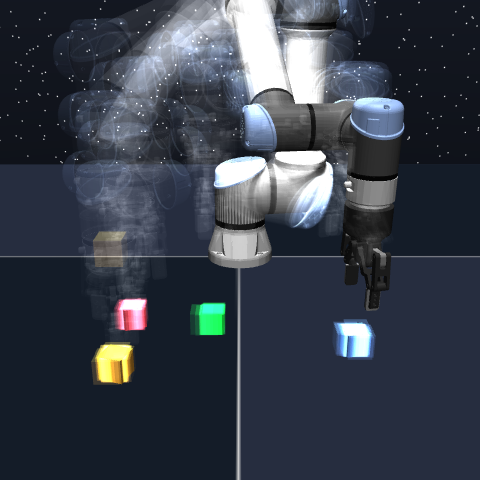} \\
        $t=0$ & $t=15$ & $t=30$ & $t=45$ & $t=60$ & $t=75$ \\[6pt]
        \includegraphics[width=0.15\textwidth]{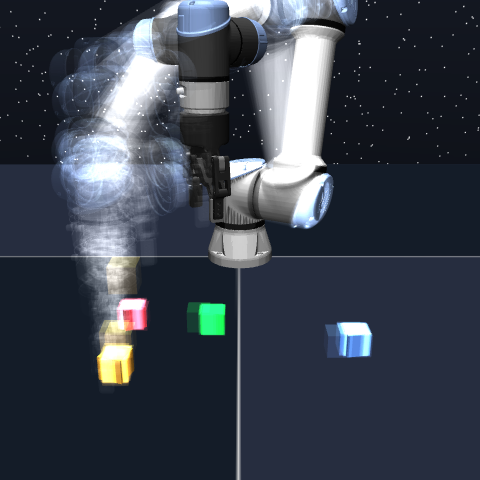} &
        \includegraphics[width=0.15\textwidth]{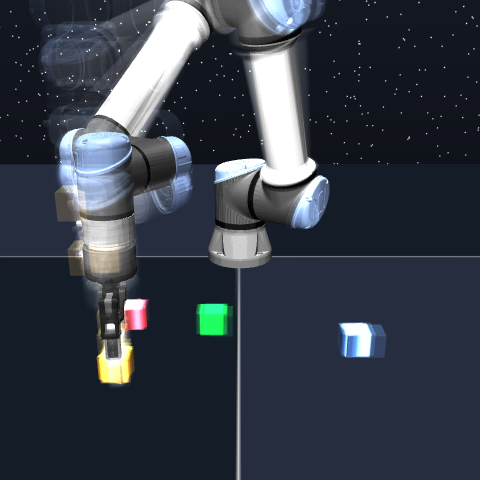} &
        \includegraphics[width=0.15\textwidth]{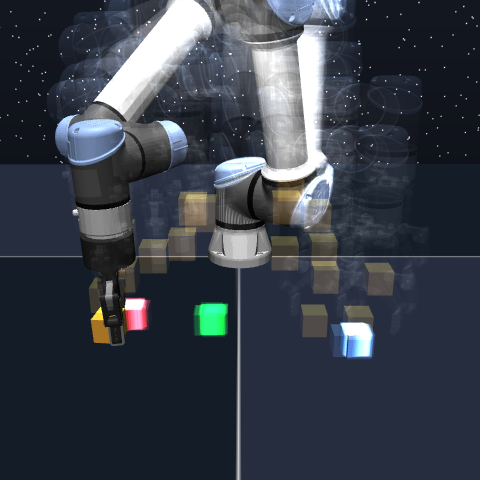} &
        \includegraphics[width=0.15\textwidth]{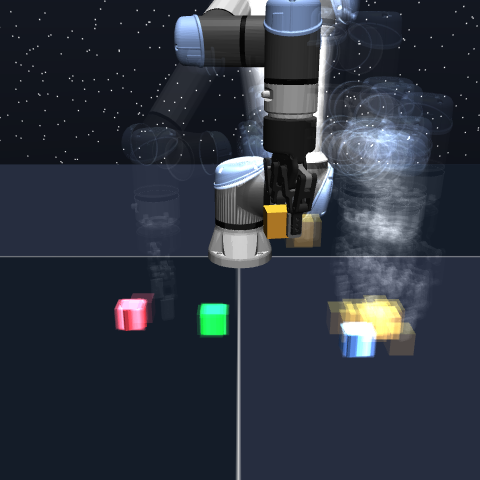} &
        \includegraphics[width=0.15\textwidth]{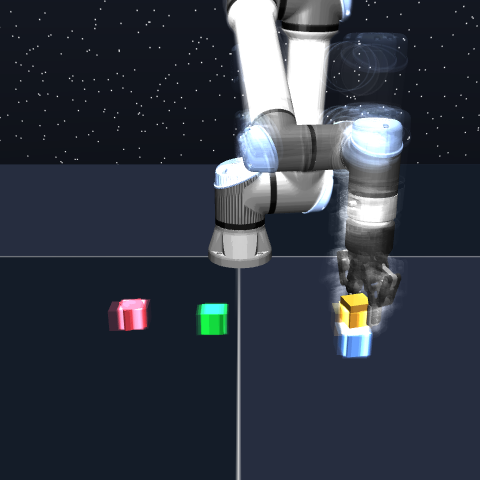} &
        \includegraphics[width=0.15\textwidth]{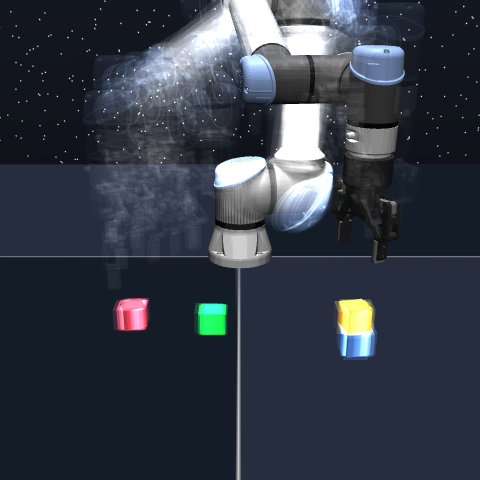} \\
        $t=90$ & $t=105$ & $t=120$ & $t=135$ & $t=150$ & $t=165$ \\[6pt]
        \includegraphics[width=0.15\textwidth]{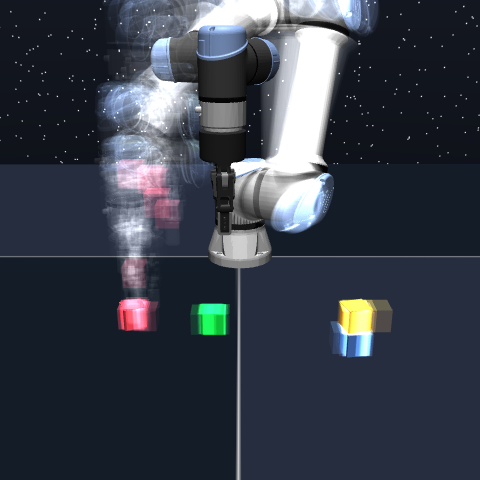} &
        \includegraphics[width=0.15\textwidth]{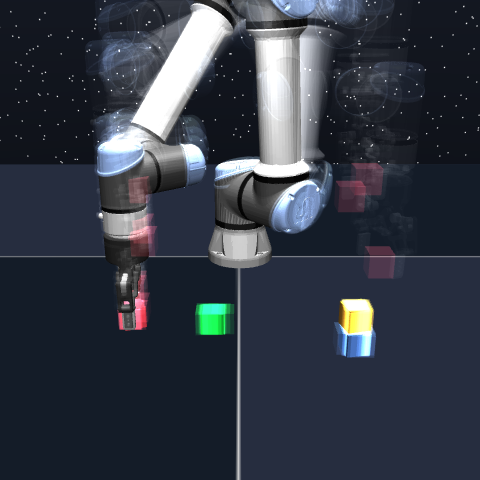} &
        \includegraphics[width=0.15\textwidth]{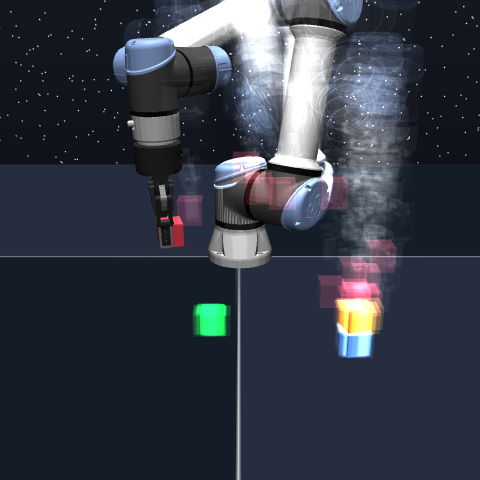} &
        \includegraphics[width=0.15\textwidth]{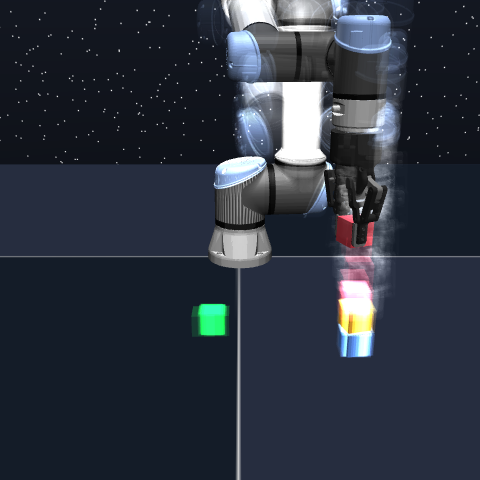} &
        \includegraphics[width=0.15\textwidth]{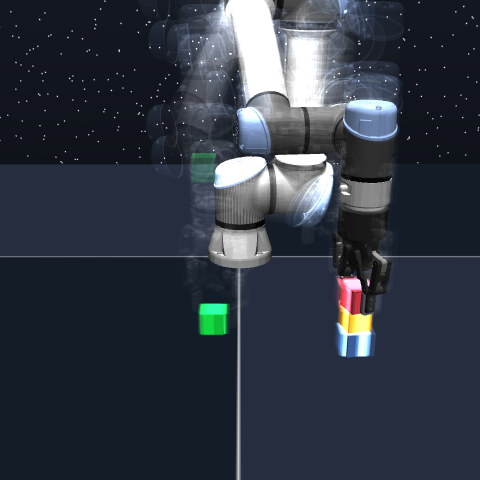} &
        \includegraphics[width=0.15\textwidth]{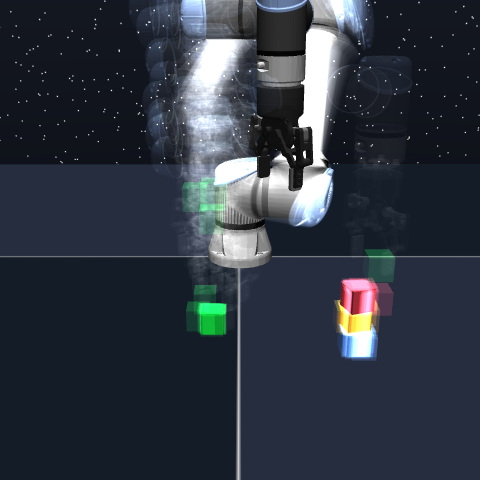} \\
        $t=180$ & $t=195$ & $t=210$ & $t=225$ & $t=240$ & $t=255$ \\[6pt]
        \includegraphics[width=0.15\textwidth]{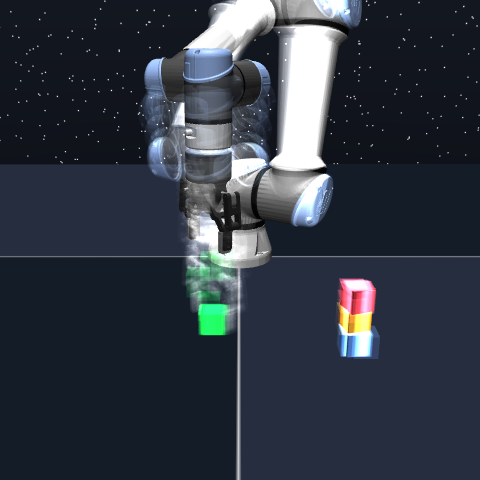} &
        \includegraphics[width=0.15\textwidth]{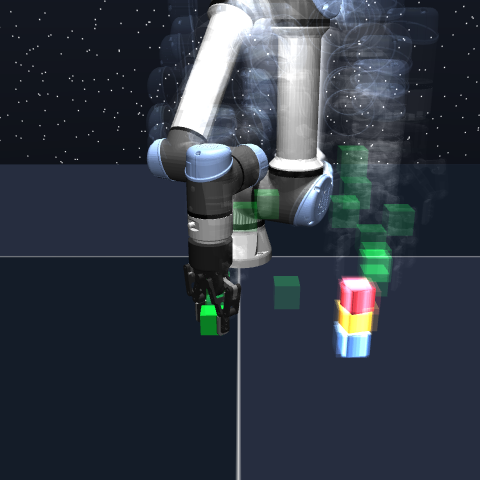} &
        \includegraphics[width=0.15\textwidth]{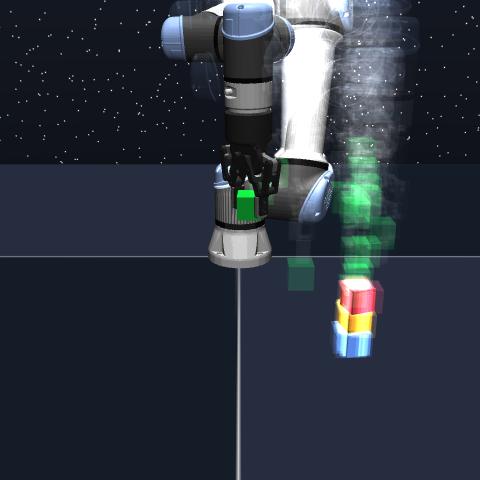} &
        \includegraphics[width=0.15\textwidth]{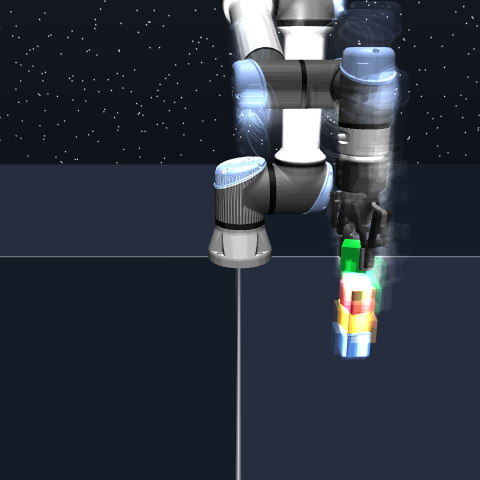} &
        \includegraphics[width=0.15\textwidth]{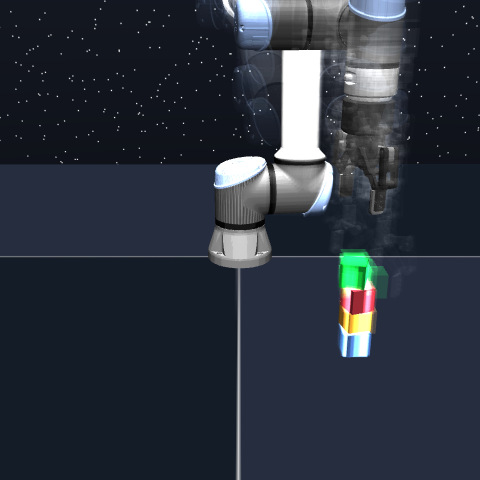} &
        \includegraphics[width=0.15\textwidth]{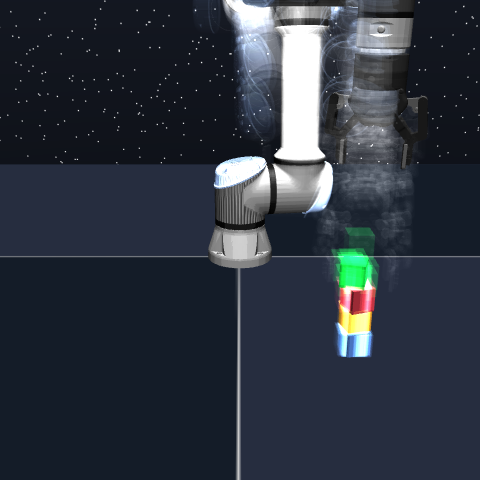} \\
        $t=270$ & $t=285$ & $t=300$ & $t=315$ & $t=330$ & $t=345$ \\
    \end{tabular}
    \caption{Qualitative visualization for an episode of \textsc{cube-4-task-5}. The robots and cubes shown with transparency indicated GHM samples from the first policy in the GSP.}
    \label{fig:cube_frames}
\end{figure}

\clearpage
\section{Experimental Details}

\subsection{Base Policies \& Hyperparameters} \label{app:policies}
Each policy class is chosen to highlight a different aspect of the pipeline (i.e., GHM learning and skill planning).
\begin{enumerate}[noitemsep, topsep=0pt,leftmargin=*]
    \item \textbf{Goal-Conditioned TD3 \citep[GC-TD3;][]{pirotta24bfmil}}: A standard goal-conditioned offline RL algorithm where we employ Flow Q-learning \citep[FQL;][]{park25flow}, a flow-based variant of TD3-BC \citep{fujimoto21td3bc} for policy extraction.
    Due to bootstrapping on its own learned policy, actions may drift out-of-distribution (OOD) from the dataset, posing a challenge for off-policy predictive modeling \citep{levine20offline}. In particular, we use the standard TD3 critic loss \citep{fujimoto18td3} to learn a goal-conditioned action-value function $Q_\phi(S, A, G)$ \citep{schaul15uvfa}:
    \begin{align*}
        \ell(\phi) = \mathbb{E}_{\substack{(S, A,S', G) \\ A' \sim \pi_G(\cdot | S')}} \left [ \left( Q_\phi(S, A, G) - \mathbb{I}\{ S' = G\}  - \gamma Q_{\bar{\phi}}(S', A', G)\right)^2 \right]\,,
    \end{align*}
    where $(S, A, S')$ are transition sampled uniformly from the dataset, and $G$ is a goal state drawn from the following mixture distribution: with probability $0.2$, set $G=S'$, with probability $0.3$ sample $G$ uniformly from the dataset and with $0.5$ set $G$ to a randomly selected future state along the trajectory starting from $(S, A)$,  where the selection time step is drawn from a $\gamma$-geometric distribution.
    \vspace{\baselineskip}

    For policy training, we use the flow Q-learning procedure. We learn jointly a behavior policy $\pi(\cdot \mid S)$ and goal-conditioned policy $\pi_G(\cdot \mid S)$  parametrized by flow-matching vector field. The behavioral policy is trained by a standard conditional flow-matching objective \citep{lipman2022flow} on state-action pairs uniformly sampled from the dataset.
    The goal-conditioned policy is modeled as a one-step flow map. We denote by $\mu_\psi(S, X_0)$ and  $\mu_\omega(S, G, X_0)$ the flow-maps of $\pi$ and $\pi_G$, respectively. 

    \begin{equation}
    \label{eq: flow-q-learning}
        \ell(\psi) = - \mathbb{E}_{\substack{(S, A, G) \\ X_0 \sim \mathcal{N}(0, I)}} \left [ Q(S, \mu_\psi(S, G, X_0), G) + \lambda \| \mu_\psi(S, G, X_0) - \mu_\omega(S, X_0) \|^2 \right ]\,,
    \end{equation}
    where $\lambda$ is the distillation coefficient that controls the behavior cloning regularization. Here $G$ is a goal state drawn from the following mixture distribution: with probability $0.5$ sample $G$ uniformly from the dataset, and with $0.5$ set $G$ to a randomly selected future state along the trajectory starting from $(S, A)$. \\
    \item \textbf{Goal-Conditioned 1-Step RL (GC-1S)}: A more conservative variant of GC-TD3 that bootstraps using the behavior policy via the dataset's actions. We expect this to yield easier-to-model occupancies as we no longer query the learned value function with OOD actions. Policy extraction on the resulting value function is also performed using FQL~\eqref{eq: flow-q-learning}.
    In particular, we follow the same training as explained for TD3, and we change only the critic objective:
 \begin{align*}
        \ell(\phi) = \mathbb{E}_{\substack{(S, A,S', A', G)}} \left [ \left( Q_\phi(S, A, G) - \mathbb{I}\{ S' = G\}  - \gamma Q_{\bar{\phi}}(S', A', G)\right)^2 \right] \,,
    \end{align*}
    where $A'$ is now sampled from the dataset rather than from the learned policy. \\
    \item \textbf{Contrastive RL \citep[CRL;][]{eysenbach22contrastive}}: An alternative value-based approach that uses contrastive learning to approximate the successor measure of the behavior policy \citep{eysenbach22contrastive}. In particular CRL learns a state-action $\phi(s, a)$ and goal encoder $\psi(g)$ by the following Monte-Carlo InfoNCE \citep{oord18representation} contrastive loss:
    \begin{align}
        \ell(\phi, \psi) = \mathbb{E}_{(S, A, G)} \left[ - \phi(S, A)^\top \psi(G) \right] - \left[ \log \sum_{(S, A, G')}  \exp\left ( \phi(S, A)^\top \psi(G')\right)\right]
    \end{align}
    where $G$ is $\gamma$-distributed sampled future state along the trajectory starting from $(S, A)$ and $G'$ is a state distribued uniformly from the dataset.
    Moreover, policy extraction is performed using FQL~\eqref{eq: flow-q-learning} with Q-function $Q(S, A, G)$ estimated as $\phi(S, A)^\top \psi(G)$. \\
    
    \item \textbf{Goal-Conditioned Behavior Cloning \citep[GC-BC;][]{lynch19learning,ghosh21gcbc}}: A purely imitative policy trained via flow matching \citep{lipman2022flow} to mimic the high-quality trajectories in the dataset to reach specific goals using hindsight relabeling \citep{andrychowicz17her}. Specially, we parametrize the policy by vector field $v_\phi(t, S, G)$ 
    \begin{equation}
        \ell(\phi) = \mathbb{E}_{\substack{t, S, A, G \\ X_0 \sim \mathcal{N}(0, I)}} \left [ \| v_\phi(t, S, G) - (A - X_0) \|^2\right ],
    \end{equation}
    where $G$ is $\gamma$-distributed sampled future state along the trajectory starting from $(S, A)$.
    A key consequence of this value-free approach is that the policy struggles to generalize to distant goals, which can limit the effectiveness of proposing goal-directed ``waypoints'' during planning.\\
    \item \textbf{Hierarchical Flow Behavior Cloning~\citep[HFBC;][]{park25horizon}}: an imitative approach that trains two policies: a high-level policy is trained to predict subgoals that are $h$ steps away from the current state for a fixed lookahead $h$, and the low-level policy is trained to predict actions to reach the given subgoal. Specially we parametrize both high-level and low-level policy by two vector fields $v_\phi(t, S, G)$ and $\nu_\psi(t, S, G)$ respectively. 
    \begin{align*}
        \ell(\phi) = \mathbb{E}_{\substack{t, S_n, S_{n+h}, G \\ X_0 \sim \mathcal{N}(0, I)}} \left [ \| v_\phi(t, S_n, G) - (S_{n+h} - X_0) \|^2\right ], \\
        \ell(\psi) = \mathbb{E}_{\substack{t, S_n, A_n, S_{n+h} \\ X_0 \sim \mathcal{N}(0, I)}} \left [ \| v_\phi(t, S_n, S_{n+h}) - (A_{n} - X_0) \|^2\right ],
    \end{align*}
    where $S_{n+h}$ is the h-steps way from the current state $S_n$ and $G$ is is $\gamma$-distributed sampled future state along the trajectory starting from $S_n$
    Consequently, the high-level policy can be used as a proposal for sub-goal distribution at planning time.
\end{enumerate}

\begin{table}[h!]
\vspace{1.5em}
\centering
\caption{Base policy hyperparameters. Parameters in $\{\;\}$ denote sweeps performed over the values inside brackets.}
\renewcommand{\arraystretch}{1.3}
\resizebox{
\ifdim\width>\columnwidth
    0.7\columnwidth
  \else
    0.95\width
  \fi
}{!}{
\begin{tabular}{@{}l p{4.5cm} p{4.5cm} p{4.5cm}@{}}
\toprule
\textbf{Method}
& \textbf{Hyperparameter}
& \textbf{Antmaze} 
& \textbf{Cube} \\
\midrule

\multirow{4}{*}{\makecell[l]{\textbf{CRL}\\\textbf{GC-TD3}\\\textbf{GC-1S}}}
& FQL Distillation coefficient 
& $\{0.1, 0.15, 0.2, 0.3, 0.4\}$
& $\{0.7, 0.8, 0.9, 1, 3\}$ \\[3pt]
\cline{2-4}
& Discount factor 
& \makecell[l]{
$\{0.995, 0.997\}$ for giant\\
$0.99$ otherwise
}
& \makecell[l]{
$\{0.99, 0.995\}$ for \textsc{cube-4}\\
$0.99$ otherwise
} \\[3pt]
\cline{2-4}
& Gradient steps
& $\{1\text{M}, 3\text{M}\}$
& \makecell[l]{
$500\text{k}$ for cube-$\{1,2\}$\\
$1\text{M}$ for \textsc{cube-3}\\
$3\text{M}$ for \textsc{cube-4}
} \\
\midrule

\multirow{4}{*}{\makecell[l]{\textbf{GC-BC}}}
& Discount factor 
& \makecell[l]{
$\{0.99, 0.995\}$ for giant \\
$\{0.98, 0.99\}$ for medium \\
$\{0.98, 0.99\}$ for large
}
& \makecell[l]{
$\{0.95, 0.96\}$ for \textsc{cube-1} \\
$\{0.96, 0.97\}$ for \textsc{cube-2} \\
$\{0.96, 0.97, 0.98\}$ for \textsc{cube-3} \\
$\{0.96, 0.98, 0.99\}$ for \textsc{cube-4}
} \\[3pt]
\cline{2-4}
& Gradient steps
& $125\text{k}$
& \makecell[l]{
$500\text{k}$ for \textsc{cube-1} \\
$1\text{M}$ for \textsc{cube-2} \\
$2\text{M}$  for \textsc{cube-3} \\
$3\text{M}$ for \textsc{cube-4}
} \\
\midrule

\multirow{4}{*}{\makecell[l]{\textbf{HFBC}\\\textbf{SHARSA}}}
& Lookahead & $[25, 50]$ & $[25, 50]$
\\
\cline{2-4}
& Discount factor 
& $\{0.99, 0.995\}$
& $\{0.95, 0.99, 0.995\}$
 \\[3pt]
\cline{2-4}
& Gradient steps
& $3\text{M}$
& \makecell[l]{
$1\text{M}$ for cube-$\{1, 2\}$ \\
$3\text{M}$  for cube-$\{3, 4\}$
} \\
\midrule
\end{tabular}
}
\end{table}

\clearpage
\subsection{Geometric Horizon Model Hyperparameters}

\begin{table}[H]
\caption{Hyperparameters for Geometric Horizon Model pre-training.}
\label{table:fm-td-consistency-hparams}
\vspace{1mm}
\renewcommand{\arraystretch}{1.25}
\centering
\resizebox{
\ifdim\width>\columnwidth
    0.7\columnwidth
  \else
    \width
  \fi
}{!}{
\begin{tabular}{@{}cll@{}}
     \toprule
     & {\fontsize{11}{13}\selectfont \textbf{Hyperparameter}}
     & {\fontsize{11}{13}\selectfont \textbf{Value}} \\
     \midrule
     \multirow{5}{*}{\makecell[c]{Flow Matching\\\citep{lipman2022flow}}}
     & Probability Path & Conditional OT ($\sigma{=}0$) \\
     & Time Sampler & $\mathcal{U}([0, 1])$ \\
     & ODE Solver & Euler \\
     & ODE $\mathrm{d}t$ (train) / steps & $0.1$ / $10$ \\
     & ODE $\mathrm{d}t$ (eval) / steps & $0.05$ / $20$ \\
     \midrule
\multirow{5}{*}{\makecell[c]{Network (U-Net)\\\citep{ronneberger15unet}}}
    & $t$-Positional Embedding Dim. & $256$ \\
    & $t$-Positional Embedding MLP & $(1024, 1024)$ \\
    & Hidden Activation & mish \citep{misra19mish} \\
    & Blocks per Stage & $1$ \\
    & Block Dimensions & $(1024, 1024, 1024)$ \\
    \midrule
\multirow{3}{*}{\makecell[c]{Conditional Encoder}}
    & Encoder MLP & $(1024, 1024, 1024)$ \\
    & Encoder Activation & mish \citep{misra19mish} \\
    & Conditioning Mixing & additive \\
    \midrule
\multirow{3}{*}{\makecell[c]{Optimizer Adam\\\citep{kingma15adam}}}
    & Learning Rate & $10^{-4}$ \\
    & Weight Decay & $0$ \\
    & Gradient Norm Clip & --- \\
    \midrule
\multirow{6}{*}{\makecell[c]{Training}}
    & Max Discount $\gamma_{\mathrm{max}}$ & $0.996$ \\
    & Target Network EMA & $5 \times 10^{-4}$ \\
    & Gradient Steps & $3$M \\
    & Batch Size & $256$ \\
    & Context Drop Probability ($z = \varnothing$) & $0.1$ \\
    & Consistency Proportion & \!\!\!\!\hspace{-0.05em}\hspace{-0.28em}$\scriptstyle \begin{cases}\; 0.25 & \text{for \textsc{antmaze}} \\ \; 0.15 & \text{for \textsc{cube}} \end{cases}$ \\
    \midrule
\multirow{3}{*}{\makecell[c]{GCRL Goal Sampling}}
    & $p(\text{trajectory goal})$ & $0.5$ \\
    & $p(\text{random goal})$ & $0.5$ \\
    & Geometric Trajectory Discount & $0.995$ \\
\bottomrule
\end{tabular}
}
\end{table}

\clearpage
\subsection{Planning Hyperparameters}

\begin{table}[h!]
\centering
\caption{\textsc{CompPlan} hyperparameters for the main results of the paper.}
\renewcommand{\arraystretch}{1.25}
\begin{adjustbox}{width=\textwidth}
\begin{tabular}{llllllll}
\toprule
{\bfseries Method}
& {\bfseries Domain}
& {\bfseries Candidates} 
& {\bfseries Effective horizons}
& {\bfseries Proposal Distribution}
& {\bfseries Eval samples} 
& {\bfseries Replan every}
& {\bfseries Discount } \\
\midrule
\multirow{7}{*}{\textbf{CRL}} & \textsc{antmaze-medium} & 256 & $[50, 50, 100, 100, 200]$ & Conditional & 256 & 1 & 0.999 \\
& \textsc{antmaze-large} & 256 & $[50, 50, 100, 100, 200]$ & Conditional & 256 & 1 & 0.999 \\
& \textsc{antmaze-giant} & 256 & $[50, 50, 100, 100, 200]$ & Conditional & 256 & 1 & 0.999 \\
& \textsc{cube-1} & 1024 & $[20, 80]$ & Unconditional & 128 & 1 & 0.99 \\
& \textsc{cube-2} & 1024 & $[20, 20, 80]$ & Unconditional & 128 & 1 & 0.99 \\
& \textsc{cube-3} & 1024 & $[20, 20, 20, 80]$ & Unconditional & 128 & 1 & 0.99 \\
& \textsc{cube-4} & 1024 & $[20, 20, 20, 20, 80]$ & Unconditional & 128 & 1 & 0.99 \\
\midrule
\multirow{7}{*}{\textbf{GC-TD3}}
& \textsc{antmaze-medium} & 256 & $[50, 50, 100, 100, 200]$ & Conditional & 256 & 1 & 0.999 \\
& \textsc{antmaze-large} & 256 & $[50, 50, 100, 100, 200]$ & Conditional & 256 & 1 & 0.999 \\
& \textsc{antmaze-giant} & 256 & $[50, 50, 100, 100, 200]$ & Conditional & 256 & 1 & 0.999 \\
& \textsc{cube-1} & 1024 & $[20, 80]$ & Unconditional & 128 & 1 & 0.99 \\
& \textsc{cube-2} & 1024 & $[20, 20, 80]$ & Unconditional & 128 & 1 & 0.99 \\
& \textsc{cube-3} & 1024 & $[20, 20, 20, 80]$ & Unconditional & 128 & 1 & 0.99 \\
& \textsc{cube-4} & 1024 & $[20, 20, 20, 20, 80]$ & Unconditional & 128 & 1 & 0.99 \\
\midrule
\multirow{7}{*}{\textbf{GC-BC}}
& \textsc{antmaze-medium} & 256 & $[50, 50, 100, 100, 200]$ & Conditional & 256 & 1 & 0.999 \\
& \textsc{antmaze-large} & 256 & $[50, 50, 100, 100, 200]$ & Conditional & 256 & 1 & 0.999 \\
& \textsc{antmaze-giant} & 256 & $[50, 50, 100, 100, 200]$ & Conditional & 256 & 1 & 0.999 \\
& \textsc{cube-1} & 1024 & $[20, 80]$ & Unconditional & 128 & 1 & 0.99 \\
& \textsc{cube-2} & 1024 & $[20, 20, 80]$ & Unconditional & 128 & 1 & 0.99 \\
& \textsc{cube-3} & 1024 & $[20, 20, 20, 80]$ & Unconditional & 128 & 1 & 0.99 \\
& \textsc{cube-4} & 1024 & $[20, 20, 20, 20, 80]$ & Unconditional & 128 & 1 & 0.99 \\
\midrule
\multirow{7}{*}{\textbf{GC-1S}}
& \textsc{antmaze-medium} & 256 & $[50, 50, 100, 100, 200]$ & Conditional & 256 & 1 & 0.999 \\
& \textsc{antmaze-large} & 256 & $[50, 50, 100, 100, 200]$ & Conditional & 256 & 1 & 0.999 \\
& \textsc{\textsc{antmaze-giant}} & 256 & $[50, 50, 100, 100, 200]$ & Conditional & 256 & 1 & 0.999 \\
& \textsc{cube-1} & 1024 & $[20, 80]$ & Unconditional & 128 & 1 & 0.99 \\
& \textsc{cube-2} & 1024 & $[20, 20, 80]$ & Unconditional & 128 & 1 & 0.99 \\
& \textsc{cube-3} & 1024 & $[20, 20, 20, 80]$ & Unconditional & 128 & 1 & 0.99 \\
& \textsc{cube-4} & 1024 & $[20, 20, 20, 20, 80]$ & Unconditional & 128 & 1 & 0.99 \\
\midrule
\multirow{7}{*}{\textbf{HFBC}}
& \textsc{antmaze-medium} & 32 & $[25]*24$ & Conditional & 128 & 1 & 0.999 \\
& \textsc{antmaze-large} & 32 & $[25]*24$ & Conditional & 128 & 1 & 0.999 \\
& \textsc{antmaze-giant} & 32 & $[25]*24$ & Conditional & 128 & 1 & 0.999 \\
& \textsc{cube-1} & 32 & $[25]*4$ & Unconditional & 128 & 1 & 0.99 \\
& \textsc{cube-2} & 32 & $[25]*5$ & Unconditional & 128 & 1 & 0.99 \\
& \textsc{cube-3} & 32 & $[25]*6$ & Unconditional & 128 & 1 & 0.99 \\
& \textsc{cube-4} & 32 & $[25]*7$ & Unconditional & 128 & 1 & 0.99 \\
\bottomrule
\end{tabular}
\end{adjustbox}
\end{table}

\end{appendices}

\end{document}